\documentclass{article}

\pdfoutput=1
    \PassOptionsToPackage{numbers, compress}{natbib}
    \usepackage[preprint]{neurips_2023}

\usepackage[utf8]{inputenc} % allow utf-8 input
\usepackage[T1]{fontenc}    % use 8-bit T1 fonts
\usepackage{hyperref}       % hyperlinks
\usepackage{url}            % simple URL typesetting
\usepackage{booktabs}       % professional-quality tables
\usepackage{amsfonts}       % blackboard math symbols
\usepackage{nicefrac}       % compact symbols for 1/2, etc.
\usepackage{microtype}      % microtypography
\usepackage{xcolor}         % colors

\usepackage{amsmath}
\usepackage{amssymb}
\usepackage{mathtools}
\usepackage{amsthm}
\usepackage{xspace}

\usepackage{wrapfig}

\usepackage{cleveref}

\newcount\Comments  % 1 suppresses notes to selves in text
\theoremstyle{plain}
\newtheorem{theorem}{Theorem}[section]

\newtheorem{lemma}[theorem]{Lemma}

\theoremstyle{definition}

\newtheorem{assumption}[theorem]{Assumption}
\theoremstyle{remark}
\newtheorem{remark}[theorem]{Remark}
\usepackage{color}
\newcommand{\kibitz}[2]{\ifnum\Comments=0\textcolor{#1}{#2}\fi}

\newcommand{\name}{\texttt{COMMIT}\xspace}

\definecolor{darkgreen}{rgb}{0,0.5,0}

\newcommand{\iou}{\mathrm{IoU}}
\newcommand{\vol}{\mathrm{Vol}}
\newcommand{\Zsub}{\gZ_{\mathrm{sub}}}

\definecolor{revblue}{RGB}{10,70,244}

\usepackage[textsize=tiny]{todonotes}
\usepackage{multirow}
\usepackage{makecell}
\usepackage{enumitem}
\usepackage{subcaption}
\usepackage{aisecure-math}
\usepackage{bbm}

\usepackage{algorithmic}
\usepackage[ruled,vlined,linesnumbered]{algorithm2e}
\usepackage{adjustbox}
\SetKwComment{Comment}{$\triangleright$\ }{}
\Crefname{algocf}{Algorithm}{Algorithms}
\usepackage{multicol}

\title{\name: Certifying Robustness of Multi-Sensor Fusion Systems against Semantic Attacks}

\author{%
  Zijian Huang \\
  Department of Computer Science \\
  University of Illinois Urbana-Champaign \\
  Champaign, IL 61820 \\
  \texttt{zijianh4@illinois.edu} \\
  \And
  Wenda Chu\\
  Institute for Interdisciplinary Information Sciences \\
  Tsinghua University \\
  Beijing, China 100084 \\
  \texttt{chuwd19@mails.tsinghua.edu.cn} \\
  \And
  Linyi Li \\
  Department of Computer Science \\
  University of Illinois Urbana-Champaign \\
  Champaign, IL 61820 \\
  \texttt{linyi2@illinois.edu} \\
  \And
  Chejian Xu \\
  Department of Computer Science \\
  University of Illinois Urbana-Champaign \\
  Champaign, IL 61820 \\
  \texttt{chejian2@illinois.edu} \\
  \And
  Bo Li \\
  Department of Computer Science \\
  University of Illinois Urbana-Champaign \\
  Champaign, IL 61820 \\
  \texttt{lbo@illinois.edu} \\
}

\begin{document}

\maketitle

\begin{abstract}
    Multi-sensor fusion systems (MSFs) play a vital role as the perception module in modern autonomous vehicles (AVs). Therefore, ensuring their robustness against common and realistic adversarial semantic transformations, such as rotation and shifting in the physical world, is crucial for the safety of AVs. 
    While empirical evidence suggests that MSFs exhibit improved robustness compared to single-modal models, they are still vulnerable to adversarial semantic transformations. Despite the proposal of empirical defenses, several works show that these defenses can be attacked again by new adaptive attacks. So far, there is no certified defense proposed for MSFs.
    In this work, we propose the first robustness certification framework \name to \underline{c}ertify r\underline{o}bustness of \underline{m}ulti-sensor fusion systems against se\underline{m}ant\underline{i}c a\underline{t}tacks.
    In particular, we propose a practical anisotropic noise mechanism that leverages randomized smoothing with multi-modal data and performs a grid-based splitting method to characterize complex semantic transformations. We also propose efficient algorithms to compute the certification in terms of object detection accuracy and IoU for large-scale MSF models.
    Empirically, we evaluate the efficacy of \name in different settings and provide a comprehensive benchmark of certified robustness for different MSF models using the CARLA simulation platform.
    We show that the certification for MSF models is at most 48.39\% higher than that of single-modal models, which validates the advantages of MSF models.
    % In particular, we show $100.00\%$ certified object detection rate and $53.23\%$ certified IoU for vehicle rotation within $30^\circ$; and $100.00\%$ certified detection rate and $48.39\%$ certified IoU for vehicle shifting within $[10,15]$, which demonstrates at least $48.39\%$ improvement over the certifications for single-modal models, which provide certified for MSF robustness.
    We believe our certification framework and benchmark will contribute an important step towards certifiably robust AVs in practice.
% \vspace{-12pt}
\end{abstract}

\section{Introduction}
\label{introduction}

% \linyi{maybe we need to draw a fancy Figure 1 for the paper}

% \linyi{motivation: smoothing cannot be applied to practical transformations feasibly. So here, we enable certification for complex transformations using feasible and classical smoothing protocols.}

% \linyi{motivation for rotation: may turn around}

% \linyi{motivation for position in $z$-axis: may suddenly take a brake}

Autonomous driving~(AD) has achieved significant advances in recent years~\cite{redmon2016you,law2018cornernet,badrinarayanan2017segnet,zhao2018icnet,zhou2018voxelnet,luo2018fast,qi2018frustum,chen2017multi}, and deep neural networks~(DNNs) have been largely deployed as the perception module for AD to process inputs from multiple sources (e.g., camera and LiDAR) to detect objects such as road signs, vehicles, and pedestrians.
To make full use of multi-modal inputs, modern AD systems usually adopt the multi-sensor fusion systems (MSFs) as the perception module~\cite{pang2020clocs,chen2022focal}.

Along with the wide deployment of AD systems, the safety of AD systems in the physical world has raised serious concerns~\cite{hendrycks2021unsolved,cao2021invisible}.
A rich body of research has shown that both adversarial perturbations and natural semantic transformations can mislead the DNN-based perception modules in AD systems with a high success rate~\cite{pei2017deepxplore,hosseini2018semantic,xiao2018spatially,guo2019safe,hendrycks2018benchmarking,engstrom2019exploring},
so that the AD system may ignore the pedestrians, the traffic signs, or other vehicles with high confidence when the object is slightly rotated or shifted, which can lead to severe consequences such as fatal traffic accidents~\cite{mccausland2019}.
Moreover, even though the multi-sensor fusion systems may be intuitively  more robust, assuming that the input transformations/perturbations are not adversarial to multiple input modalities at the same time, existing work~\cite{cao2021invisible,hallyburton2022security} has falsified such intuition by proposing feasible and highly efficient attacks against multi-sensor fusion systems~\cite{cao2021invisible}.
In other words, serious robustness issues still exist in existing MSFs, resulting in practical safety vulnerabilities in AD.
\begin{wrapfigure}{r}{0.5\linewidth}
    \centering
    \vspace{-2mm}
    \includegraphics[width=\linewidth]{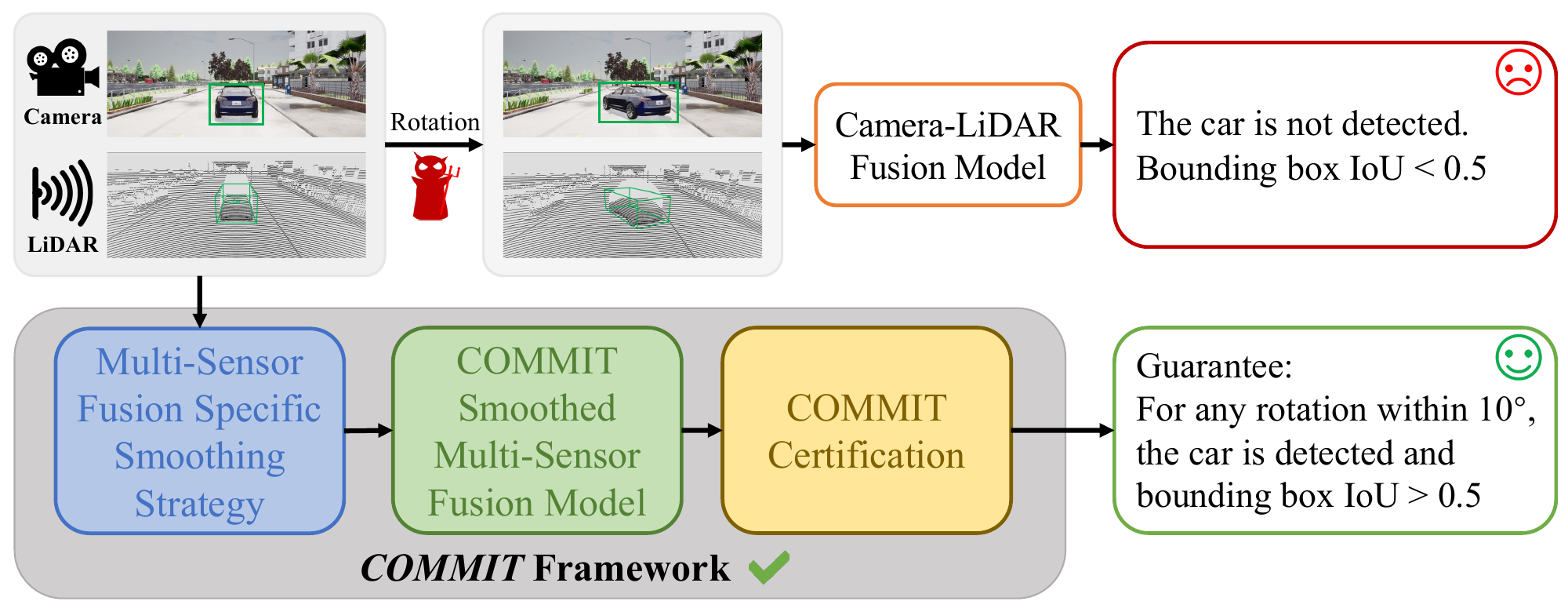}
    \vspace{-5mm}
    \caption{\small Overview of \name, the first framework that provides certified robustness for multi-sensor fusion systems against semantic transformations.}
    \label{fig:framework}
% \vspace{-2em}
% \vspace{-10pt}
\vspace{-2mm}
\end{wrapfigure}

\vspace{-3mm}
To mitigate such safety threats, several empirical defenses have been proposed for both single-modal models~\cite{madry2018towards} and multi-sensor fusion systems~\cite{zhong2022detecting}.
However, certified defenses exist only for single-modal models~\cite{wong2018provable,cohen2019certified,li2021tss,chu2022tpc}. Recent works have shown that the empirical defenses for MSFs can be adaptively attacked again by stealthy perturbations or transformations~\cite{zhong2022detecting,huang2022pointcat}.
In this paper, we aim to provide \textbf{the first robustness certification and enhancement framework for multi-sensor fusion systems} in AD against various semantic transformations in the physical world.

Our framework leverages the \emph{randomized smoothing} technique~\cite{cohen2019certified}, while randomized smoothing cannot directly provide a robustness guarantee against semantic transformations~(e.g., object rotation and shifting) for multi-sensor fusion systems due to three main reasons:
(1)~\textbf{Heterogeneous input dimensions}: In randomized smoothing, the isotropic noise is added to all input dimensions, which is sub-optimal for multi-sensor fusion systems since different input modalities need different noises.
(2)~\textbf{Intractable perturbation spaces}: Semantic transformations incur large $\ell_p$ that cannot be handled by classical randomized smoothing~\cite{yang2020randomized}, and the transformation function does not have a closed-form expression that is required for semantic-smoothing-based certification~\cite{li2021tss}.
(3)~\textbf{Unsupported certification criterion}: Existing randomized smoothing techniques are designed for certifying output consistency for classification~\cite{cohen2019certified} and regression~\cite{kumar2020certifying} tasks. However, multi-sensor fusion systems output 3D bounding boxes for detected objects and use IoU as the evaluation criterion, while randomized smoothing cannot provide worst-case certification for IoU.

To solve these challenges, our framework \name provides the following techniques: 
(1)~We derive an anisotropic noise mechanism that is practical~(agnostic to the transformations to be certified) and efficient for randomized smoothing over multi-modal data.
(2)~Under a mild assumption, we propose a grid-based splitting method to integrate small $\ell_p$ certifications and form a holistic certification against complex semantic transformations. 
(3)~We derive the first rigorous lower bounds of detection confidence and lower bounds of IoU for MSF models.
% These techniques solve above challenges respectively.

% To go a step further, we want to address the certified robustness problem of transformation attacks in the physical world, especially in the Multi-Sensor Fusion Systems. One of the biggest challenge is that smoothing cannot be applied to practical transformations feasibly (e.g. smoothing cannot be applied to rotation angle or shift distance directly). Therefore, we enable certification for complex transformations using feasible smoothing protocols on the model inputs directly, and prove the effectiveness of our certification in \Cref{sec:Multi-Sensor Fusion Certification}. 

We leverage our framework to certify the state-of-the-art large-scale camera and LiDAR fusion 3D object detection models~(CLOCs \cite{pang2020clocs} and FocalsConv \cite{chen2022focal}) and compare them with a camera-based 3D object detector~(MonoCon \cite{liu2022learning}) as well as one LiDAR-based 3D object detector~(SECOND \cite{yan2018second}) in CARLA simulator. We consider transformations such as  rotation and shifting, which correspond to the turning around  and sudden brake cases in the real world. In particular, we are able to achieve 100.00\% certified detection accuracy and 53.23\% certified IoU under rotation transformations. Compared to single-modal models, the certification improvements are 25.19\% and 53.23\%, respectively. We demonstrate that the certified robustness depends on both the input and the fusion pipeline structure. 
% \zijian{Add some other experimental conclusions.}
\vspace{-2mm}
\paragraph{\underline{Technical Contributions.}} In this paper, we provide the first certification framework for multi-sensor fusion systems against adversarial semantic transformations. 
% representing an important step towards certifiably safe autonomous driving. 
We make contributions on both theoretical and empirical fronts.
\vspace{-5pt}

\begin{itemize}[leftmargin=*]
    \item We propose the \emph{first} generic framework for certifying the robustness of multi-sensor fusion systems against practical semantic transformations in the physical world. 
    
    \item We propose a practical anisotropic noise mechanism to leverage randomized smoothing given multi-modal data, a grid-based splitting method to characterize complex semantic transformations, 
    and efficient algorithms to compute the certification for object detection and IoU lower bounds for large-scale MSF models.
    
    % We provide practical and efficient noise mechanisms for multi-sensor fusion systems, a grid-based splitting method for transformation certification, and a computable IoU lower bound for 3D object detection.
    % % tasks given bounding box ranges. 
    % % theoretically prove the certified lower bound of detection confidence score and the IoU between predicted bounding boxes and the ground truth bounding boxes under bounded adversarial transformation perturbations.
    % % Specifically, we propose the \textit{first} algorithm to compute the lower bound of IoU given the lower bound and upper bound of bounding box parameters. 
    
    \item We construct extensive experiments and provide a benchmark of certified robustness for multi-sensor fusion systems based on \name.
    We certify several state-of-the-art camera-LiDAR fusion models and compare them with single-modal models. We show that the multi-sensor fusion systems provide nontrivial gains on certified robustness, e.g., achieving 53.23\% improvement against the rotation transformation. In addition, we present several interesting observations which would further inspire the development of robust sensor fusion algorithms. 
    The benchmark will be open source upon acceptance and will be continuously expanding to evaluate more AD systems. 
    % \linyi{maybe need a website or repo link}
\end{itemize}

\section{Related Work}
\label{sec:related}

\textbf{Multi-sensor fusion systems.} 
Multi-sensor fusion DNN systems leverage data of multiple modalities to predict 3D bounding boxes for object detection. 
In this work, we consider multi-sensor fusion systems that take both image~(from a camera) and point clouds~(from a LiDAR sensor) for object detection \cite{pang2020clocs,chen2022focal}, which is one of the most common forms of AD perception module~\cite{shen2022sok}.
These fusion systems typically integrate outputs from sub-models for each modality via learning-based methods or aggregation rules.
Note that our framework is architecture-agnostic --- applicable for any  fusion system regardless of their internal architectures.

\textbf{Adversarial attacks for DNNs.}
The robustness vulnerabilities of DNNs are manifested by adversarial attacks.
A rich body of research shows that DNNs can be attacked by pixel-wise perturbations bounded by small $\ell_p$ norm with even 100\% success rate~\cite{szegedy2013intriguing,goodfellow2014explaining,carlini2017towards}.
Besides $\ell_p$-bounded perturbations, subsequent research shows that spatial transformations~\cite{xiao2018spatially}, occlusions~\cite{sun2020towards}, and semantic transformations that naturally exist~\cite{pei2017deepxplore,hendrycks2018benchmarking,Ghiasi2020BREAKING} can also mislead DNNs to make severe incorrect predictions.
In particular, several physically realizable and effective adversarial attacks have been proposed against multi-sensor fusion systems~\cite{eykholt2018robust,cao2021invisible}, posing serious safety threats to modern AD.
% Since the adversarial vulnerability of neural networks has been discovered \cite{goodfellow2014explaining,szegedy2013intriguing}, a great amount of evasion attacks have been proposed \cite{carlini2017towards, xiao2018generating, athalye2018obfuscated, tramer2020adaptive}. In addition to the ordinary adversarial attacks, many researchers notice that ML models can be misled by many semantic transformations, such as rotation and transition,  \cite{hendrycks2018benchmarking, xiao2018spatially, Ghiasi2020BREAKING}, which can degrade the performance of ML models about 40\%-100\% on CIFAR-10 or ImageNet. With the successful attacks in the digital space, several attack methods in the physical world have been proposed \toCite, which can destroy image classifiers and object detectors. To go a step further, our work studies the effect of transformation attacks in the real world. 

% \textbf{Certified robustness against adversarial attacks.} While heuristic defenses against evasion attacks \cite{madry2018towards, samangouei2018defense, shafahi2019adversarial} and semantic attacks \cite{engstrom2019exploring, }have reached significant improvement, 
\textbf{Certified robustness for DNNs.}
To mitigate the robustness vulnerabilities, several defenses are proposed, which can be roughly categorized into empirical and certified defenses.
The \emph{empirical defenses}~\cite{madry2018towards, samangouei2018defense, shafahi2019adversarial} train DNNs with heuristic approaches, e.g., adversarial training, to defend against  adversarial attacks.
However, they cannot provide rigorous robustness guarantees against possible future attacks.
In contrast, \emph{certified defenses} can prove that the trained DNNs are certifiably robust against any possible attacks under some perturbation constraints~\cite{li2023sok}.
Certified defenses are mainly based on verification methods like linear relaxations with branch-and-bound~\cite{wong2018provable,zhang2022general}, Lipschitz DNN architectures~\cite{zhang2022rethinking,singla2021skew,xu2022lot}, and randomized smoothing~\cite{cohen2019certified,yang2020randomized,chiang2020detection,li2022double,sun2022spectral,carlini2023certified}.

For multi-sensor fusion systems, although some  empirical defenses have been proposed~\cite{liu2022attack,zhong2022detecting}, there is no certified defense that provides robustness guarantees to our best knowledge.
Thus, here we aim to provide the first robustness certification and enhancement framework for MSF models.

% \textbf{Certified robustness against adversarial attacks.} While heuristic defenses against evasion attacks \cite{madry2018towards, samangouei2018defense, shafahi2019adversarial} and semantic attacks \cite{hendrycks2018benchmarking, engstrom2019exploring} have been investigated, new adaptive attacks \toCite have been proposed to break these empirical defenses, which emphasizes the importance of provable robustness of ML models.  In the domain of certification for the models against transformation attacks, discrete transformation can be solved by efficient enumeration \cite{pei2017towards}, and common semantic transformations like can be handled by interval bound propagation \cite{balunovic2019certifying, singh2019abstract, fischer2020certified}, while linear relaxations \toCite and randomized smoothing \toCite can deal with this problem more tightly. However, there is still few works on the certied robustness against semantic transformation attacks in the physical world or on multi-modality models, which is what we want to solve in our work. 

\section{Robustness Certification for Multi-Sensor Fusion Systems}
\label{sec:Multi-Sensor Fusion Certification}

In this section, we introduce our framework \name for certifying the robustness of multi-sensor fusion systems against semantic transformations in detail.

\subsection{Threat Model and Certification Goal}
    \label{subsec:setup}
\textbf{Notation.}
We consider a multi-sensor fusion system that takes an image from a camera and a point cloud from a LiDAR sensor as the input and outputs several labeled 3D bounding boxes for its detected objects.
In particular, the image input $\vx \in \gX \subseteq \sR^d$ has $d$ dimensions, and the point cloud input $\vp\in \gP \subseteq \sR^{3\times N}$ contains $N$ (un-ordered) 3D point coordinates.
% \begin{rev}
Note that our framework can be easily extended to handle point clouds with intensity.
% \end{rev}
In the output, each labeled 3D bounding box is a tuple of box coordinates $B = (x,y,z,w,h,l,r) \in \gB \subseteq \sR^6 \times [0,2\pi]$~(where $x,y,z,w,h,l$ are 3D center coordinates and width, height, length respectively, and $r$ is the rotation angle in the $x-z$ plane), label $c\in \gC$, and confidence score $s\in [0,1]$.
Hence, a multi-sensor fusion system can be modeled by a function $g: \gX \times \gP \to (\gB \times \gC \times [0,1])^n$ where $n$ is of variable length and stands for the number of output bounding boxes.
% Given an image $\vx \in \gX \subseteq \mathbb R^d$ and a point cloud input $\vp\in \gP \subseteq \sR^{3\times N}$, a multi-sensor fusion system $g: \mathcal X \times \mathcal P \to (\mathcal B\times \mathcal C \times [0,1])^n$ outputs $n$ 3D bounding boxes. 
% $B \in \mathcal B\subseteq \mathbb R^6\times[0,2\pi]$ describes the geometric shape $(x,y,z,w,h,l,r)$ of the bounding box, $\mathcal C$ is the set of labels and $s \in [0,1]$ represents for its confidence score.

\textbf{Threat model.} 
An adversary can apply a certain parameterized transformation that may alter both the image and point clouds to mislead the model.
We formulate a transformation by two functions $T = \{T_x, T_p\}$ where
$T_x: \gX \times \gZ \to \gX$ transforms images and $T_p: \gP \times \gZ \to \gP$ transforms point clouds respectively.
Note that $\gZ \subseteq \sR^m$ is the set of valid and continuous parameters of the transformation, which is usually in a low-dimensional space, i.e., $m$ is small.
We consider the strongest adversary that can pick an \emph{arbitrary} parameter $\vz \in \gZ$ to transform the input $\small\left( \begin{matrix} \vx \\ \vp \end{matrix} \right) \mapsto \left( \begin{matrix} T_x(\vx, \vz) \\ T_p(\vp, \vz) \end{matrix} \right)$ and feed into the system.

In particular, we will instantiate our robustness certification framework for two common transformations: \textbf{rotation} and \textbf{shifting}.
The rotation transformation $T_{\mathrm{rot}}$ takes a scalar rotation angle $r$ as the parameter and rotates the front car in the $x-z$ plane clockwise.
The angle can be negative, meaning a counterclockwise rotation.
The shifting transformation $T_{\mathrm{sft}}$ takes a scalar distance $\delta\in\sR_+$ as the parameter and places the front car in $\delta$ meters away, i.e., imposes a $\delta$ displacement along the $z$-axis.
\Cref{fig:transformation_visualization} illustrates these two types of transformations.
Note that our framework will require only oracle access to the output of the transformation function to derive robustness certification.
% \begin{rev}
Hence, our framework can be readily extended to other transformations, as long as the transformation is measurable, i.e., can be deterministically parameterized. 
% \end{rev}

\begin{figure}[!t]
    \centering
    \subcaptionbox{Rotate based on y-axis\label{fig:rotation viz}\vspace{5pt}}{
        \includegraphics[width=.22\linewidth]{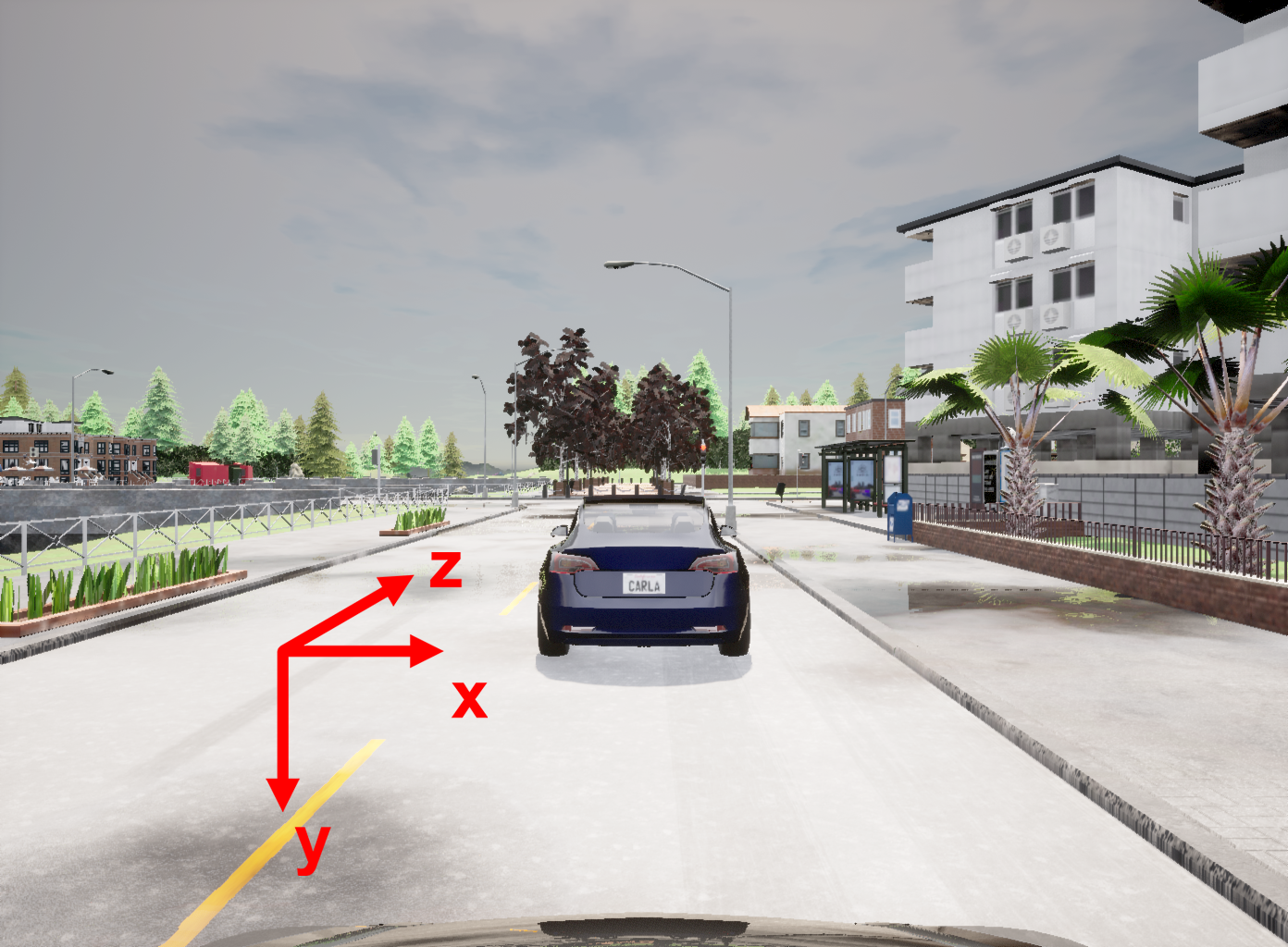}\quad
        \includegraphics[width=.22\linewidth]{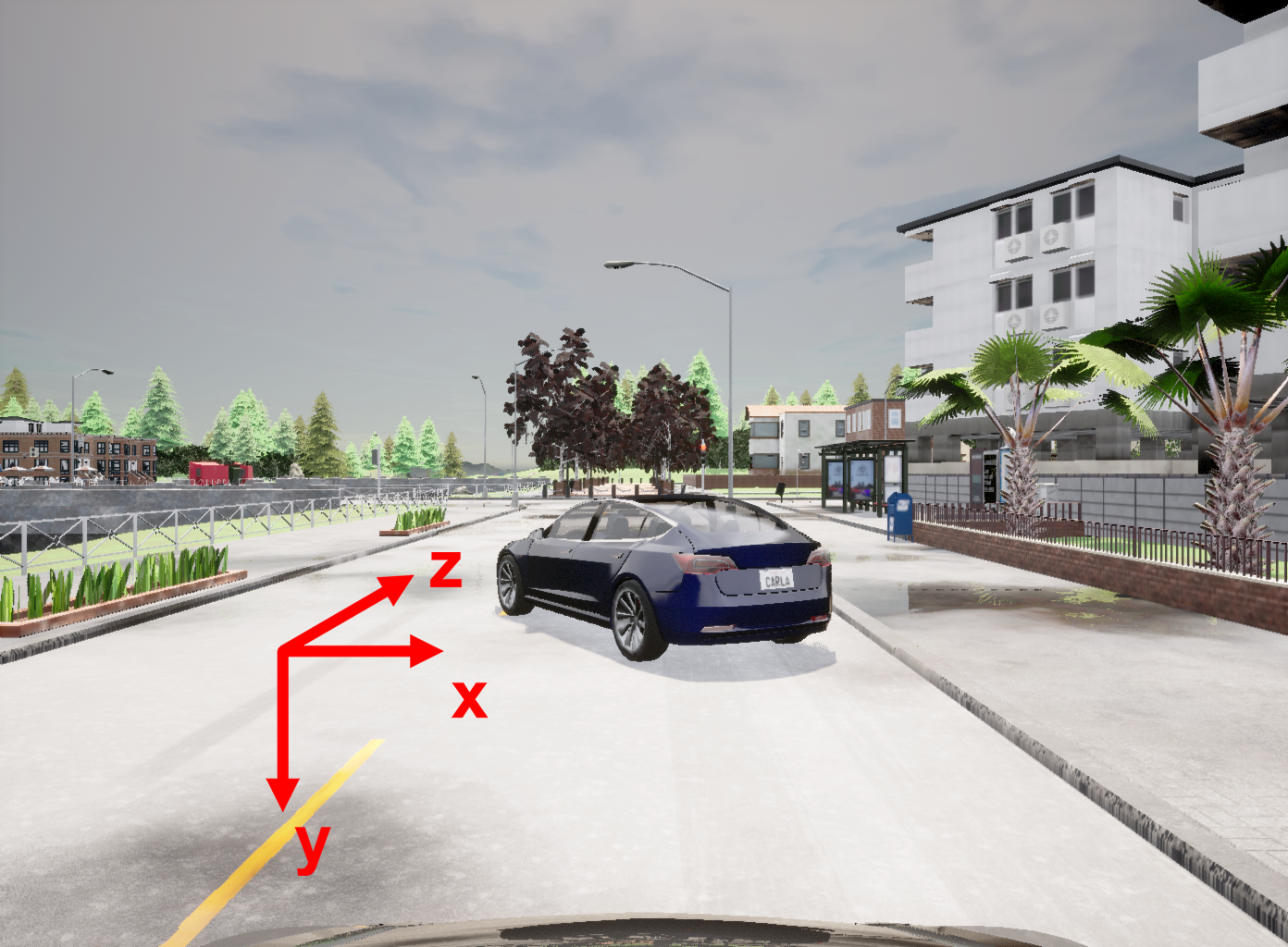}
    }
    \subcaptionbox{Shift along z-axis\label{fig:shift viz}}{
        \includegraphics[width=.22\linewidth]{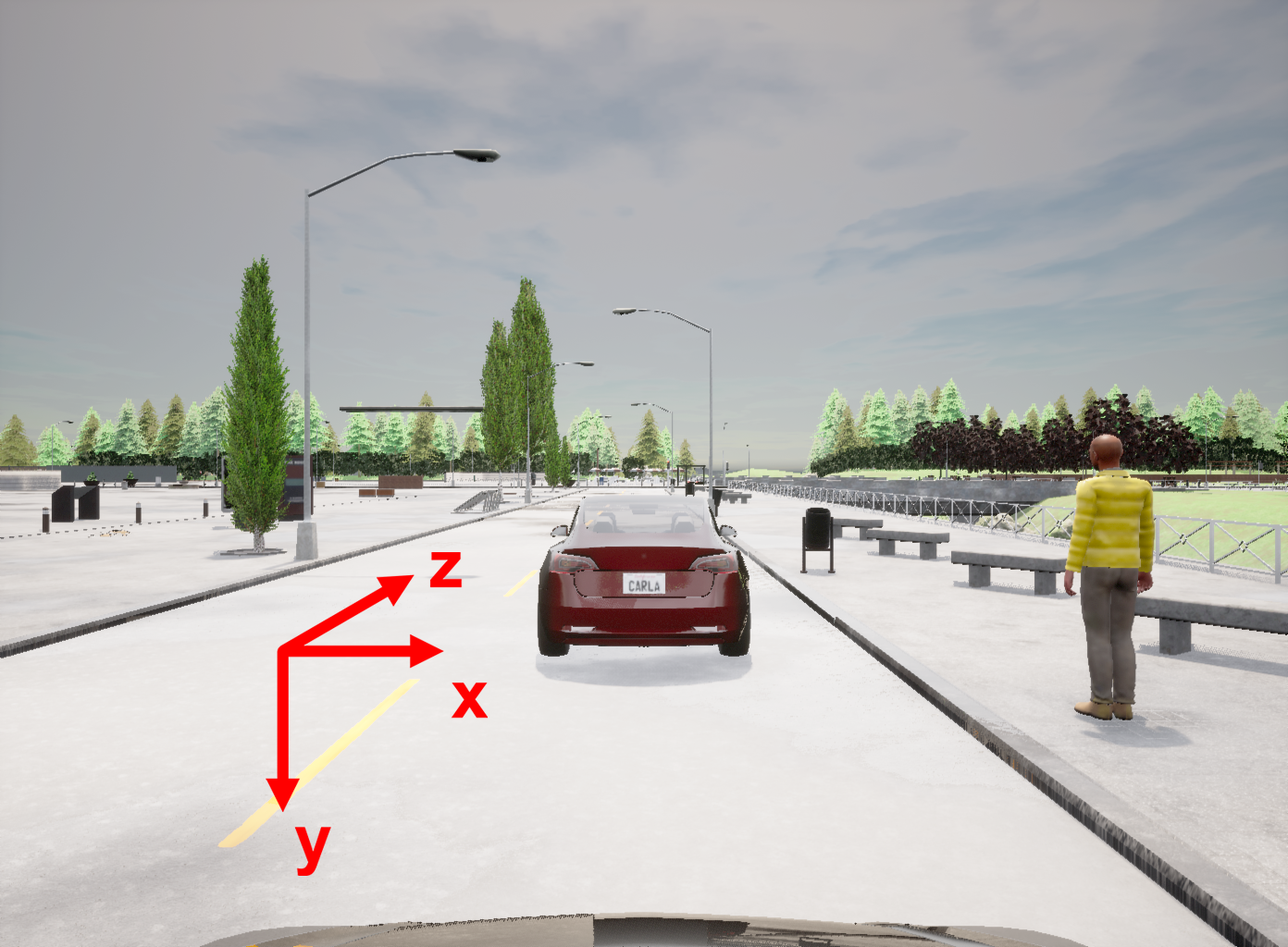}\quad
        \includegraphics[width=.22\linewidth]{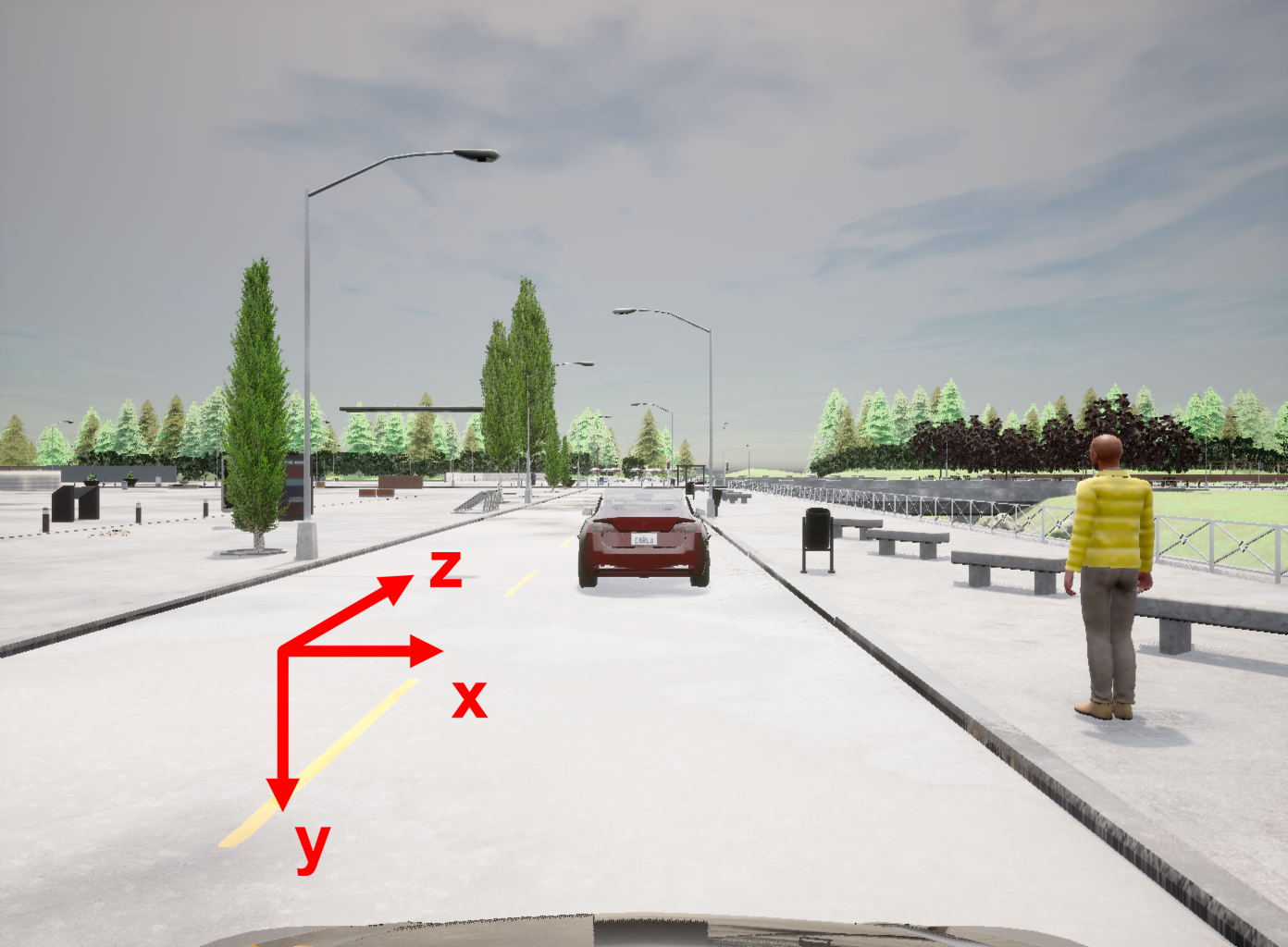}
    }
    % \subcaptionbox{coordinate\label{fig:coordinate}}{
    %     \includegraphics[width=.25\linewidth]{coordinate.pdf}
    % }
    \vspace{-1em}
    \caption{Visualization of the rotation and shifting transformations. The x-axis, y-axis, and z-axis point to the left, down, and forward respectively, while the original point is at the center of the bottom plane for the ego vehicle bounding box.}
    \label{fig:transformation_visualization}
\vspace{-10pt}
\end{figure}

% \input{fig_vis.tex}

% \begin{figure}[!t]
%     \linyi{need two figures here. In each figure, there are two vehicles, one before the transformation and one after the transformation. $x$, $y$, $z$ axes should be drawn.}
%     \caption{Rotation~(left) and shift~(right) transformation.}
%     \label{fig:transformation_visualization}
% \end{figure}

% We consider a hypothetical adversary that can apply certain transformations to an object. Suppose the transformations can be parameterized by $\mathbf z\in \mathcal Z\subseteq \mathbb R^m$. As images and point clouds are transformed simultaneously, we define attacks on a multi-sensor fusion system by $T = \{T_x, T_p\}$, where
% $T_x : \mathcal X\times \mathcal Z \to \mathcal X$ transforms images and $ T_p = \mathcal P \times \mathcal Z\to \mathcal P$ transforms point clouds.

% \linyi{Maybe we need some running example here.}

\paragraph{Fine partition assumption.}
For common transformations, we find that when the parameter space is partitioned into tiny subspaces with $\ell_\infty$ diameter smaller than some threshold $\tau$,  in each subspace bounded by $\ell_2$ norm, the distortion incurred by the transformation is upper bounded by the distortion with extreme points as transformation parameter.
We formally state such partition assumption and empirically verify it in \Cref{append:assum-eval}. 

\textbf{Robustness certification goal.}
Our goal is to certify that, no matter what transformation parameter is chosen by the adversary within a bounded constraint or what transformation strategy is used, the multi-sensor fusion system can always detect the object and locate the object precisely. 
% \bo{the parameter needs to be bounded right?}
Here we mainly focus on the task that the multi-sensor fusion system aims to detect the front vehicle when it is present.
Extensions to other tasks such as multi-object detection are straightforward via box alignment~\cite{chiang2020detection}.
Now, we formalize this certification goal by two criteria:
\emph{Given an input $(\vx, \vp)$ containing a front vehicle, a transformation $T$, and a constrained parameter space $\gS$\,
% footnote{\begin{rev}Suppose the valid parameter space of $T$ is $\gZ$, then the space to certify $\gS$ is a bounded and constrained subspace of $\gZ$, i.e., $\gS\subseteq\gZ$.\end{rev}}, 
for any transformed input $(T_x(\vx, \vz), T_p(\vp, \vz))$ with $\vz \in \gS$,} 

\vspace{-8pt}
\begin{itemize}[leftmargin=*]
    \item (Detection Certification) \emph{the multi-sensor fusion system always outputs a bounding box for the vehicle with confidence $\ge \eta$, where $\eta$ is a pre-defined threshold};

    \item (IoU Certification) \emph{the multi-sensor fusion system always outputs a bounding box for the vehicle whose volume IoU~(intersection over union) with the ground-truth bounding box $\ge$ some value $v$}.
\end{itemize}
\vspace{-8pt}

In the above criteria, $\eta$ determines whether the confidence is high enough to report ``vehicle detected'', which is usually set to $0.8$; the IoU is the standard for evaluating bounding box precision (i.e., given two 3D bounding boxes $B_1, B_2 \in \gB$, $\iou(B_1,B_2) = \frac{\vol(B_1\cap B_2)}{\vol(B_1\cup B_2)}$ denotes the ratio of intersection volume over the union volume). 

\subsection{Constructing Certifiably Robust MSFs via Smoothing}
    \label{subsec:certification-construction}
    Common multi-sensor fusion systems are challenging to be certified due to complex DNN architectures and fusion rules.
    Hence, we leverage the randomized smoothing~\cite{cohen2019certified}, in particular, median smoothing~\cite{chiang2020detection}, as the post-processing protocol to construct a \emph{smoothed} multi-sensor system.
    Formally, for each coordinate of the multi-sensor fusion system $g_i : \gX \times \gP \to \left((\gB \times \gC \times [0,1])^n\right)_i$, we add \emph{anisotropic Gaussian noise} to the input and define $q$-percentile of the resulting distribution of $g_i$:
    \begin{equation}
        \small
        \label{eq:median}
        {h_i}_q(\vx,\vp) = \sup\{y\in \mathbb R\ |\ \Pr[g_i(\vx+\delta_x,\vp+\delta_p)\leq y] \leq q\},
        % \vspace{2mm}
    \end{equation}
    where $\delta_x \sim \mathcal N(0,\sigma_x^2 \mI_d)$ and $\delta_p \sim \mathcal N(0,\sigma_p^2 \mI_{3\times N})$.
    We define the resulting smoothed multi-sensor fusion system $h_q := ({h_1}_q, {h_2}_q, \dots)$.
    In practice, we use finite $\delta_x$ and $\delta_p$ samples to approximate $h_q$ by $\hat h_q$ with high probability and deploy~($q$ is usually set to $0.5$ so it is called median smoothing). 
    For any $q$, we can obtain high-confidence intervals for $h_q$ via Monte-Carlo sampling~\cite{chiang2020detection}.

    Though existing work provides robustness certification for smoothed models~\cite{cohen2019certified,chiang2020detection,li2021tss,chu2022tpc}, such certification is limited to single-modal classification or regression against $\ell_p$-bounded perturbations.
    In contrast, our goal is to certify the robustness of multi-sensor fusion systems against semantic transformations under the two aforementioned criteria, where direct applications of prior work are infeasible due to heterogeneous input dimensions, intractable perturbation spaces, and unsupported certification criteria.
    In the following text, we introduce theoretical results that fulfill our robustness certification goal.

\subsection{General Detection Certification}
\label{sec:det}
% \textbf{Median Smoothing.} Given a base predictor $g:\mathcal X\times \mathcal P\to \mathbb R$, the percentile smoothing of $g$ is defined as
% \begin{equation}
% \label{eq:median}
%     h_q(\mathbf x,\mathbf p) = \sup\{y\in \mathbb R\ |\ \Pr[g(\mathbf x+\delta_x,\mathbf p+\delta_p)\leq y] \leq q\}.
% \end{equation}
% where $\delta_x \sim \mathcal N(0,\sigma_x^2 \mathbb I_d)$ and $\delta_p \sim \mathcal N(0,\sigma_p^2 \mathbb I_{3\times N})$.

    For detection certification, we locate the vehicle bounding box with the highest confidence and consider the confidence of this box as the detection confidence.
    Hence, 
    for notation simplicity, we let $g: \gX \times \gP \to [0,1]$ to represent this detection confidence of the multi-sensor fusion system.

\begin{theorem}
Let $T=\{T_x, T_p\}$ be a transformation with parameter space $\mathcal Z$. Suppose $\mathcal S \subseteq \mathcal Z$ and $\{\alpha_i\}_{i=1}^M\subseteq \mathcal S$. 
For detection confidence $g:\mathcal X \times \mathcal P \to [0,1]$, let $h_q(\vx, \vp)$ be the median smoothing of $g$ as defined in \cref{eq:median}.
Then for all transformations $\vz\in \gS$, the confidence score of the smoothed detector satisfies: 
% \vspace{-5pt}
\begin{equation}
    h_q(T_x(\vx,\vz), T_p(\vp, \vz)) \geq \min_{1\leq i\leq M} h_{\underline q}(T_x(\vx,\alpha_i), T_p(\vp, \alpha_i))
    \label{eq:detection-lower-bound}
\end{equation}
\vspace{-25pt}
\begingroup
\allowdisplaybreaks

\small
\begin{align}
    \text{where} \quad
    \underline q &= \Phi\left(\Phi^{-1}(q) - \sqrt{\frac{M_x^2}{\sigma_x^2} + \frac{M_p^2}{\sigma_p^2}}\right),\\
    M_x &= \max_{\alpha\in \mathcal S}\min_{1\leq i\leq M} \|T_x(\vx,\alpha) - T_x(\vx,\alpha_i)\|_2,
    % \label{eq:Mx} 
    %\\
    \quad 
    M_p = \max_{\alpha\in \mathcal S}\min_{1\leq i\leq M} \|T_p(\vp,\alpha) - T_p(\vp,\alpha_i)\|_2.\label{eq:Mp}
\end{align}
\normalsize
\endgroup
\label{thm:detection}
\end{theorem}

% \begin{proof}[Proof.]
%     % \TODO
    
% \end{proof}
\vspace{-10pt}
\begin{remark}
    A full proof for \cref{thm:detection} is in \cref{append:proof-det}.
    Suppose we have upper bounds for $M_x$ and $M_p$~(to be given in \Cref{lem:upper-bound-mx-mp}), we can compute a lower bound of $\underline q$, and a high-confidence lower bound of $h_{\underline q}(T_x(\vx, \alpha_i), T_p(\vp, \alpha_i))$ via Monte-Carlo sampling.
    As a result, we can compute a high-confidence lower bound of detection confidence $h_q$.
    By comparing it with $\eta$ in \Cref{subsec:setup}, we can derive the detection certification.
\end{remark}

\begin{lemma}
    If the parameter space to certify $\gS = [l_1, u_1] \times \cdots \times [l_m, u_m]$ is a hypercube satisfying the finite partition assumption~(\Cref{assum}) with threshold $\tau$, and $\{\alpha_i\}_{i=1}^M = \{\frac{K_1-k_1}{K_1}l_1 + \frac{k_1}{K_1}u_1: k_1 = 0,1,\dots,K_1\} \times \cdots \times \{\frac{K_m-k_m}{K_m}l_m + \frac{k_m}{K_m}u_m: k_m = 0,1,\dots,K_m\}$, where $K_i \ge \frac{u_i - l_i}{\tau}$, then
    % \begin{align*}
    %     \small
    %     M_x \le & \max_{0\le k_1 < K_1} \max_{0\le k_2 < K_2} \cdots \max_{0\le k_m < K_m} \nonumber \\
    %     & \hspace{-2em} \| T_x(\vx, \vw(k_1,\dots,k_m) ) - T_x(\vx, \vw(k_1,\dots,k_m) + \omega) \|_2, \\
    %     M_p \le & \max_{0\le k_1 < K_1} \max_{0\le k_2 < K_2} \cdots \max_{0\le k_m < K_m} \nonumber \\
    %     & \hspace{-2em} \| T_p(\vp, \vw(k_1,\dots,k_m) ) - T_p(\vp, \vw(k_1,\dots,k_m) + \omega) \|_2, 
    % \end{align*}
    \vspace{-5pt}
    \small 
    \begin{equation}
      \resizebox{\linewidth}{!}{$ \displaystyle
        M_x \leq \sum_{i=1}^m \max_{\vk\in \Delta}\Big\|T_x(\vx, \vw(\vk)) - T_x(\vx, \vw(\vk)+w_i)\Big\|_2,%\nonumber\\
        \quad
        M_p \leq \sum_{i=1}^m\max_{\vk\in \Delta}\Big\|T_p(\vp, \vw(\vk)) - T_p(\vp, \vw(\vk)+w_i)\Big\|_2
      $}
    \end{equation}
    \normalsize
    \vspace{-10pt}
    
    where $\Delta = \{(k_1,\dots, k_m)\in \mathbb Z^m\ |\ 0\leq k_i< K_i\}$ and $\vw(\vk) = ( \frac{K_1-k_1}{K_1} l_1 + \frac{k_1}{K_1} u_1, \cdots, \frac{K_m-k_m}{K_m} l_m + \frac{k_m}{K_m} u_m )$. $w_i = \frac{u_i - l_i}{K_i}\ve_i$, where $\ve_i$ is a unit vector at coordinate $i$.
    \label{lem:upper-bound-mx-mp}
\end{lemma}

\vspace{-5pt}
\begin{remark}
    This lemma splits each dimension of $\gS$ by a $\tau$-cover: $\{\frac{K_i-k_i}{K_i}l_i + \frac{k_i}{K_i}u_i: k_i = 0,1,\dots,K_i\}$. Hence, for each tiny subspace defined by $[\vw(\vk), \vw(\vk)+(w_1,\dots, w_m)]$, we can apply the finite partition assumption~(\Cref{assum}) and the lemma follows.
    A full proof is in \cref{proof:mx-mp}.
    The lemma provides feasible upper bounds~(via computing maximum of finite terms) for $M_x$ and $M_p$, so a lower bound of $\underline q$ is computable, and hence the robustness certification in \Cref{thm:detection} is computationally feasible.
\end{remark}

\subsection{General IoU Certification for 3D Bounding Boxes}
\label{sec:iou}
\textbf{Median smoothing for 3D bounding boxes.} Given a base 3D bounding box predictor for the front vehicle $g: \mathcal X \times \mathcal P \to \mathcal B$ with $\mathcal B \subseteq \mathbb R^6\times [0,2\pi]$ describing the geometric shape of the bounding box~(details in  \Cref{subsec:setup}), we denote by $h_q(\vx,\vp)$ the coordinate-wise median smoothing on the outputs of $g$ following \Cref{eq:median}.
% \linyi{added pointers to previous definitions}

First, by applying \Cref{thm:detection} on each coordinate of the bounding box from two sides, we obtain the intervals of bounding box coordinates after any possible transformation. 
% \linyi{maybe we can move Theorem 3.8 here}.\wenda{done}

\begin{theorem}
Let $T = \{T_x,T_p\}$ be a transformation with parameter space $\mathcal Z$. Suppose $\mathcal S \subseteq \mathcal Z$ and $\{\alpha_i\}_{i=1}^N\subseteq \mathcal S$. Let $g_i:\mathcal X \times \mathcal P \to (\mathcal B)_i$ be the $i$-th coordinate of a predicted bounding box of a multi-sensor fusion system, and $h_{iq}(\vx,\vp)$ be the median smoothing of $g_i$ as defined in \cref{eq:median}. Then for all transformations $z\in \mathcal S$, the $i$-th coordinate of the median smoothed bounding box predictor satisfies:
\begin{align}
    \min_{1\leq i \leq M} h_{i\underline q}(T_x(\vx,\alpha_i), T_p(\vp, \alpha_i)) &\leq h_{iq}(T_x(\vx, \vz),T_p(\vp, \vz)) %\nonumber\\[-0.2em]
    \leq \max_{1\leq i \leq M} h_{i\overline q}(T_x(\vx,\alpha_i),T_p(\vp, \alpha_i))
\end{align}
where
\begin{small}
\vspace{-10pt}
\begin{align}
    \underline q = \Phi\left(\Phi^{-1}(q) - \sqrt{\frac{M_x^2}{\sigma_x^2} + \frac{M_p^2}{\sigma_p^2}}\right),%\\
    \quad
    \overline q = \Phi\left(\Phi^{-1}(q) + \sqrt{\frac{M_x^2}{\sigma_x^2} + \frac{M_p^2}{\sigma_p^2}}\right).
\end{align}
\end{small}
with $M_x, M_p$ defined as \cref{eq:Mp}.
\end{theorem}

With the intervals of bounding box coordinates, we propose the following theorem for computing the lower bound of IoU between the output bounding box and the ground truth.

\begin{theorem}
\label{theorem:iou}
Let $\mathbf B$ be a set of bounding boxes whose coordinates are bounded. We denote the lower bound of each coordinate by $(\underline x, \underline y, \underline z, \underline w,\underline h , \underline l, \underline r)$ and upper bound by $(\bar x,\bar y,\bar z, \bar w, \bar h, \bar l, \bar r)$. Let $B_{gt} = (x,y,z,w,h,l,r)$ be the ground truth bounding box.
% and $B_0 = (x,y,z,w,h,l,r)\in \mathbf B$. 
Then for any $B_i \in \mathbf B$,
\begin{equation}
    \iou(B_i,B_{gt}) \geq \frac{h_1 \cdot \left(\underline l\underline w-\vol(\underline S \backslash S_{gt})\right)}{hwl + \bar h\bar w \bar l - h_2 \cdot \left(\bar l \bar w - \vol(\bar S \backslash S_{gt})\right)}
\end{equation}
where $\underline S, \bar S$ are convex hulls formed by $(\underline x,\underline z,\underline r,\bar x,\bar z,\bar r)$ with respect to $(\underline w,\underline l)$ and $(\bar w,\bar l)$ (details in \cref{append:proof-iou}), and $S_{gt} = (x,z,w,l,r)_{gt}$ is the projection of $B_{gt}$ to the $x-z$ plane.
% \vspace{-5pt}
\small
\begin{align}
    h_1 &= \max\big(\min_{y^\prime\in [\underline y, \bar y]} \min\{h, \underline h, \frac{h+\underline h}{2} - |y^\prime-y|\},0\big),
    % \nonumber\\[-.2em]
    h_2 = \max\big(\min_{y^\prime\in [\underline y, \bar y]} \min\{h, \bar h, \frac{h+\bar h}{2} - |y^\prime-y|\},0\big).
\end{align}
\normalsize
% \vspace{-5pt}
% \begin{align}
%     \underline S &= \mathrm{Conv}\left(\{\underline x-\underline d, \bar x +\underline d\} \otimes  \{\underline z-\underline d, \bar z -\underline d\}\right),\nonumber\\
%     \bar S &= \mathrm{Conv}\left(\{\underline x-\bar d, \bar x + \bar d\} \otimes  \{\underline z-\bar d, \bar z -\bar d\}\right)
% \end{align}
% \wenda{Moved the exact definition of convex hull to appendix as it's too complicated.}
% where $\underline d = \sqrt{\underline l^2 + \underline w^2}/2$, $\bar d = \sqrt{\bar l^2 + \bar w^2}/2$.
\end{theorem}
% \linyi{need proof sketch and figure illustration}
% \vspace{-20pt}
\begin{proof}[Proof Sketch.]
To prove a lower bound for the IoU between $B_i$ and the ground truth $B_{gt}$, we lower bound the volume of the intersection $B_i \cap B_{gt}$ and upper bound the volume of the union $B_i \cup B_{gt}$ separately.
% Given a fixed range for center and rotation angle $(x,y,z,r)$, we only need to consider the bounding box $B_i\in \mathbf B$ with the smallest and largest size $(w,h,l)$.
% the intersection is the smallest when the shape of bounding box is $(\underline w, \underline h, \underline l)$, while the union is largest when it is $(\bar w,\bar h, \bar l)$
We estimate the upper bound of union by $\vol(B_{gt})+\vol(B_{\max})-\min_{(x,y,z,r)} \vol(B(x,y,z,\bar w,\bar h,\bar l,r)\cap B_{gt})$.
We calculate $h_1$ and $h_2$ as the smallest possible intersection between $B_i$ and $B_{gt}$ along $y$ axis given height $\underline h$ and $\bar h$, respectively.
We then prove the lower bound of their intersection on the $x-z$ plane. We leverage the fact that $\vol(S\cap S_{gt}) = \vol(S) - \vol(S\backslash S_{gt})$ and upper bound the volume of $S\backslash S_{gt}$ by considering the convex hull that contains all possible bounding boxes with bounded $(x,z,r)$.
The full proof is in \Cref{append:proof-iou}.
\end{proof}

We illustrate the computing procedures for both detection and IoU certification in \Cref{append:methods}.

\subsection{Instantiating Certification for Rotation and Shifting}

In this section, we demonstrate how our certification framework works for concrete transformations.
Specifically, we discuss two of the most common transformations for vehicles---rotation and shifting. 
For rotation transformation, we consider a vehicle rotating around the vertical axis based on bounded angle $z$ within a radius $r$, i.e., $z\in [-r, r]$.
For shifting transformation, we consider a vehicle moving along the road based on bounded distance $z\in [a,a+2r]$ where $a$ is the original distance.

We instantiate \Cref{thm:detection} on certifying detection and \Cref{theorem:iou} on certifying IoU against both transformations by computing their interpolation errors $M_x$ and $M_p$ as defined in \Cref{eq:Mp}. We choose $\{\alpha_i\}_{i=1}^K = \{\frac{2i-K}{K}r\}_{i=1}^K$ according to \Cref{lem:upper-bound-mx-mp} and compute the  interpolation errors $M_x$ and $M_p$ for both transformations. 
% \wenda{$M_p$ can be more efficiently computed but for simplicity we may not discuss it?}
We then leverage $M_x$ and $M_p$ to derive the lower bound for the detecting confidence score and the IoU regarding the ground-truth bounding box based on \Cref{thm:detection,theorem:iou}.

% \begin{rev}
Although the certification procedure can be time-consuming due to space partitioning, the certification cost usually happens before deployment~(i.e., pre-deployment verification). After the model is deployed, the inference of smoothed inference is  efficient~\cite{cohen2019certified}.
It is an active field to further reduce the inference cost~\cite{horvath2022boosting}, and our framework can seamlessly integrate these advances.
% \end{rev}

\newpage
\section{Experimental Evaluation}
\vspace{-2mm}
\label{sec:exp}
% % \begin{table}[!t]
%   \begin{wraptable}{r}{0.5\linewidth}
%     \centering
%     % \vspace{-2em}
%     \caption{\small Certification data. 
%     % This table shows the structure of our certification data. 
%     Each row stands for one setting, and the columns ``vehicle color", ``building", ``pedestrian", and ``amount" represent the color of the car, whether buildings exist, whether a pedestrian exists, and the number of corresponding data respectively. \bo{never mentioned, put in appendix and REF to it}}
%     % \vspace{2mm}
%     \scalebox{0.8}{
%     \begin{tabular}{c|c|c|c}
%     \hline
%         vehicle color & building & pedestrian & amount \\\hline\hline
%         blue & yes & no & 15 \\
%         red & yes & no & 4 \\
%         red & yes & yes & 4 \\
%         black & yes & no & 4 \\
%         black & yes & yes & 4 \\
%         blue & no & no & 15 \\
%         red & no & no & 4 \\
%         red & no & yes & 4 \\
%         black & no & no & 4 \\
%         black & no & yes & 4 \\
%     \hline
%     \end{tabular}}
%     \label{tab:certification data}
%   \vspace{-2em}
% % \end{table}
% \end{wraptable}

In this section, we first construct a benchmark for evaluating certified robustness, then systematically evaluate our certification framework \name on several state-of-the-art MSFs.

\textbf{Dataset.} 
There is no established benchmark for certified robustness evaluation for multi-sensor fusion systems to our knowledge.
Hence, we construct a diverse dataset leveraging the CARLA simulator~\cite{Dosovitskiy17}.
% To overcome the difficulty during generating rotation and shifting transformation dataset in the physical world, we create a dataset by leveraging the CARLA simulator \cite{Dosovitskiy17}. 
We consider two types of transformation: 1) Rotation transformation, which is common in the real world since the relative orientation of the car in front of the ego vehicle frequently changes. %continuously all the time. 
2) Shifting transformation, which simulates the scenario where the distance between the front and the ego vehicle changes drastically within a short time.
% and it is always better to have a longer perception distance. 

% Here are the concrete description of our training, testing, and certification (rotation and shifting) data.
We provide details of training, testing, and certification data construction as below. %obtained by 2 transformations below.

\vspace{-8pt}
\begin{itemize}[noitemsep,leftmargin=*]
    \item \textbf{Training and testing data.} We generate our KITTI-CARLA dataset \cite{deschaud2021kitticarla} with $5,000$ frames in CARLA Town01 with $50$ pedestrians and $100$ vehicles randomly spawned, in which $3,500$ frames are used for training and $1,500$ frames are used for testing. 
    % \linyi{the following sentence does not belong here --- we are talking about dataset here} To adapt models with Gaussian noises, we augment our training and testing data with Gaussian noises with $\sigma=0.25$ and $\sigma=0.5$, which are used for rotation experiments and shift experiments respectively.
    \item \textbf{Rotation certification data.} We spawn our ego vehicle at 15 spawn points randomly chosen in CARLA Town01, and we then spawn a leading vehicle in front of the ego vehicle within rotation interval $[-30^\circ,30^\circ]$. To study the effect of car color and surrounding objects on the rotation robustness, we collect our rotation certification data with 3 different colors of the leading vehicle in 4 different settings (combination of with or without buildings + with or without pedestrians),
    % without pedestrians, with buildings with pedestrians, without buildings without pedestrians, and without buildings with pedestrians), 
    which is summarized in \Cref{tab:certification data} in \Cref{appendix:dataset}. 
    \item \textbf{Shifting certification data.} We spawn our ego vehicle at the same 15 spawn points mentioned above and then spawn a leading vehicle facing forward in front of the ego vehicle. 
    We choose $[10, 15]$ for the shifting intervals, since according to empirical experiments there is a performance difference among different models at a distance of around 11 meters from the ego vehicle.
    % We set the shifting intervals as [10, 15] since the performance difference between different models occur when the distance is around 11 meter from the ego vehicle empirically. 
    Similar to the rotation certification data, we use the same environment settings to study the effect of vehicle color, buildings, and pedestrians.
\end{itemize}

\vspace{-5pt}
In total, in our benchmark dataset, the certification data contains 62 scenarios for rotation and 62 scenarios for shifting.
We set the size of the image input to $64\times87$ following the standard setting.
% to make sure the $\epsilon$ distribution is certifiable.
% \vspace{-3pt}
% However, this preprocessing might degrade the performance of models which highly depend on the image information (e.g. MonoCon \cite{liu2022learning} and FocalsConv \cite{chen2022focal} in our case, CLOCs only uses the outputs of 2D detection models as candidates). 

\textbf{Models.} We choose two fusion models based on image and point clouds, which are highly ranked on the KITTI leaderboard: FocalsConv \cite{chen2022focal} and CLOCs \cite{pang2020clocs}. FocalsConv (Voxel R-CNN (Car) + multimodal) achieves 85.22\% 3D Average Precision (AP) on the moderate KITTI Car detection task and 100\% 3D AP on our KITTI-CARLA dataset. CLOCs (Faster RCNN \cite{ren2015faster} + SECOND \cite{yan2018second}) achieves 80.67\% 3D AP on moderate KITTI Car detection task and 100\% 3D AP on our KITTI-CARLA dataset.

To compare the performance between fusion models and single-modal models, we select a camera-based model--MonoCon \cite{liu2022learning}, and a LiDAR-based model--SECOND \cite{yan2018second}, which achieve 19.03\% and 78.43\% 3D AP respectively on the moderate KITTI Car detection task and 100\% on our KITTI-CARLA moderate car detection task. 

\textbf{Metrics.} We consider two metrics: detection rate and IoU. In detection certification, attackers aim to reduce the object detection confidence score to fool the detectors to detect nothing. We aim to certify the lower bound of the detection rate under a detection threshold, where \textbf{Det@80} means the ratio of detected bounding boxes with confidence score larger than 0.8. 
In IoU certification, we aim to lower bound the IoU between the detected bounding box and the ground truth bounding box when attackers are allowed to attack the IoU in a transformation space. As for the notation in all tables, \textbf{AP@50} means the ratio of detected bounding boxes whose IoU with the ground truth bounding boxes is larger than 0.5. In \Cref{fig:rotation-bounds-th} and \Cref{fig:shift-bounds-th}, we show corresponding results by choosing different detection or IoU thresholds.
% \bo{adversarial accuracy for vanilla models;
% adversarial accuracy for vanilla models}

\textbf{Certification Details.} To make the models adapt with Gaussian noise smoothed data, we train two sets of models with Gaussian augmentation~\cite{cohen2019certified} using noise variance $\sigma=0.25$ and $\sigma=0.5$.
For the ease of robustness certification,
for rotation certification, we use models trained with $\sigma=0.25$ to construct smoothed models;
for shifting certification, we use models trained with $\sigma=0.5$.
% where models trained with $\sigma=0.25$ is used in the rotation transformation experiments and models trained with $\sigma=0.5$ is used in the shifting transformation experiments for ease of robustness certification.
% to make sure the $\epsilon$ distribution is certifiable.
% \Linyi{$\epsilon$ not defined!}
Note that our framework allows using different $\sigma$ and sample strategies for image and point cloud data.
% , as long as the $\epsilon$ distribution is in the certifiable range.  

% \begin{table}[h]
%     \scriptsize
%     \centering
%     \begin{tabular}{c|c|c|c}
%         \hline
%         Model & Vanilla & $\sigma=0.25$ & $\sigma=0.5$ \\\hline\hline
%         MonoCon \cite{liu2022learning} & 100.00\% & 100.00\% & 65.18\% \\\hline
%         SECOND \cite{yan2018second} & 100.00\% & 72.73\% & 63.36\% \\\hline
%         CLOCs \cite{pang2020clocs} & 100.00\% & 50.00\% & 50.00\% \\\hline
%         FocalsConv \cite{chen2022focal} & 100.00\% & 100.00\% & 100.00\% \\\hline
%     \end{tabular}
%     \caption{Model benign performance when trained in different settings. Each row represents the corresponding model and columns ``Vanilla'', ``$\sigma=0.25$'', ``$\sigma=0.5$'' stands for training without augmentation, training in augmented noise with $\sigma=0.25$ and training in augmented noise with $\sigma=0.5$. The number in each cell is the \textbf{mAP@50} of the corresponding model in the corresponding training and testing setting.}
%     \label{tab:model training}
% \vspace{-10pt}
% \end{table}

% \linyi{experimental results for rotation and colorization certification respectively---align with the order of theorem presentation}

\begin{table*}[!t]
  % \tiny
  \fontsize{7}{7}\selectfont
  \centering
  \caption{Certified and empirical robustness of different models against different semantic transformations. 
  Each row represents the corresponding model and attack radius. 
  ``Benign'', ``Adv (Vanilla)'', ``Adv (Smoothed)'', and ``Certification'' stand for benign performance, vanilla models' performance under attacks, smoothed models' performance under attacks, and certified lower bound of performance under bounded transformations. 
  % Each column represents the results under different thresholds. 
  \textbf{Det@80} and \textbf{AP@50} mean that we use 0.8 and 0.5 as the thresholds of confidence score and IoU score.
  Results under other thresholds are in \Cref{subsubsec: detailed-experiments}.}
  % \vspace{1mm}
  \begin{subtable}[]{\linewidth}\centering
  \caption{Rotation}
  \vspace{-0.5em}
  \resizebox{0.85\linewidth}{!}{
  \begin{tabular}{c|c|c|cc|cc|cc|cc}
      \hline
      \multirow{2}{*}{Model} & Input & \multirow{2}{*}{Attack Radius} & \multicolumn{2}{c|}{Benign} & \multicolumn{2}{c|}{Adv (Vanilla)} & \multicolumn{2}{c|}{Adv (Smoothed)} & \multicolumn{2}{c}{Certification}\\
      & Modality & & \textbf{Det@80} & \textbf{AP@50} & \textbf{Det@80} & \textbf{AP@50} & \textbf{Det@80} & \textbf{AP@50} & \textbf{Det@80} & \textbf{AP@50} \\\hline\hline
      \multirowcell{5}{MonoCon\\\cite{liu2022learning}} & \multirowcell{5}{Image} & $|r|\leq10^\circ$ & \multirow{5}{*}{100.00\%} & \multirow{5}{*}{100.00\%} & 58.06\% & 56.45\% & 80.65\% & 82.26\% & 75.81\% & 0.00\% \\
                                                      & & $|r|\leq15^\circ$ & & & 58.06\% & 54.84\% & 80.65\% & 82.26\% & 75.81\% & 0.00\% \\
                                                      & & $|r|\leq20^\circ$ & & & 58.06\% & 53.23\% & 80.65\% & 74.19\% & 75.81\% & 0.00\% \\
                                                      & & $|r|\leq25^\circ$ & & & 45.16\% & 35.48\% & 80.65\% & 16.13\% & 75.81\% & 0.00\% \\
                                                      & & $|r|\leq30^\circ$ & & & 32.26\% & 0.00\% & 80.65\% & 3.23\% & 75.81\% & 0.00\% \\\hline
      \multirowcell{5}{SECOND\\\cite{yan2018second}} & \multirowcell{5}{Point Cloud} &  $|r|\leq10^\circ$ & \multirow{5}{*}{100.00\%} & \multirow{5}{*}{100.00\%} & 19.35\% & 96.77\% & 0.00\% & 100.00\% & 0.00\% & 100.00\% \\
                                                   & & $|r|\leq15^\circ$ & & & 19.35\% & 96.77\% & 0.00\% & 100.00\% & 0.00\% & 100.00\% \\
                                                   & & $|r|\leq20^\circ$ & & & 19.35\% & 96.77\% & 0.00\% & 100.00\% & 0.00\% & 100.00\% \\
                                                   & & $|r|\leq25^\circ$ & & & 1.61\% & 83.87\% & 0.00\% & 96.77\% & 0.00\% & 0.00\% \\
                                                   & & $|r|\leq30^\circ$ & & & 1.61\% & 51.61\% & 0.00\% & 54.84\% & 0.00\% & 0.00\% \\\hline
      \multirowcell{5}{CLOCs\\\cite{pang2020clocs}} & \multirowcell{5}{Image + \\ Point Cloud} & $|r|\leq10^\circ$ & \multirow{5}{*}{100.00\%} & \multirow{5}{*}{100.00\%} & 100.00\% & 90.32\% & 88.71\% & 100.00\% & 88.71\% & 100.00\% \\
                                                  & & $|r|\leq15^\circ$ & & & 100.00\% & 90.32\% & 66.13\% & 98.39\% & 66.13\% & 87.10\% \\
                                                  & & $|r|\leq20^\circ$ & & & 100.00\% & 88.71\% & 50.00\% & 98.39\% & 50.00\% & 69.35\% \\
                                                  & & $|r|\leq25^\circ$ & & & 20.97\% & 87.10\% & 50.00\% & 98.39\% & 50.00\% & 67.74\% \\
                                                  & & $|r|\leq30^\circ$ & & & 3.23\% & 80.65\% & 50.00\% & 98.39\% & 50.00\% & 53.23\% \\\hline
      \multirowcell{5}{FocalsConv\\\cite{chen2022focal}} & \multirowcell{5}{Image +\\ Point Cloud} & $|r|\leq10^\circ$ & \multirow{5}{*}{100.00\%} & \multirow{5}{*}{100.00\%} & 100.00\% & 96.77\% & 100.00\% & 100.00\% & 100.00\% & 0.00\% \\
                                                       & & $|r|\leq15^\circ$ & & & 100.00\% & 0.00\% & 100.00\% & 0.00\% & 100.00\% & 0.00\% \\
                                                       & & $|r|\leq20^\circ$ & & & 100.00\% & 0.00\% & 100.00\% & 0.00\% & 100.00\% & 0.00\% \\
                                                       & & $|r|\leq25^\circ$ & & & 100.00\% & 0.00\% & 100.00\% & 0.00\% & 100.00\% & 0.00\% \\
                                                       & & $|r|\leq30^\circ$ & & & 98.39\% & 0.00\% & 100.00\% & 0.00\% & 100.00\% & 0.00\% \\\hline
      \end{tabular}
  }
  \label{tab:rotation overview}
  \end{subtable}
  \begin{subtable}[]{\linewidth}\centering
  \vspace{1mm}
  \caption{Shifting}
  % \vspace{-0.5em}
  \resizebox{0.85\linewidth}{!}{
  \begin{tabular}{c|c|c|cc|cc|cc|cc}
      \hline
      \multirow{2}{*}{Model} & Input & \multirow{2}{*}{Attack Radius} & \multicolumn{2}{c|}{Benign} & \multicolumn{2}{c|}{Adv (Vanilla)} & \multicolumn{2}{c|}{Adv (Smoothed)} & \multicolumn{2}{c}{Certification}\\
      & Modality & & \textbf{Det@80} & \textbf{AP@50} & \textbf{Det@80} & \textbf{AP@50} & \textbf{Det@80} & \textbf{AP@50} & \textbf{Det@80} & \textbf{AP@50} \\\hline\hline
      \multirowcell{5}{MonoCon\\\cite{liu2022learning}} & \multirowcell{5}{Image} & $10\leq z\leq11$ & \multirow{5}{*}{100.00\%} & \multirow{5}{*}{100.00\%} & 66.13\% & 77.42\% & 66.13\% & 77.42\% & 64.52\% & 41.94\% \\
                                                      & & $10\leq z\leq12$ & & & 62.90\% & 74.19\% & 62.90\% & 74.19\% & 61.29\% & 1.61\% \\
                                                      & & $10\leq z\leq13$ & & & 56.45\% & 72.58\% & 56.45\% & 72.58\% & 51.61\% & 0.00\% \\
                                                      & & $10\leq z\leq14$ & & & 46.77\% & 33.87\% & 46.77\% & 33.87\% & 41.94\% & 0.00\% \\
                                                      & & $10\leq z\leq15$ & & & 27.42\% & 1.61\% & 27.42\% & 1.61\% & 27.42\% & 0.00\% \\\hline
      \multirowcell{5}{SECOND\\\cite{yan2018second}} & \multirowcell{5}{Point Cloud} & $10\leq z\leq11$ & \multirow{5}{*}{100.00\%} & \multirow{5}{*}{100.00\%} & 0.00\% & 93.55\% & 0.00\% & 100.00\% & 0.00\% & 0.00\% \\
                                                   & & $10\leq z\leq12$ & & & 0.00\% & 80.65\% & 0.00\% & 80.65\% & 0.00\% & 0.00\% \\
                                                   & & $10\leq z\leq13$ & & & 0.00\% & 80.65\% & 0.00\% & 80.65\% & 0.00\% & 0.00\% \\
                                                   & & $10\leq z\leq14$ & & & 0.00\% & 80.65\% & 0.00\% & 80.65\% & 0.00\% & 0.00\% \\
                                                   & & $10\leq z\leq15$ & & & 0.00\% & 80.65\% & 0.00\% & 80.65\% & 0.00\% & 0.00\% \\\hline
      \multirowcell{5}{CLOCs\\\cite{pang2020clocs}} & \multirowcell{5}{Image + \\ Point Cloud} & $10\leq z\leq11$ & \multirow{5}{*}{100.00\%} & \multirow{5}{*}{100.00\%} & 100.00\% & 93.55\% & 93.55\% & 100.00\% & 67.74\% & 79.03\% \\
                                                  & & $10\leq z\leq12$ & & & 100.00\% & 80.65\% & 93.55\% & 80.65\% & 66.13\% & 51.61\% \\
                                                  & & $10\leq z\leq13$ & & & 85.48\% & 80.65\% & 88.71\% & 80.65\% & 64.52\% & 48.39\% \\
                                                  & & $10\leq z\leq14$ & & & 64.52\% & 80.65\% & 85.48\% & 80.65\% & 62.90\% & 48.39\% \\
                                                  & & $10\leq z\leq15$ & & & 64.52\% & 80.65\% & 83.87\% & 80.65\% & 61.29\% & 48.39\% \\\hline
      \multirowcell{5}{FocalsConv\\\cite{chen2022focal}} & \multirowcell{5}{Image + \\ Point Cloud} & $10\leq z\leq11$ & \multirow{5}{*}{100.00\%} & \multirow{5}{*}{100.00\%} & 96.77\% & 0.00\% & 96.77\% & 100.00\% & 54.84\% & 0.00\% \\
                                                       & & $10\leq z\leq12$ & & & 96.77\% & 0.00\% & 96.77\% & 100.00\% & 4.84\% & 0.00\% \\
                                                       & & $10\leq z\leq13$ & & & 0.00\% & 0.00\% & 0.00\% & 82.26\% & 0.00\% & 0.00\% \\
                                                       & & $10\leq z\leq14$ & & & 0.00\% & 0.00\% & 0.00\% & 14.52\% & 0.00\% & 0.00\% \\
                                                       & & $10\leq z\leq15$ & & & 0.00\% & 0.00\% & 0.00\% & 8.06\% & 0.00\% & 0.00\% \\\hline
      \end{tabular}
  }
  \label{tab:shift overview}
  \end{subtable}
  \label{tab:overview}
  \vspace{-2.5em}
\end{table*}
% \vspace{-0.5em}

% \input{fig_threshold_rot}

\subsection{Certification against Rotation Transformation}
\label{subsec:rotation certification}

In this section, we present the evaluations for the certified and empirical results of our framework \name against rotation transformation. 
In terms of certification, we use small intervals of rotation angles $0.1^\circ$ and samples 1000 times with $\sigma_x=0.25, \sigma_p=0.25$ Gaussian noises for each interval (in total $600\times 1000$ Gaussian noises with certification confidence 95\%) to estimate $h_q$~(see definition in \Cref{subsec:certification-construction}).
Empirically, we split the rotation intervals into small $0.01^\circ$  and use the models' worst empirical performance in these samples as the empirical robustness against rotation attacks, which is equivalent to the PGD attack with step $0.01^\circ$.
We set the overall confidence of certification to be $95\%$, aligning with the setting in \cite{kang2022certifying}.

{
\renewcommand{\thesubfigure}{\alph{subfigure}}
% {\refstepcounter{subfigure}\textbf{(\thesubfigure) }{\ignorespaces #1}}

\begin{figure*}[!t]
% \vspace{-2em}

\newlength{\utilheightc}
\settoheight{\utilheightc}{\includegraphics[width=.160\linewidth]{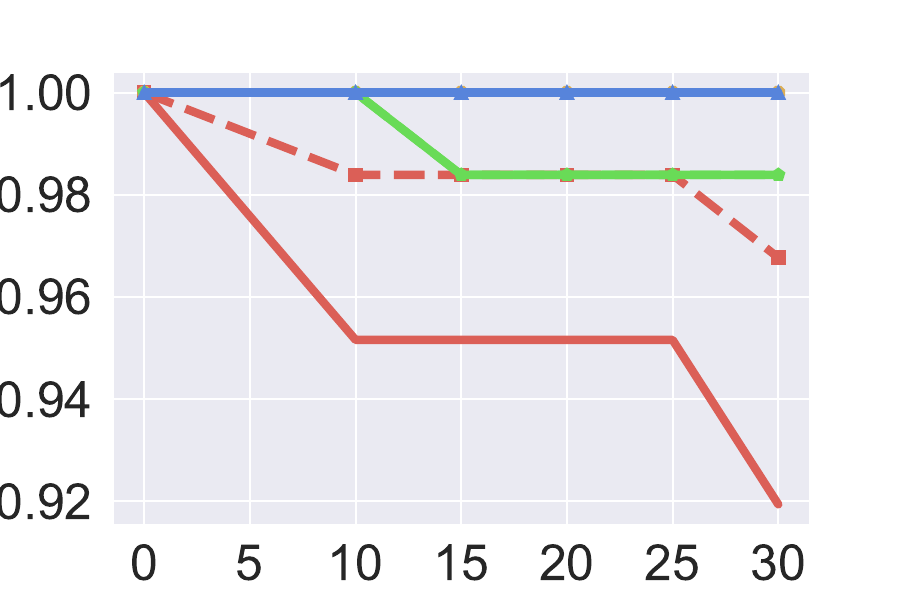}}%

\newlength{\utilheightd}
\settoheight{\utilheightd}{\includegraphics[width=.165\linewidth]{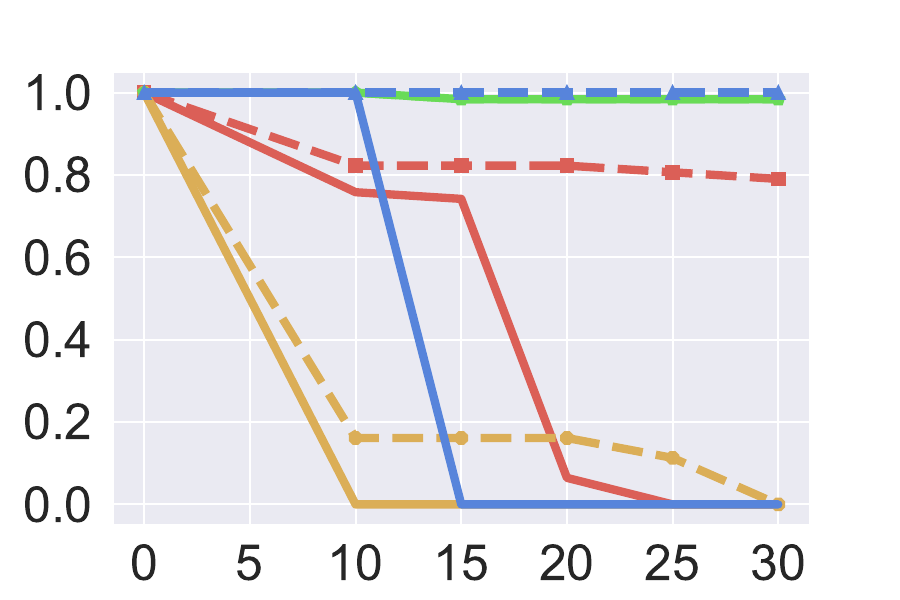}}%

% \newlength{\utilheightaa}
% \settoheight{\utilheightaa}{\includegraphics[width=.162\linewidth]{Freeway_adasearch_sigma-0.05.pdf}}%

% \newlength{\utilheightcp}
% \settoheight{\utilheightcp}{\includegraphics[width=.160\linewidth]{Pong_global_mean_sigma-0.001.pdf}}%

% \newlength{\utilheightdp}
% \settoheight{\utilheightdp}{\includegraphics[width=.165\linewidth]{Pong_global_median_sigma-0.001.pdf}}%

% \newlength{\utilheightaap}
% \settoheight{\utilheightaap}{\includegraphics[width=.162\linewidth]{Pong_adasearch_sigma-0.05.pdf}}%

% \newlength{\utilheight}
% \settoheight{\utilheight}{\includegraphics[width=.138\linewidth]{twitch-DE: low degree.pdf}}%

% \newlength{\attackheightb}
% \settoheight{\attackheightb}{\includegraphics[width=.138\linewidth]{twitch-DE: high degree.pdf}}%

\newlength{\legendheightb}
\setlength{\legendheightb}{0.48\utilheightc}%

\newcommand{\rownamec}[1]% #1 = text
{\rotatebox{90}{\makebox[\utilheightc][c]{\tiny #1}}}

\newcommand{\rownamed}[1]% #1 = text
{\rotatebox{90}{\makebox[\utilheightd][c]{\tiny #1}}}

\centering

{
\renewcommand{\tabcolsep}{10pt}

\begin{subtable}[th]{\linewidth}
\begin{tabular}{l}
\includegraphics[height=1.0\legendheightb]{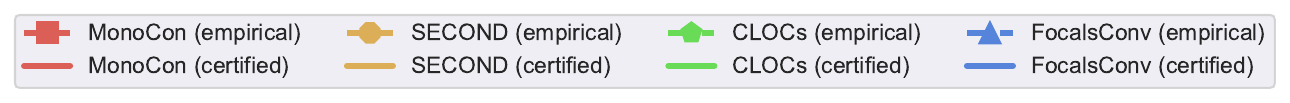}
\end{tabular}
\end{subtable}
% \vspace{-2pt}
\begin{subtable}{\linewidth}
\centering
\resizebox{\linewidth}{!}{%
\begin{tabular}{@{}p{3mm}@{}c@{}c@{}c@{}c@{}c@{}}
        & \makecell{\tiny{$|r|\leq 10^\circ$}}
        & \makecell{\tiny{$|r|\leq 15^\circ$}}
        & \makecell{\tiny{$|r|\leq 20^\circ$}}
        & \makecell{\tiny{$|r|\leq 25^\circ$}}
        & \makecell{\tiny{$|r|\leq 30^\circ$}}
        % & \makecell{\tiny{$\text{TH}_\text{IoU}=0.8$}}
        \vspace{-2pt}\\
\rownamec{\makecell{Detection}}&
\includegraphics[height=\utilheightc]{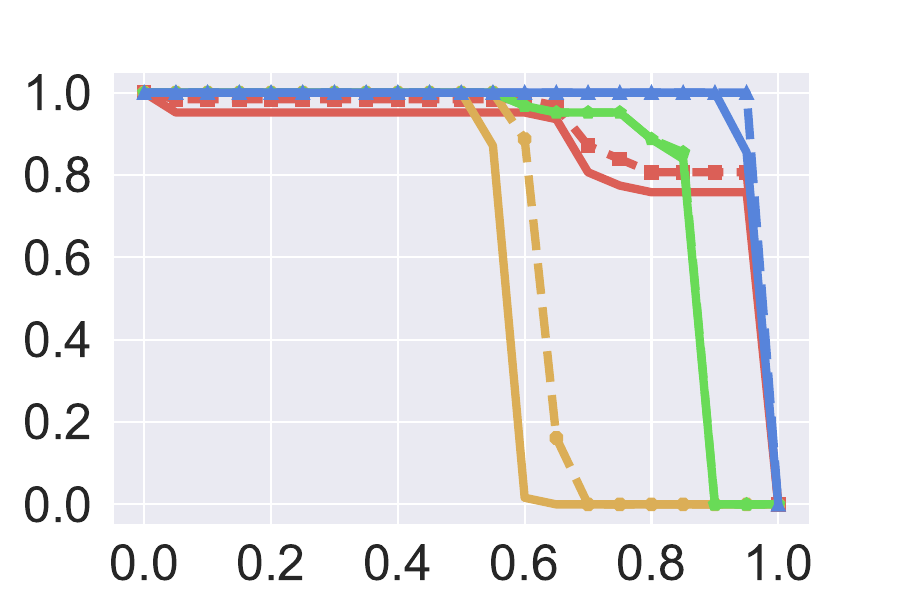}&
\includegraphics[height=\utilheightc]{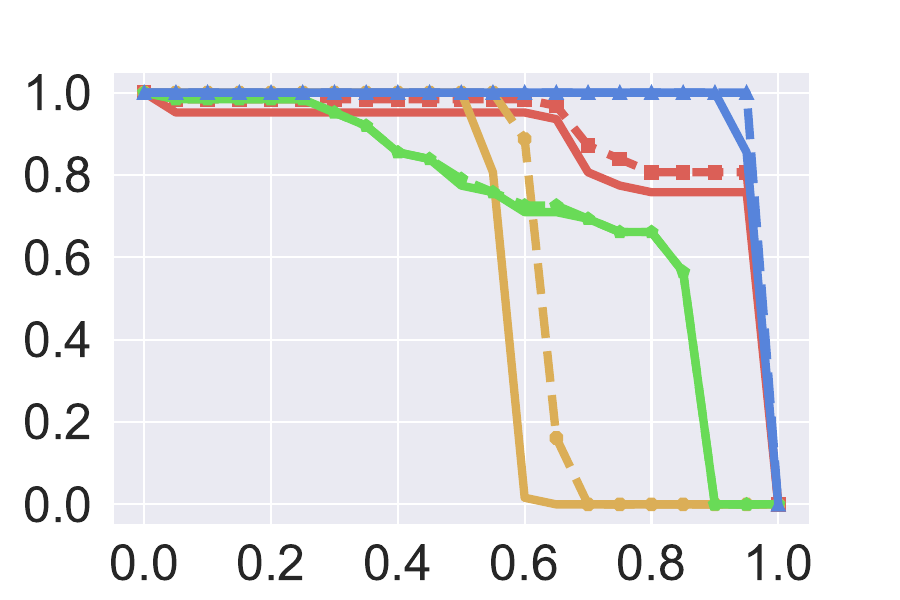}&
\includegraphics[height=\utilheightc]{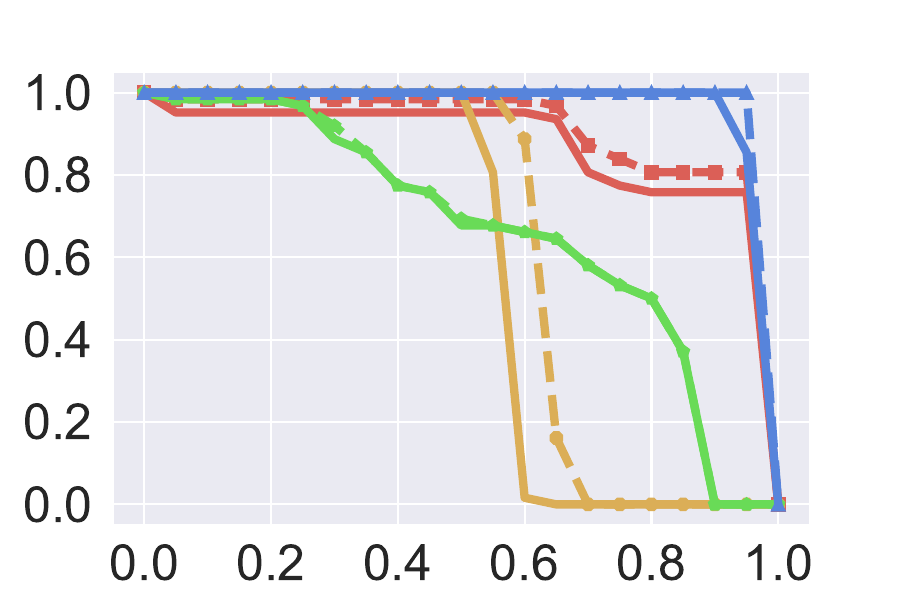}&
\includegraphics[height=\utilheightc]{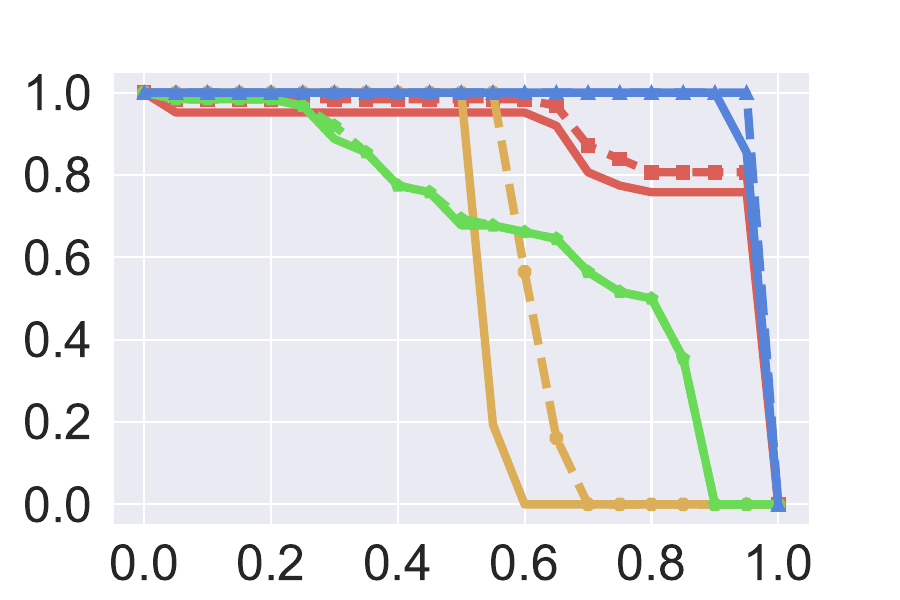}&
\includegraphics[height=\utilheightc]{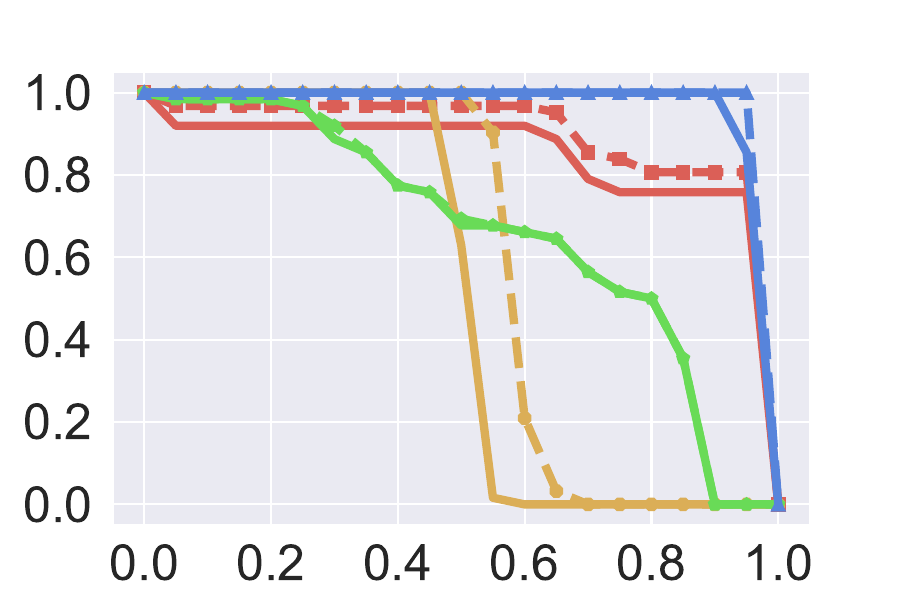}\\[-1.2ex]
        & \makecell{\tiny{$\text{TH}_\text{conf}$}}
        & \makecell{\tiny{$\text{TH}_\text{conf}$}}
        & \makecell{\tiny{$\text{TH}_\text{conf}$}}
        & \makecell{\tiny{$\text{TH}_\text{conf}$}}
        & \makecell{\tiny{$\text{TH}_\text{conf}$}}\\
\rownamec{\makecell{IoU}}&
\includegraphics[height=\utilheightc]{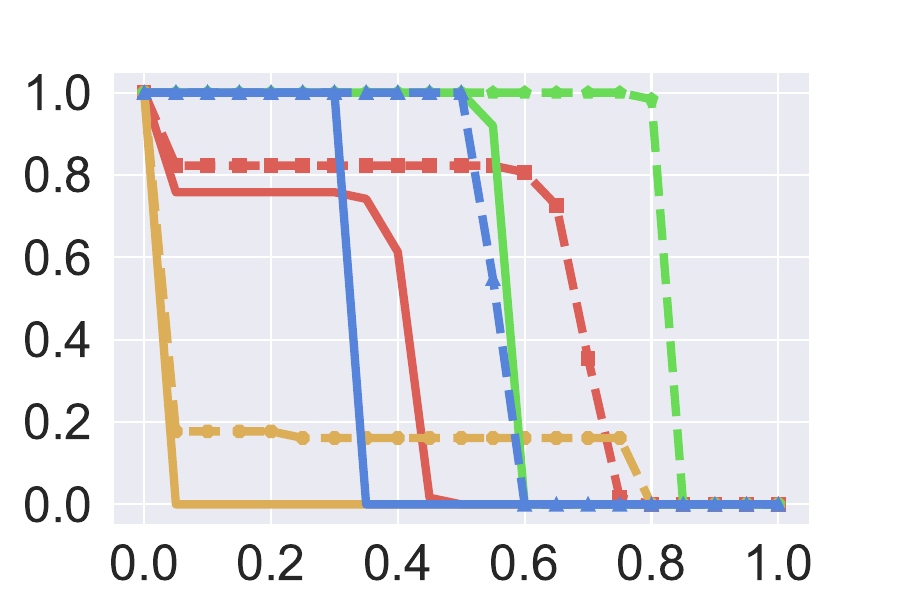}&
\includegraphics[height=\utilheightc]{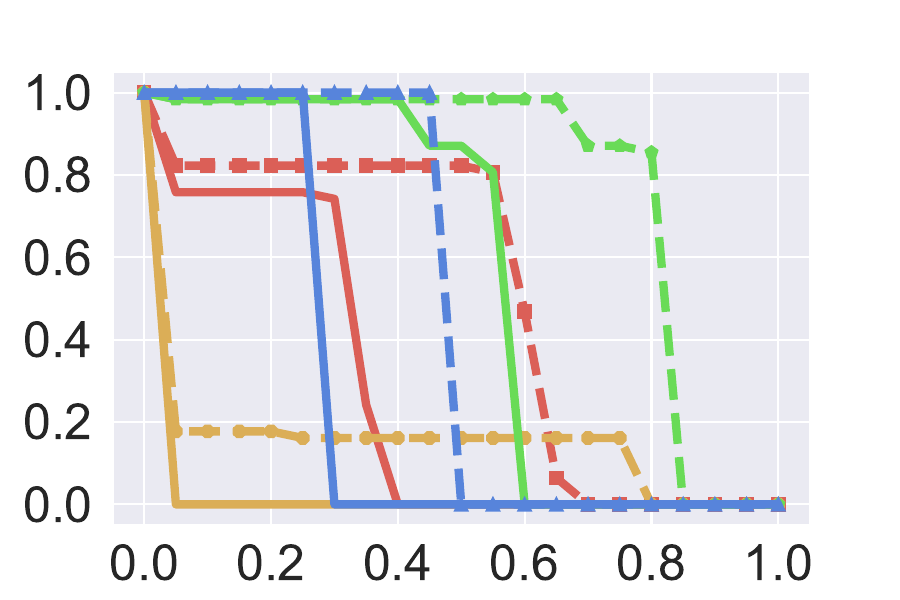}&
\includegraphics[height=\utilheightc]{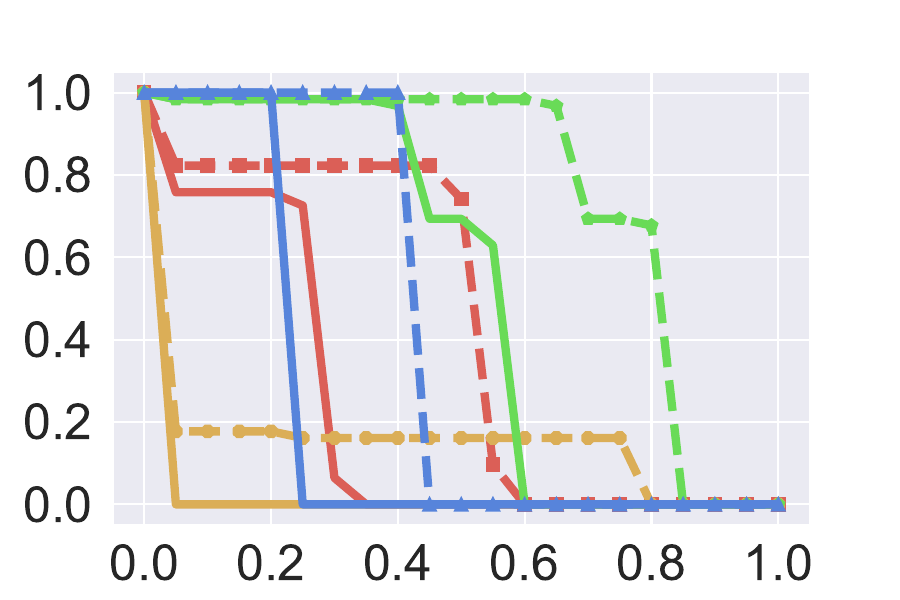}&
\includegraphics[height=\utilheightc]{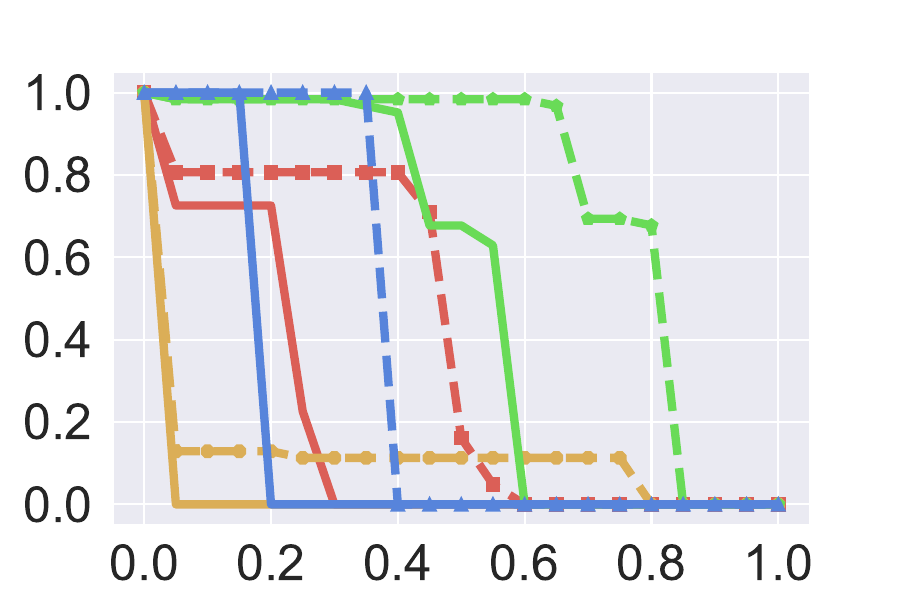}&
\includegraphics[height=\utilheightc]{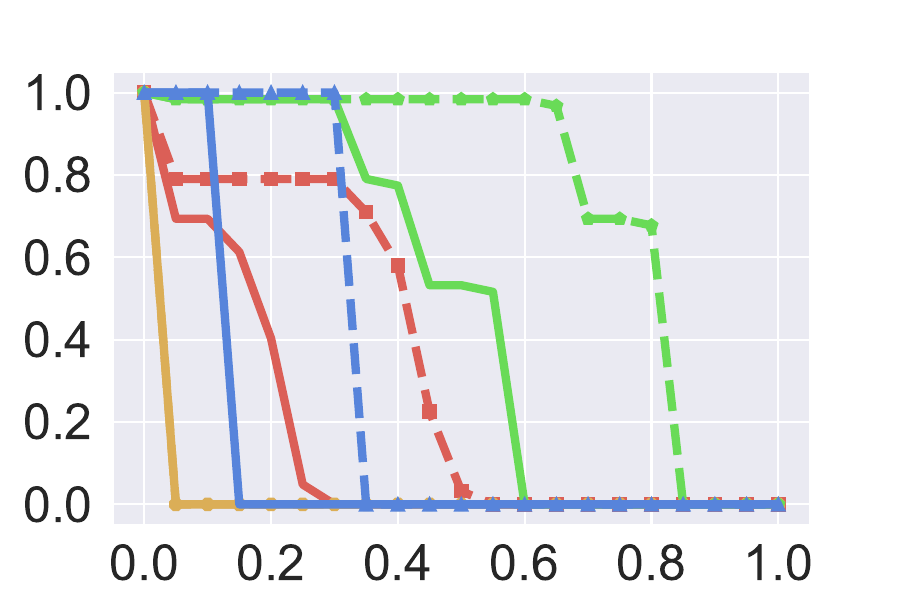}\\[-1.2ex]
        & \makecell{\tiny{$\text{TH}_\text{IoU}$}}
        & \makecell{\tiny{$\text{TH}_\text{IoU}$}}
        & \makecell{\tiny{$\text{TH}_\text{IoU}$}}
        & \makecell{\tiny{$\text{TH}_\text{IoU}$}}
        & \makecell{\tiny{$\text{TH}_\text{IoU}$}}
\end{tabular}
}
% \vspace{-3mm}
% \caption{\small Robustness certification for rotation transformation}
% \label{tab:bound-rotation-th}
\end{subtable}

}

\caption{Certified and empirical robustness on detection rate  and IoU  against rotation transformation (smoothing $\sigma=0.25$) under different thresholds. Solid lines represent the certified bounds, and dashed lines show the empirical performance under PGD attacks. $x$-axis represents the threshold for confidence score ($\text{TH}_\text{conf}$) and IoU score ($\text{TH}_\text{IoU}$), and $y$-axis represents the ratio of detection whose confidence / IoU score is larger than the confidence / IoU threshold. 
% \linyi{What does $x$-axis mean in each sub-figure? What does $y$-axis mean in each sub-figure?}
% \caption{\small Robustness certification for cumulative reward, including \textit{expectation bound} $\uje$, \textit{percentile bound} $\ujp$ ($p=50\%$), and \textit{absolute lower bound} $\uj$. Each column corresponds to one smoothing variance. Solid lines represent the certified reward bounds, and dashed lines show the empirical performance under PGD.
% shows the change of the lower bounds w.r.t. the attack magnitude $\eps$.
% The solid curves represent the certified lower bounds, while the dotted lines represent the empirical values.
}%
\label{fig:rotation-bounds-th}
% \vspace{-8pt}
\end{figure*}
}
% to do inference on models training with Gaussian noise augmentation in $\sigma_x=0.25, \sigma_p=0.25$. When computing the lower bound and upper bound of parameters, we apply a 95\% confidence level, which aligns with the setting in \cite{kang2022certifying}.

\Cref{tab:rotation overview} shows the results in the $[-30^\circ, 30^\circ]$ rotation interval. 
% where each row represents results for each model and the columns stand for the benign performance of the vanilla model, vanilla models' performance under rotation attacks, smoothed models' performance under rotation attacks and the certified lower bound of smoothed models' performance. In each cell, the number before the slash is the value for confidence score and the number after the slash is the IoU value. 
We can see that in terms of detection ability, the robustness order is $\text{FocalsConv} > \text{MonoCon} > \text{CLOCs} > \text{SECOND}$, according to both the empirical and certification results. 
Hence, FocalsConv and MonoCon may be more likely to predict the existence of the object when it exists. Furthermore, we observe that multi-sensor fusion models have better detection robustness compared with single-modal models under the same threshold (e.g., $\text{FocalsConv}>\text{MonoCon},\text{CLOCs}>\text{SECOND}$).
In addition, we find that our certification is pretty tight in many cases. In particular, row ``Certification" serves as the lower bound of row ``Adv (Smoothed)", and in most cases they are very close, indicating the tight robustness certification. 

Now we study the IoU metric since we expect that models can predict not only with high confidence but also precise bounding boxes. By comparing the empirical and certified results in IoU metric, CLOCs outperforms all other models. It is easy to understand that CLOCs outperforms single-modal models since it combines the information from both images and point clouds. However, FocalsConv is not as robust as CLOCs even though it is a also camera-LiDAR fusion model, which could be due to the fusion mechanism, and thus more robust fusion algorithms will help improve the model prediction robustness. 

We also present some interesting findings of rotation transformation in \Cref{subsubsec: detailed-experiments} (e.g., the effect of threshold, the effect of attack radius) and possible reasons for detection failure cases in rotation transformation in \Cref{subsubsec: examples}.
% \bo{what's failure? attack fail or defense fail?}

% \Cref{tab:rotation overview 1} shows the results under the condition ``color 1 with building no pedestrian", and the sub-tables from top to bottom represent the results in the rotation intervals $[-10^\circ,10^\circ], [-20^\circ,20^\circ], [-30^\circ,30^\circ]$ respectively, in which we can observe that all models degrade when the interval grows and the same conclusion about models' robustness order in \Cref{tab:rotation overview 1} holds here. 

% \linyi{conclusion: which model is more robust?}
\vspace{-0.5em}
\subsection{Certification against Shifting Transformation}
\label{subsec:shifting certification}

% \linyi{mention certifying confidence (99\% used in Maurice Gramian, 95\% used in Mintong's certified fairness, here we can cite if our confidence level is the same as these)}

% \linyi{threat model: change all channels, delta is $\ell_p$-norm bounded}

% \linyi{assumption: original car change $\delta$ results in the uniform same change of car in the camera-rendered images. With this assumption, the assumption in Section 3.1 naturally holds.}

% \linyi{conclusion: which model is more robust?}

Here we present the evaluations for the certified and empirical results of our framework \name against shifting transformation.
In terms of the robustness certification, we use small shifting intervals of size 0.01 and samples 1000 times with $\sigma_x=0.5, \sigma_p=0.5$ Gaussian noises in each interval to estimate $h_q$~(see definition in \Cref{subsec:certification-construction}).
In the empirical experiments, we divide the shifting intervals into small intervals of size 0.001 and use the worst empirical performance of the model among these samples as the empirical robustness against PGD attacks. 
We set the overall confidence of certification to be $95\%$ following the standard setting.
% do inference on models training with Gaussian noise augmentation in $\sigma_x=0.5, \sigma_p=0.5$. When computing the lower bound and upper bound of parameters, we apply a 95\% confidence level to align the setting in the rotation experiments.
% we apply 95\% confidence level and smoothing Gaussian noises with $\sigma_x=0.5, \sigma_p=0.5$ on models trained in augmentation noises with $\sigma_x=0.5, \sigma_p=0.5$ to make sure the epsilon distribution is in the certifiable range. 

\Cref{tab:shift overview} shows the certified and empirical robustness of different models against shifting transformation. The robustness  in terms of both detection ability and IoU is $\text{CLOCs} > \text{SECOND} \approx \text{MonoCon} > \text{FocalsConv}$. We  also notice that SECOND outperforms MonoCon when the distance is larger while they have a similar performance within short distances, which could be due to the accurate estimation of large distances by LiDAR sensors and the lack of depth information in the 2D camera images. However, this does not mean that image data do not have meaningful features because CLOCs is always more robust than SECOND, which could be caused by the fact that more candidates are proposed by 2D detectors (Faster RCNN in our case) which are ignored by the point cloud detectors. On the other hand, there is an interesting finding that FocalsConv performs poorly against shifting transformation. The reason might be that FocalsConv highly depends on the image features, and shifting transformation can attack the image and point cloud spaces at the same time. This implies that the design of the fusion mechanism is also an important factor on the robustness of multi-modal sensor fusion models. 

We also present some interesting findings under shifting transformation in \Cref{subsubsec: detailed-experiments} (e.g., the effects of threshold, the effect of attack radius) and possible reasons for detection failure cases in shifting transformation in \Cref{subsubsec: examples}.

\vspace{-0.5em}
\section{Conclusion}
\vspace{-0.5em}
In this work we provide the first robustness certification framework \name for multi-sensor fusion systems against semantic transformations based on different criteria. Our theoretical certification framework is flexible for different  models and transformations.
Our evaluations show that current fusion models are more robust than single-modal models, and the design of the fusion mechanism is an important factor in improving the robustness against semantic transformations.
% \subsection{Ablation Studies}
% \zijian{Maybe move to appendix?}

%     \subsubsection{Examples}

\clearpage
\bibliography{neurips_2023}
\bibliographystyle{plain}

%%%%%%%%%%%%%%%%%%%%%%%%%%%%%%%%%%%%%%%%%%%%%%%%%%%%%%%%%%%%

\newpage
\appendix
% \section{Appendix}

% Optionally include extra information (complete proofs, additional experiments and plots) in the appendix.
% This section will often be part of the supplemental material.

The appendices are organized as follows:
\begin{itemize}
    \item In \Cref{append:assum-eval}, we formally present the finite partition assumption~(from \Cref{subsec:setup}) and provide an empirical verification of the assumption.
    \item In \Cref{append:proofs}, we present the detailed proofs for lemmas and theorems in \Cref{sec:det,sec:iou} of the main paper.
    \item In \Cref{append:methods}, we present some additional details of our two certification strategies: \Cref{subsec:detection cert} for certifying the detection rate, and \Cref{subsec:IoU cert} for certifying the IoU between the detection and the ground truth. We include the detailed algorithm description and the complete pseudocode for each algorithm.  We will make our implementation public upon acceptance. 
    \item In \Cref{append:exp}, we show dataset details (\Cref{appendix:dataset}), detailed experimental evaluation (\Cref{subsubsec: detailed-experiments}), some ablation studies on sample strategies (\Cref{subsubsec: sample-strategy}) and smoothing parameter $\sigma$ (\Cref{subsubsec: smoothing-parameter}), and failure case analysis (\Cref{subsubsec: examples}). 
\end{itemize}

\section{Details of Fine Partition Assumption}
\label{append:assum-eval}

As introduced in \Cref{subsec:setup}, we impose the following assumption for the transformation.
\begin{assumption}
\label{assum}
For given transformation $T = \{T_x, T_p\}$ with parameter space $\gZ \subseteq \sR^m$,
there exists a small threshold $\tau > 0$, 
for any polytope of the parameter space $\Zsub \subseteq \gZ$ whose $\ell_\infty$ diameter is smaller than $\tau$, i.e., $\mathrm{diam}_\infty (\gZ_\mathrm{sub}) < \tau$, when parameters are picked from the subspace, the pairwise $\ell_2$ distance between transformed outputs is upper bounded by maximum pairwise $\ell_2$ distance with extreme points picked as parameters.
Formally,
let $E(\Zsub)$ be the set of extreme points of $\Zsub$, then $\forall \vz_1, \vz_2 \in \Zsub, \vx \in \gX, \vp \in \gP$,
\vspace{-3pt}
\begin{equation}
    \label{eq:assum}
    \small
    \begin{aligned}
        % & \forall \vz_1, \vz_2 \in \Zsub, \vx \in \gX, \vp \in \gP, \\
        & \|T_x(\vx,\vz_1) - T_x(\vx,\vz_2)\|_2 \le \max_{\vz'_1, \vz'_2 \in E(\gZ_{\mathrm{sub}})} \|T_x(\vx,\vz'_1) - T_x(\vx,\vz'_2)\|_2, \\
        & \|T_p(\vp,\vz_1) - T_p(\vp,\vz_2)\|_2 \le \max_{\vz'_1, \vz'_2 \in E(\gZ_{\mathrm{sub}})} \|T_p(\vp,\vz'_1) - T_p(\vp,\vz'_2)\|_2. \\
    \end{aligned}
\end{equation}
% For sufficiently small transformation space partitions $[z_{l}, z_{u}]$, the distance from transformed input to the endpoint input is monotonically non-decreasing, i.e.,
% \begin{align}
%     \theta &\mapsto \| T_x (x,z_l) - T_x(x,z_l+\theta (z_u-z_l))\|,\theta \in [0,1]\\
%     \theta & \mapsto \| T_p (p,z_l) - T_p(p,z_l+\theta (z_u-z_l))\|, \theta\in [0,1]
% \end{align}
% are monotonically non-decreasing.
\end{assumption}

\begin{remark} 
    Intuitively, the assumption states that, within a tiny subspace of the parameter space, the displacement incurred by the transformation, when measured by Euclidean distance, is proportional to the magnitude between parameters, so the maximum displacement can be upper bounded by the displacement incurred by choosing extreme points as the transformation parameters.
    Taking the rotation as an example, within a sufficiently small range of rotation angle $[r - \Delta, r + \Delta]$ where $2\Delta < \tau$, the assumption means that, the difference between rotated images $\|T_x(\vx,\delta_1) - T_x(\vx,\delta_2)\|_2$ and point clouds $\|T_p(\vp,\delta_1) - T_p(\vp,\delta_2)\|_2$ is no larger than $\|T_x(\vx,r-\Delta) - T_x(\vx,r+\Delta)\|_2$ and $\|T_p(\vp,r-\Delta) - T_p(\vp,r+\Delta)\|_2$ respectively. 
    % We empirically verify this assumption in practice in \Cref{append:assum-eval}. 
    % \linyi{I recommend that we verify this assumption and report the result if time permits.}
\end{remark}

While it is hard to prove the \Cref{assum}, we empirically evaluate the \Cref{assum} by plotting the distribution of image $\ell_2$ norm with different interval sizes ($0.001^\circ,0.01^\circ,0.02^\circ,0.03^\circ,0.04^\circ,0.05^\circ$ for rotation and $0.001,0.01,0.02,0.03,0.04,0.05$ for shifting) in randomly selected big intervals ($0.06^\circ$ for rotation and 0.07 for shifting) in \Cref{fig:transformation l2}. 

From \Cref{fig:rotation l2} and \Cref{fig:shift l2}, we can notice that with larger rotation and shifting intervals, the image $\ell_2$ norm becomes larger and larger, and the $\ell_2$ distance between the endpoints of each big interval can bound the $\ell_2$ distance between randomly chosen points in that big interval, which means that the pairwise $\ell_2$ distance picks the maximum value with extreme points when the transformation intervals are sufficiently small, and thus \Cref{assum} is empirically confirmed.

% \begin{figure}[h]
%     \centering
%     \subcaptionbox{Rotation $\ell_2$ norm distribution. Left to right: $\ell_2$ norm distribution of rotation interval $0.01^\circ, 0.02^\circ, 0.05^\circ, 0.1^\circ$\label{fig:rotation l2}}{
%         \includegraphics[width=.24\linewidth]{rotation_0.01_l2_norm_distribution.pdf}\quad
%         \includegraphics[width=.24\linewidth]{rotation_0.02_l2_norm_distribution.pdf}\quad
%         \includegraphics[width=.24\linewidth]{rotation_0.05_l2_norm_distribution.pdf}\quad
%         \includegraphics[width=.24\linewidth]{rotation_0.1_l2_norm_distribution.pdf}
%     }
%     \subcaptionbox{Shift $\ell_2$ norm distribution. Left to right: $\ell_2$ norm distribution of shift interval $0.001, 0.01, 0.05, 0.1$\label{fig:shift l2}}{
%         \includegraphics[width=.24\linewidth]{shift_0.001_l2_norm_distribution.pdf}\quad
%         \includegraphics[width=.24\linewidth]{shift_0.01_l2_norm_distribution.pdf}\quad
%         \includegraphics[width=.24\linewidth]{shift_0.05_l2_norm_distribution.pdf}\quad
%         \includegraphics[width=.24\linewidth]{shift_0.1_l2_norm_distribution.pdf}
%     }
%     \caption{Rotation and shift transformation $\ell_2$ norm distribution.}
%     \label{fig:transformation l2}
% % \vspace{-15pt}
% \end{figure}
{
\renewcommand{\thesubfigure}{\alph{subfigure}}
% {\refstepcounter{subfigure}\textbf{(\thesubfigure) }{\ignorespaces #1}}

\begin{figure*}[ht]
% \vspace{-2em}

% \newlength{\utilheightc}
\settoheight{\utilheightc}{\includegraphics[width=.160\linewidth]{rotation_conf_0.2.pdf}}%

% \newlength{\utilheightd}
\settoheight{\utilheightd}{\includegraphics[width=.165\linewidth]{rotation_IoU_0.3.pdf}}%

% \newlength{\utilheightaa}
% \settoheight{\utilheightaa}{\includegraphics[width=.162\linewidth]{Freeway_adasearch_sigma-0.05.pdf}}%

% \newlength{\utilheightcp}
% \settoheight{\utilheightcp}{\includegraphics[width=.160\linewidth]{Pong_global_mean_sigma-0.001.pdf}}%

% \newlength{\utilheightdp}
% \settoheight{\utilheightdp}{\includegraphics[width=.165\linewidth]{Pong_global_median_sigma-0.001.pdf}}%

% \newlength{\utilheightaap}
% \settoheight{\utilheightaap}{\includegraphics[width=.162\linewidth]{Pong_adasearch_sigma-0.05.pdf}}%

% \newlength{\utilheight}
% \settoheight{\utilheight}{\includegraphics[width=.138\linewidth]{twitch-DE: low degree.pdf}}%

% \newlength{\attackheightb}
% \settoheight{\attackheightb}{\includegraphics[width=.138\linewidth]{twitch-DE: high degree.pdf}}%

% \newlength{\legendheightb}
\setlength{\legendheightb}{0.48\utilheightc}%

\newcommand{\rownamec}[1]% #1 = text
{\rotatebox{90}{\makebox[\utilheightc][c]{\tiny #1}}}

\newcommand{\rownamed}[1]% #1 = text
{\rotatebox{90}{\makebox[\utilheightd][c]{\tiny #1}}}

\centering

{
\renewcommand{\tabcolsep}{10pt}

\begin{subtable}{\linewidth}
\centering
\resizebox{\linewidth}{!}{
\begin{tabular}{@{}p{3mm}@{}c@{}c@{}c@{}c@{}c@{}c@{}}
\rownamec{\makecell{$\ell_2$ norm}}&
\includegraphics[height=\utilheightc]{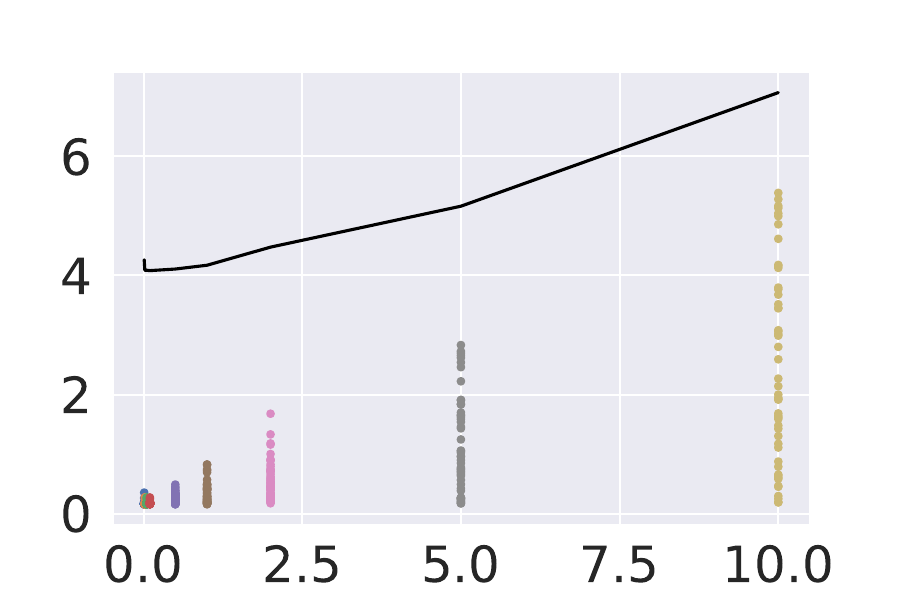}&
\includegraphics[height=\utilheightc]{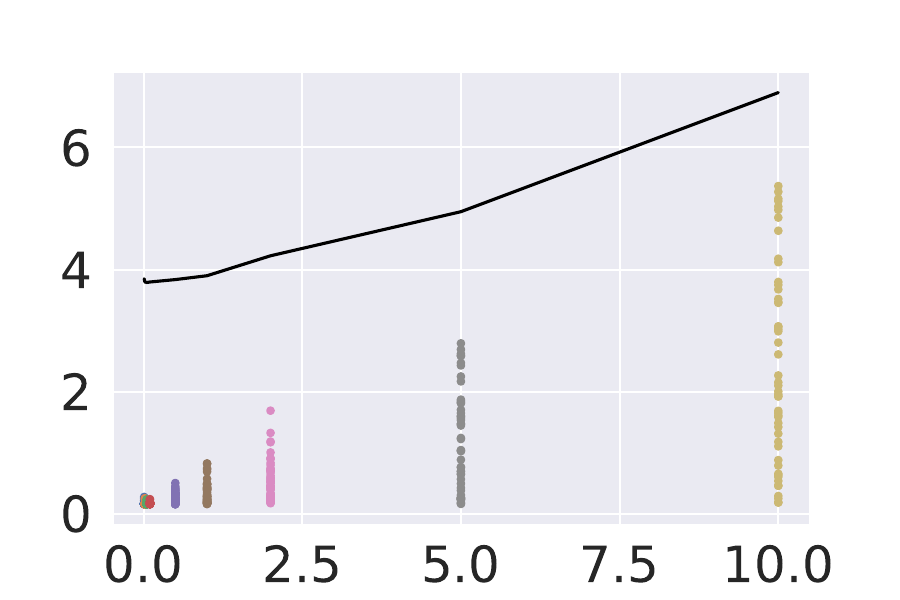}&
\includegraphics[height=\utilheightc]{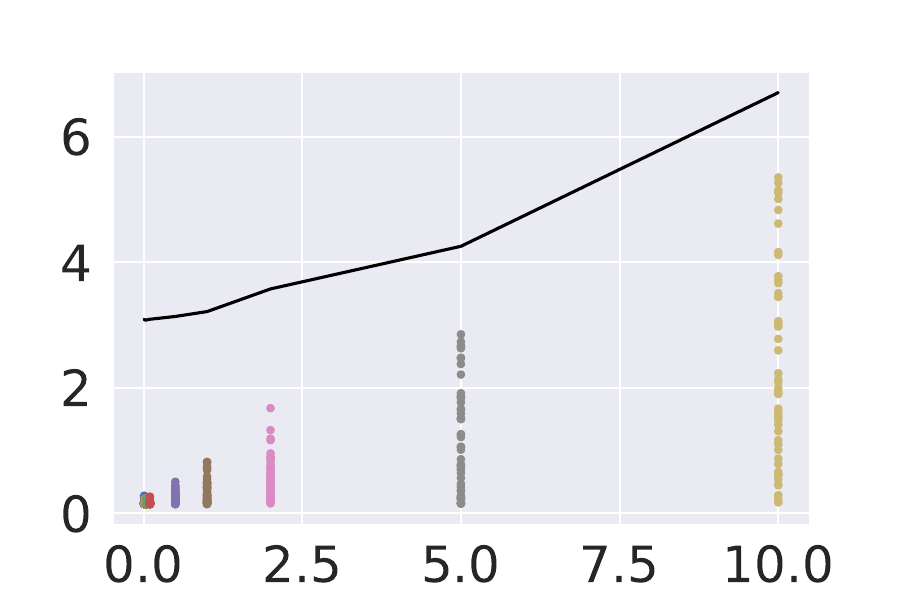}&
\includegraphics[height=\utilheightc]{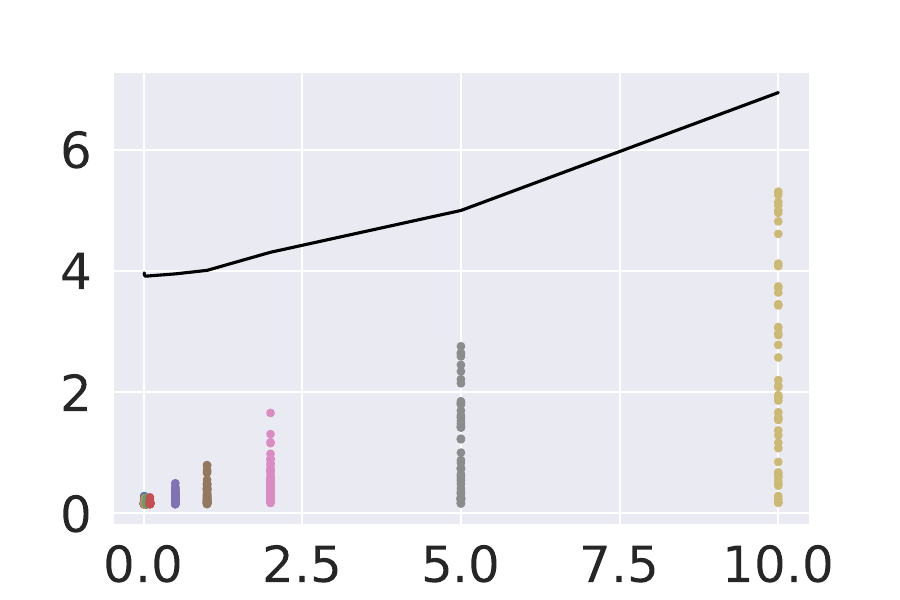}&
\includegraphics[height=\utilheightc]{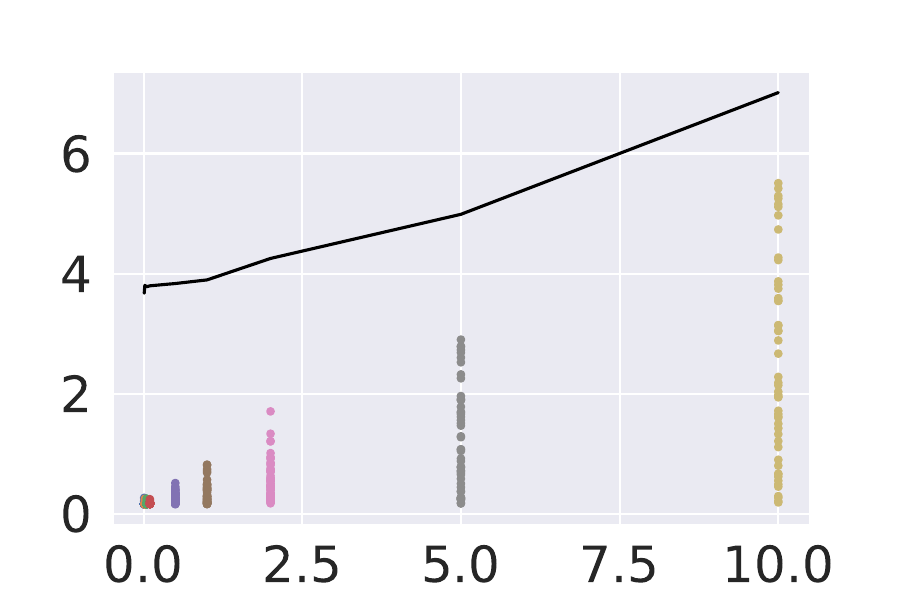}&
\includegraphics[height=\utilheightc]{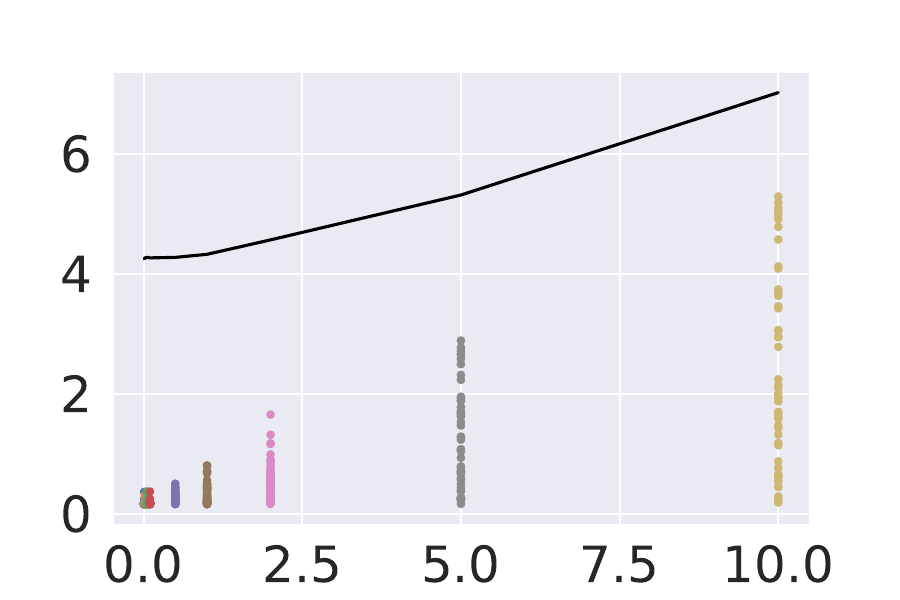}\\[-1.2ex]
        & \makecell{\tiny{Interval size}}
        & \makecell{\tiny{Interval size}}
        & \makecell{\tiny{Interval size}}
        & \makecell{\tiny{Interval size}}
        & \makecell{\tiny{Interval size}}
        & \makecell{\tiny{Interval size}}
        % & \makecell{\tiny{Shifting interval size}}
% \includegraphics[height=\utilheightc]{rotation_0.02_l2_norm_distribution.pdf}&
% \includegraphics[height=\utilheightc]{rotation_0.05_l2_norm_distribution.pdf}&
% \includegraphics[height=\utilheightc]{rotation_0.1_l2_norm_distribution.pdf}
\end{tabular}
}
% \vspace{-3mm}
\caption{Image $\ell_2$ norm distribution in rotation transformation.}\label{fig:rotation l2}
\end{subtable}

\begin{subtable}{\linewidth}
\centering
\resizebox{\linewidth}{!}{
\begin{tabular}{@{}p{3mm}@{}c@{}c@{}c@{}c@{}c@{}c@{}}
\rownamec{\makecell{$\ell_2$ norm}}&
\includegraphics[height=\utilheightc]{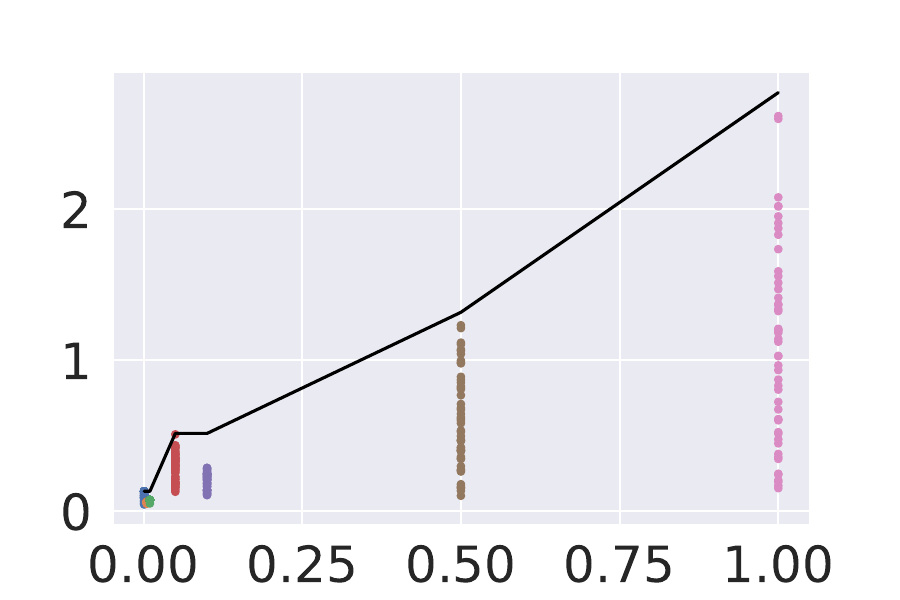}&
\includegraphics[height=\utilheightc]{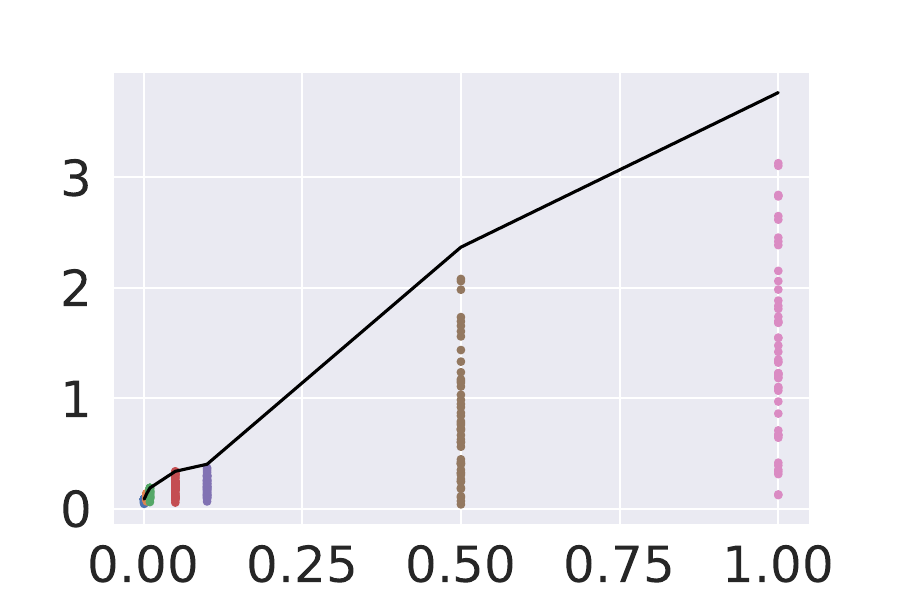}&
\includegraphics[height=\utilheightc]{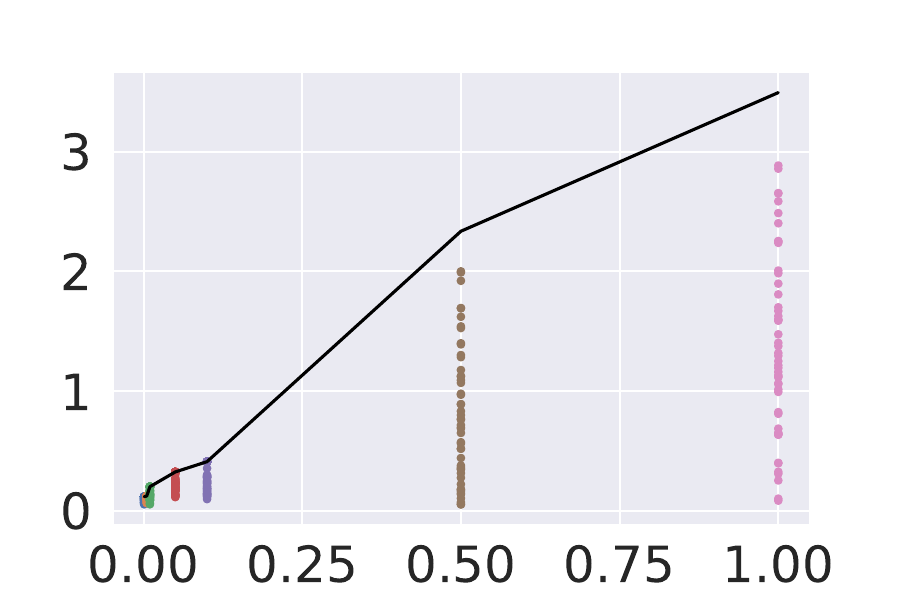}&
\includegraphics[height=\utilheightc]{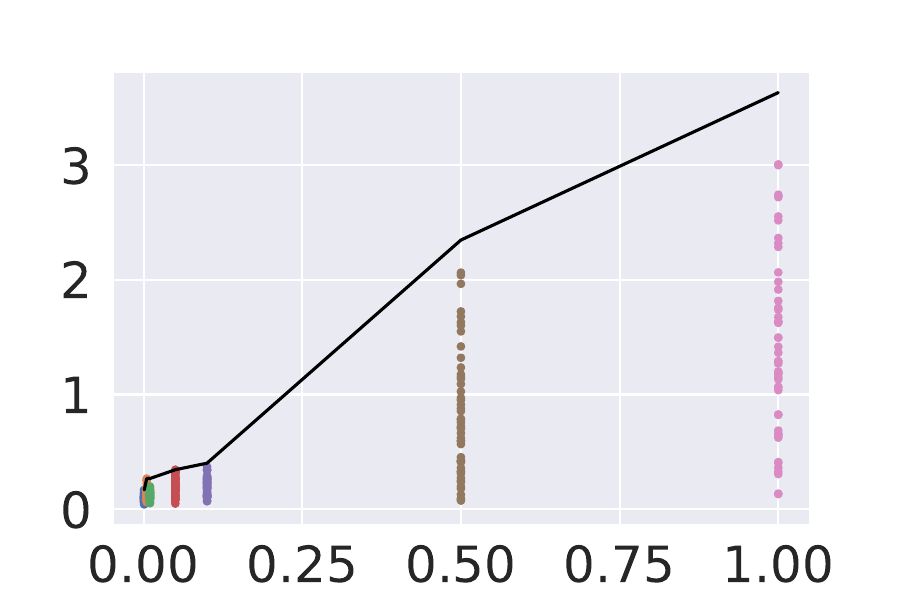}&
\includegraphics[height=\utilheightc]{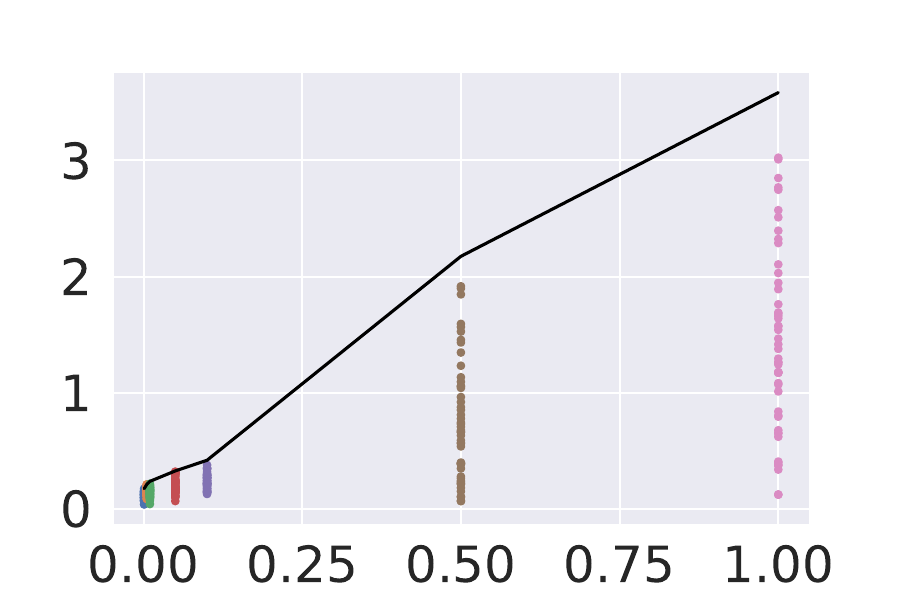}&
\includegraphics[height=\utilheightc]{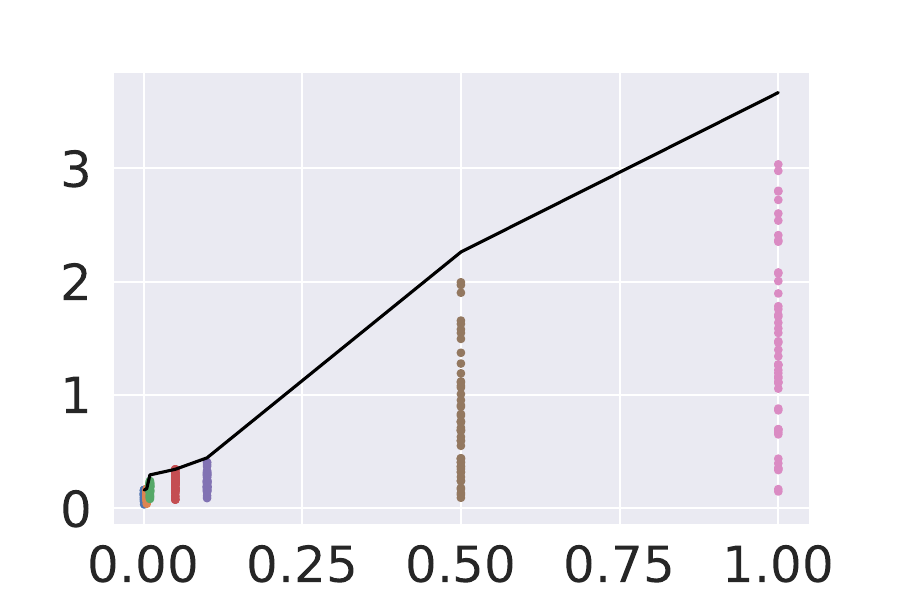}\\[-1.2ex]
        & \makecell{\tiny{Interval size}}
        & \makecell{\tiny{Interval size}}
        & \makecell{\tiny{Interval size}}
        & \makecell{\tiny{Interval size}}
        & \makecell{\tiny{Interval size}}
        & \makecell{\tiny{Interval size}}
        % & \makecell{\tiny{Shifting interval size}}
% \includegraphics[height=\utilheightc]{rotation_0.02_l2_norm_distribution.pdf}&
% \includegraphics[height=\utilheightc]{rotation_0.05_l2_norm_distribution.pdf}&
% \includegraphics[height=\utilheightc]{rotation_0.1_l2_norm_distribution.pdf}
\end{tabular}
}
% \vspace{-3mm}
\caption{Image $\ell_2$ norm distribution in shifting transformation.}\label{fig:shift l2}
\end{subtable}

}

\caption{Image $\ell_2$ norm distribution in rotation and shift transformation. Images are from spawn point 15, 30, 43, 46, 57 and 86 from dataset with building and without pedestrian. The scatter plots show the $\ell_2$ distance of randomly chosen pairs in randomly chosen big intervals. The black line is the $\ell_2$ distance between the endpoints of big intervals.}
\label{fig:transformation l2}
% \vspace{-3em}
\end{figure*}
}

\section{Proofs}
\label{append:proofs}

\subsection{Proof of \cref{thm:detection}: Detection Certification}
\label{append:proof-det}
In this section, we present the full proof of \cref{thm:detection}, which provides generic certification for multi-sensor fusion detection against an abstract transformation. We first restate this theorem from the main text.
\begingroup
\def\thetheorem{\ref{thm:detection}}
\begin{theorem}[restated]
Let $T=\{T_x, T_p\}$ be a transformation with parameter space $\mathcal Z$. Suppose $\mathcal S \subseteq \mathcal Z$ and $\{\alpha_i\}_{i=1}^M\subseteq \mathcal S$. 
For detection confidence $g:\mathcal X \times \mathcal P \to [0,1]$, let $h_q(\vx, \vp)$ be the median smoothing of $g$ as defined in \cref{eq:median}.
Then for all transformations $\vz\in \gS$, the confidence score of the median smoothed detector satisfies: 
\begin{equation}
    h_q(T_x(\vx,\vz), T_p(\vp, \vz)) \geq \min_{1\leq i\leq M} h_{\underline q}(T_x(\vx,\alpha_i), T_p(\vp, \alpha_i))
    % \label{eq:detection-lower-bound}
\end{equation}
where
\begin{align}
    \underline q &= \Phi\left(\Phi^{-1}(q) - \sqrt{\frac{M_x^2}{\sigma_x^2} + \frac{M_p^2}{\sigma_p^2}}\right),\\
    M_x &= \max_{\alpha\in \mathcal S}\min_{1\leq i\leq M} \|T_x(\vx,\alpha) - T_x(\vx,\alpha_i)\|_2, \\
    M_p &= \max_{\alpha\in \mathcal S}\min_{1\leq i\leq M} \|T_p(\vp,\alpha) - T_p(\vp,\alpha_i)\|_2.
\end{align}
\end{theorem}
\addtocounter{theorem}{-1}
\endgroup
\begin{proof}
We first recall the median smoothed classifier $h_q(\vx,\vp) = \sup\{y\in \mathbb R\ |\ \Pr [g(\vx+\delta_x,\vp+\delta_p)\leq y]\leq q\}$, where $\delta_x \sim \gN (0,\sigma_x^2 \mathbf I_d)$ and $\delta_p \sim \gN(0,\sigma^2_p \mathbf I_{3\times N})$.

Consider a function $f:\mathbb R^d \times \mathbb R^{3\times N}\to [0,1]$ with $f(\vx,\vp) = \Pr[g(\vx + \delta_x, \vp + \delta_p) \leq h_{\underline q}(\vx, \vp)]$. We define $\tilde g: \gX \times \gP \to [0,1]$ as $\tilde g(\vx, \vp) = g(\vx, \frac{\sigma_p}{\sigma_x}\vp)$. Then
\begin{equation}
    f(\vx, \vp) = \Pr[\tilde g(\vx + \delta_x, \vp^\prime+\delta_p^\prime)\leq h_{\underline q}(\vx, \vp)],\ \text{where}\ \vp^\prime = \frac{\sigma_x}{\sigma_p}\vp,\ \delta_p^\prime \sim \gN (0,\sigma_x^2\mathbf I_{3\times N}).
\end{equation}

We now introduce a lemma from \cite{Salman2019provably}(Lemma 2) and \cite{chiang2020detection}(Corollary 1).
\begin{lemma}
For any $g:\mathbb R^d \to [l,u]$, let $f(x) = \mathbb E[g(x+G)]$ where $G\sim \gN(0,\sigma^2 \mathbf I)$. Then the map $\eta(x) = \sigma\cdot \Phi^{-1}(\frac{f(x)-l}{u-l})$ is 1-Lipschitz.
\end{lemma}
Note that $f(\vx, \vp^\prime) = \mathbb E\left(\mathbbm 1[\tilde g(\vx + \delta_x, \vp^\prime + \delta_p^\prime)\leq h_{\underline q}(\vx, \vp)]\right)$, which means $\sigma_x \cdot \Phi^{-1}(f(\vx, \vp^\prime))$ is 1-Lipschitz.

Now let us consider an arbitrary transformation $\vz\in \gS$, by the definition of $M_x$ and $M_p$ we have
\begin{equation}
    \forall \vz\in \gS,\ \exists \alpha_i,\ \|T_x(\vx,\vz) - T_x(\vx, \alpha_i)\|_2 \leq M_x, \|T_p(\vp,\vz) - T_p(\vp,\alpha_i)\|_2 \leq M_p.
\end{equation}

Then
\begin{align}
    \sigma_x &\Phi^{-1}\Big( \Pr\Big[g(T_x(\vx,\vz)+\delta_x, T_p(\vp,\vz)+\delta_p)\leq h_{\underline q}(\vx,\vp)\Big]\Big) = \sigma_x \Phi^{-1}\left(f(T_x(\vx,\vz),\frac{\sigma_x}{\sigma_p}T_p(\vp,\vz))\right)\\
    &\geq \sigma_x \Phi^{-1}\left(f\Big(T_x(\vx,\alpha_i),\frac{\sigma_x}{\sigma_p}T_p(\vp,\alpha_i)\Big)\right) - \sqrt{M_x^2 + \frac{\sigma_x^2}{\sigma_p^2}M_p^2}\\
    &\geq \sigma_x \Phi^{-1}\Big(\Pr\Big[g(T_x(\vx,\alpha_i)+\delta_x, T_p(\vp,\alpha_i)+\delta_p)\leq h_{\underline q}(\vx,\vp)\Big] \Big) - \sqrt{M_x^2 + \frac{\sigma_x^2}{\sigma_p^2}M_p^2}\\
    &= \sigma_x \Phi^{-1}(\underline q)  - \sqrt{M_x^2 + \frac{\sigma_x^2}{\sigma_p^2}M_p^2}\\
    &= \sigma_x\Phi^{-1}(q).
\end{align}

Since that $\Phi(\cdot)$ is monotonic, we conclude that
\begin{equation}
    \forall \vz\in \gS,\ \exists \alpha_i, \mathrm{s.t.}\ h_q(T_x(\vx,\vz), T_p(\vp,\vz)) \geq h_{\underline q}(\vx,\vp).
\end{equation}
\end{proof}

\subsection{Proof of \cref{lem:upper-bound-mx-mp}: Upper Bound for the Interpolation Error}
\label{proof:mx-mp}
\begingroup
\def\thetheorem{\ref{lem:upper-bound-mx-mp}}
\begin{lemma}[restated]
  If the parameter space to certify $\gS = [l_1, u_1] \times \cdots \times [l_m, u_m]$ is a hypercube satisfying \Cref{assum} with threshold $\tau$, and $\{\alpha_i\}_{i=1}^M = \{\frac{K_1-k_1}{K_1}l_1 + \frac{k_1}{K_1}u_1: k_1 = 0,1,\dots,K_1\} \times \cdots \times \{\frac{K_m-k_m}{K_m}l_m + \frac{k_m}{K_m}u_m: k_m = 0,1,\dots,K_m\}$, where $K_i \ge \frac{u_i - l_i}{\tau}$, then
    \begin{align}
        M_x &\leq \sum_{i=1}^m \max_{\vk\in \Delta}\Big\|T_x(\vx, \vw(\vk)) - T_x(\vx, \vw(\vk)+w_i)\Big\|_2, \label{eq:append_Mx}\\
        M_p &\leq \sum_{i=1}^m\max_{\vk\in \Delta}\Big\|T_p(\vp, \vw(\vk)) - T_p(\vp, \vw(\vk)+w_i)\Big\|_2 \label{eq:append_Mp}
    \end{align}
    where $\Delta = \{(k_1,\dots, k_m)\in \mathbb Z^m\ |\ 0\leq k_i< K_i\}$ and $\vw(\vk) = ( \frac{K_1-k_1}{K_1} l_1 + \frac{k_1}{K_1} u_1, \cdots, \frac{K_m-k_m}{K_m} l_m + \frac{k_m}{K_m} u_m )$. $w_i = \frac{u_i - l_i}{K_i}\ve_i$, where $\ve_i$ is a unit vector at coordinate $i$.
\end{lemma}
\addtocounter{theorem}{-1}
\endgroup
\begin{proof}
Let  $\vz \in \gS\subseteq \mathbb R^m$ be a parameter in the parameter space to certify and $\vz = (z_1,\dots, z_m)$. There must be $(k_1,k_2,\dots, k_m)$ with $k_i\in \{0,1,\dots, K_i-1\}$, such that
\begin{equation}
    \frac{K_i - k_i}{K_i}l_i + \frac{k_i}{K_i}u_i \leq z_i \leq \frac{K_i - k_i-1}{K_i}l_i + \frac{k_i+1}{K_i}u_i,\ \forall i=1,2,\dots, m.
\end{equation}

Let $\vk = (k_1,k_2,\dots, k_m)$. Consider the small polytope $\gZ_{\mathrm{sub}} = \vk + [0,1]^m \cdot (w_1,\dots, w_m)$. By \cref{assum},
\begin{align}
    \|T_x(\vx, \vz) - T_x(\vx, \vw(k))\|_2 &\leq \max_{\alpha,\beta\in E(\gZ_{\mathrm{sub}})} \|T_x (\vx,\alpha) - T_x(\vx,\beta)\|_2\\
    &= \|T_x(\vx,\vz_1) - T_x(\vx,\vz_2)\|_2
\end{align}
where $\vz_1 = \vk + I_1\cdot (w_1,\dots, w_m)$ and $\vz_2 = \vk + I_2\cdot(w_1,\dots, w_m)$ for some $I_1,I_2 \in \{0,1\}^m$. Let $\{\alpha_1,\alpha_2,\dots, \alpha_n\}$ be a shortest path from $\vz_1$ to $\vz_2$ such that $\alpha_i \in \{0,1\}^m\cdot (w_1,\dots, w_m) + \vk$, $\alpha_1 = \vk_1$, and $\alpha_n = \vk_2$. Moreover, $\alpha_j$ and $\alpha_{j+1}$ differ on exactly 1 non-repeating coordinate $c_j$. Then
\begin{align}
    \|T_x(\vx,\vz_1) - T_x(\vx,\vz_2)\|_2 & = \|\sum_{j=1}^{n-1} T_x(\vx,\alpha_j) - T_x(x,\alpha_{j+1})\|_2\\
    &\leq \sum_{j=1}^{n-1} \|T_x(\vx,\alpha_{j}) - T_x(\vx,\alpha_{j+1})\|_2\\
    &\leq \sum_{j=1}^{n-1} \max_{\vk\in \Delta} \|T_x(\vx,\vw(\vk)) - T_x(\vx,\vw(\vk)+w_{c_j})\|_2,\\
    &\leq \sum_{i=1}^m \max_{\vk\in \Delta}\|T_x(\vx,\vw(\vk)) - T_x(\vx,\vw(\vk)+w_i)\|_2.
\end{align}
which implies \cref{eq:append_Mx}. \cref{eq:append_Mp} also holds, following exactly the same argument for point clouds.
\end{proof}

\subsection{Proof of \cref{theorem:iou}: General IoU Certification for 3D Bounding Boxes}
\label{append:proof-iou}
We first recall \cref{theorem:iou} from the main paper. Note that we omit details for the convex hulls $\underline S,\bar S$ in the main paper version for simplicity. Here we provide a complete version with a formal description for $\underline S,\bar S$.
\begingroup
\def\thetheorem{\ref{theorem:iou}}
\begin{theorem}[restated]
Let $\mathbf B$ be a set of bounding boxes whose coordinates are bounded. We denote the lower bound of each coordinate by $(\underline x, \underline y, \underline z, \underline w,\underline h , \underline l, \underline r)$ and upper bound by $(\bar x,\bar y,\bar z, \bar w, \bar h, \bar l, \bar r)$. Let $B_{gt} = (x,y,z,w,h,l,r)$ be the ground truth bounding box.
Then for any $B_i \in \mathbf B$,
\begin{equation}
    \iou(B_i,B_{gt}) \geq \frac{h_1 \cdot \left(\underline l\underline w-\vol(\underline S \backslash S_{gt})\right)}{hwl + \bar h\bar w \bar l - h_2 \cdot \left(\bar l \bar w - \vol(\bar S \backslash S_{gt})\right)}
\end{equation}
where $S_{gt} = (x,z,w,l,r)_{gt}$ is the projection of $B_{gt}$ to the $x-z$ plane.
\begin{align}
    h_1 &= \max\big(\min_{y^\prime\in [\underline y, \bar y]} \min\{h, \underline h, \frac{h+\underline h}{2} - |y^\prime-y|\},0\big),\nonumber\\
    h_2 &= \max\big(\min_{y^\prime\in [\underline y, \bar y]} \min\{h, \bar h, \frac{h+\bar h}{2} - |y^\prime-y|\},0\big).
\end{align}
$\underline S, \bar S$ are convex hulls formed by $(\underline x,\underline z,\underline r,\bar x,\bar z,\bar r)$ with respect to $(\underline w,\underline l)$ and $(\bar w,\bar l)$. Here we formally define $C(\underline x, \underline z,\underline r,\bar x,\bar z,\bar r,w,l)$. Let $\varphi = \arctan(\frac{l}{w})$, we first define
\begin{align}
    \Delta x_{\max,k} &= \max_{\theta\in [\underline r, \bar r]} \sqrt{w^2 + l^2}\cos(\theta + k\varphi),\\
    \Delta x_{\min,k} &= \min_{\theta \in [\underline r,\bar r]} \sqrt{w^2 + l^2} \cos(\theta + k\varphi),\\
    \Delta z_{\max,k} &= \max_{\theta\in [\underline r, \bar r]} \sqrt{w^2 + l^2}\sin(\theta + k\varphi),\\
    \Delta z_{\min,k} &= \min_{\theta\in [\underline r, \bar r]} \sqrt{w^2 + l^2}\sin(\theta + k\varphi).
\end{align}
where $k\in \{-1,1\}$. The range for each of the four coordinates can be expressed as
\begin{align}
    P_{1,1} &= \{\underline x + \frac{\Delta x_{\min,1}}{2}, \bar x + \frac{\Delta x_{\max, 1}}{2}\} \otimes  \{\underline z + \frac{\Delta z_{\min,1}}{2}, \bar z + \frac{\Delta z_{\max, 1}}{2}\},\\
    P_{1,-1} &= \{\underline x + \frac{\Delta x_{\min,-1}}{2}, \bar x + \frac{\Delta x_{\max, -1}}{2}\} \otimes  \{\underline z + \frac{\Delta z_{\min,-1}}{2}, \bar z + \frac{\Delta z_{\max, -1}}{2}\},\\
    P_{-1,1} &= \{\underline x - \frac{\Delta x_{\max,-1}}{2}, \bar x - \frac{\Delta x_{\min, -1}}{2}\} \otimes  \{\underline z - \frac{\Delta z_{\max,-1}}{2}, \bar z - \frac{\Delta z_{\min, -1}}{2}\},\\
    P_{-1,-1} &= \{\underline x - \frac{\Delta x_{\max,1}}{2}, \bar x - \frac{\Delta x_{\min, 1}}{2}\} \otimes  \{\underline z - \frac{\Delta z_{\max,1}}{2}, \bar z - \frac{\Delta z_{\min, 1}}{2}\}.
\end{align}
The convex hull $C(\underline x,\underline z, \underline r, \bar x,\bar z,\bar r,w,l)$ is
\begin{equation}
        C(\underline x, \underline z,\underline r,\bar x,\bar z,\bar r, w,l) = \mathrm{Conv}\left( P_{1,1}\cup P_{1,-1} \cup P_{-1,1}\cup P_{-1,-1}\right).
\end{equation}
The convex hull $\underline S = C(\underline x,\underline z, \underline r, \bar x,\bar z,\bar r,\underline w,\underline l)$, and $\bar S = C(\underline x,\underline z, \underline r, \bar x,\bar z,\bar r,\bar w,\bar l)$.
\end{theorem}
\addtocounter{theorem}{-1}
\endgroup
\begin{proof}
Let $B_i\in \mathbf B$ be a bounding box whose coordinates are lower bounded by $(\underline x, \underline y, \underline z, \underline w,\underline h, \underline l, \underline r)$ and upper bounded by $(\bar x,\bar y,\bar z, \bar w, \bar h, \bar l, \bar r)$. Let $B_{gt} = (x,y,z,w,h,l,r)$ be the ground truth.
\begin{equation}
    \mathrm{IoU}(B_i, B_{gt}) = \frac{\vol(B_i \cap B_{gt})}{\vol(B_i \cup B_{gt})}
\end{equation}
Given a fixed center $(x,y,z)_i$ and a rotation angle $r_i$ for $B_i$, the volumes $\vol(B_i\cap B_{gt})$ and $\vol(B_i \cup B_{gt})$ are both monotonic in terms of the size of $B_i$, $(w,h,l)_i$. Hence
\begin{equation}
    \mathrm{IoU}(B_i, B_{gt}) \geq \frac{\min_{x,y,z,r} \vol(B_i(x,y,z, \underline w,\underline h, \underline l, r) \cap B_{gt})}{\max_{x,y,z,r} \vol(B_i(x,y,z,\bar w, \bar h, \bar l, r)\cup B_{gt})}
    \label{eq:proof-iou}
\end{equation}

Note that
\begin{align}
    & \vol(B_i(x_i,y_i,z_i,\bar w, \bar h, \bar l, r_i)\cup B_{gt}) \\
    = & \vol(B_i(x_i,y_i,z_i,\bar w, \bar h, \bar l, r_i))+\vol(B_{gt})-\vol(B_i(x_i,y_i,z_i,\bar w, \bar h, \bar l, r_i)\cap B_{gt})\nonumber \\
    = & \bar w \bar h\bar l + whl -\vol(B_i(x_i,y_i,z_i,\bar w, \bar h, \bar l, r_i)\cap B_{gt}).
\end{align}
Therefore,
\begin{equation}
\label{eq:proof-union}
    \max_{x,y,z,r} \vol\big(B_i(x,y,z,\bar w,\bar h,\bar l, r)\cup B_{gt}\big) = \bar w \bar h \bar l + whl - \min_{x,y,z,r} \vol(B_i(x,y,z,\bar w,\bar h,\bar l,r)\cap B_{gt}).
\end{equation}
Combine \cref{eq:proof-iou,eq:proof-union}, we are left with the work of estimating $\min_{x,y,z,r}\vol(B_i(x,y,z,w_i,h_i,l_i,r)\cap B_{gt})$ for some fixed $(w_i,h_i,l_i) = (\underline w,\underline h,\underline l)$ or $(\bar w,\bar h,\bar l)$. Notice that 3D bounding boxes can be arbitrarily rotated along the y-axis, we consider the intersection on the y-axis and on the x-z plane separately.

\textbf{Intersection on the y-axis.} Projecting $B_i$ and $B_{gt}$ to the y-axis, we want to lower bound the intersection between an interval $I_1$ with length $h_i$ centered at $y_i \in [\underline y, \bar y]$ and the ground truth interval $I_2 = [y-\frac{h}{2}, y+\frac{h}{2}]$.

Suppose $h_i < h$. If $|y_i-y|<\frac{h-h_i}{2}$, $|I_1(y_i)\cap I_2| = h_i$; otherwise $|I_1(y_i)\cap I_2| = \max\{\frac{h+h_i}{2}-|y_i-y|,0\}$. In this case we conclude that $|I_i(y_i)\cap I_2| = \max\{\min\{h_i,\frac{h+h_i}{2}-|y^\prime - y|\},0\}$. By the exact same argument, when $h_i \geq h$, $|I_1(y_i)\cap I_2| = \max\{\min\{h,\frac{h+h_i}{2}-|y_i - y|\},0\}$. Thus,
\begin{equation}
    |I_1(y_i)\cap I_2| \geq \max\{\min_{y_i\in [\underline y, \bar y]} \min\{h,h_i, \frac{h+h_i}{2} - |y_i- y|\} ,0\}.
\end{equation}
In particular, when $h_i = \underline h$ and $h_i = \bar h$, the intersection between $B_i$ and $B_{gt}$ on y-axis is larger than $h_1$ and $h_2$, respectively, where
\begin{align}
    h_1 &= \max\big(\min_{y^\prime\in [\underline y, \bar y]} \min\{h, \underline h, \frac{h+\underline h}{2} - |y^\prime-y|\},0\big),\nonumber\\
    h_2 &= \max\big(\min_{y^\prime\in [\underline y, \bar y]} \min\{h, \bar h, \frac{h+\bar h}{2} - |y^\prime-y|\},0\big).
\end{align}
Note that both $h_1$ and $h_2$ can be precisely numerically computed, where the pseudocode is in \Cref{alg:iou_lower_bound}.

\textbf{Intersection on the x-z plane.} Next, we consider the projection of $B_i$ and $B_{gt}$ on the x-z plane, denoted by $S_i$ and $S_{gt}$, respectively.
We have
\begin{align}
    \min_{x,z,r}\vol(S_i(x,z,w_i,l_i,r)\cap S_{gt}) &= \vol(S_i(x,z,w_i,l_i,r)) - \min_{x,z,r}\vol(S_i(x,z,w_i,l_i,r)\backslash S_{gt})\\
    &=  w_i l_i - \min_{x,z,r}\vol(S_i(x,z,w_i,l_i,r)\backslash S_{gt})\\
    &\geq w_i l_i - \vol(C(\underline x, \underline z,\underline r, \bar x,\bar z,\bar r, w_i,l_i) \backslash S_{gt})\label{eq:append:iouxz}
\end{align}
where $C(\underline x, \underline z,\underline r, \bar x,\bar z,\bar r, w_i,l_i)$ is an envelop that contains all possible x-z bounding boxes $S_i$ with $(\underline x,\underline z,\underline r) \leq(x,z,r)\leq (\bar x,\bar z,\bar r)$ and a fixed size $(w_i,l_i)$.
\begin{figure}[ht]
    \centering
    \includegraphics[width=.5\textwidth]{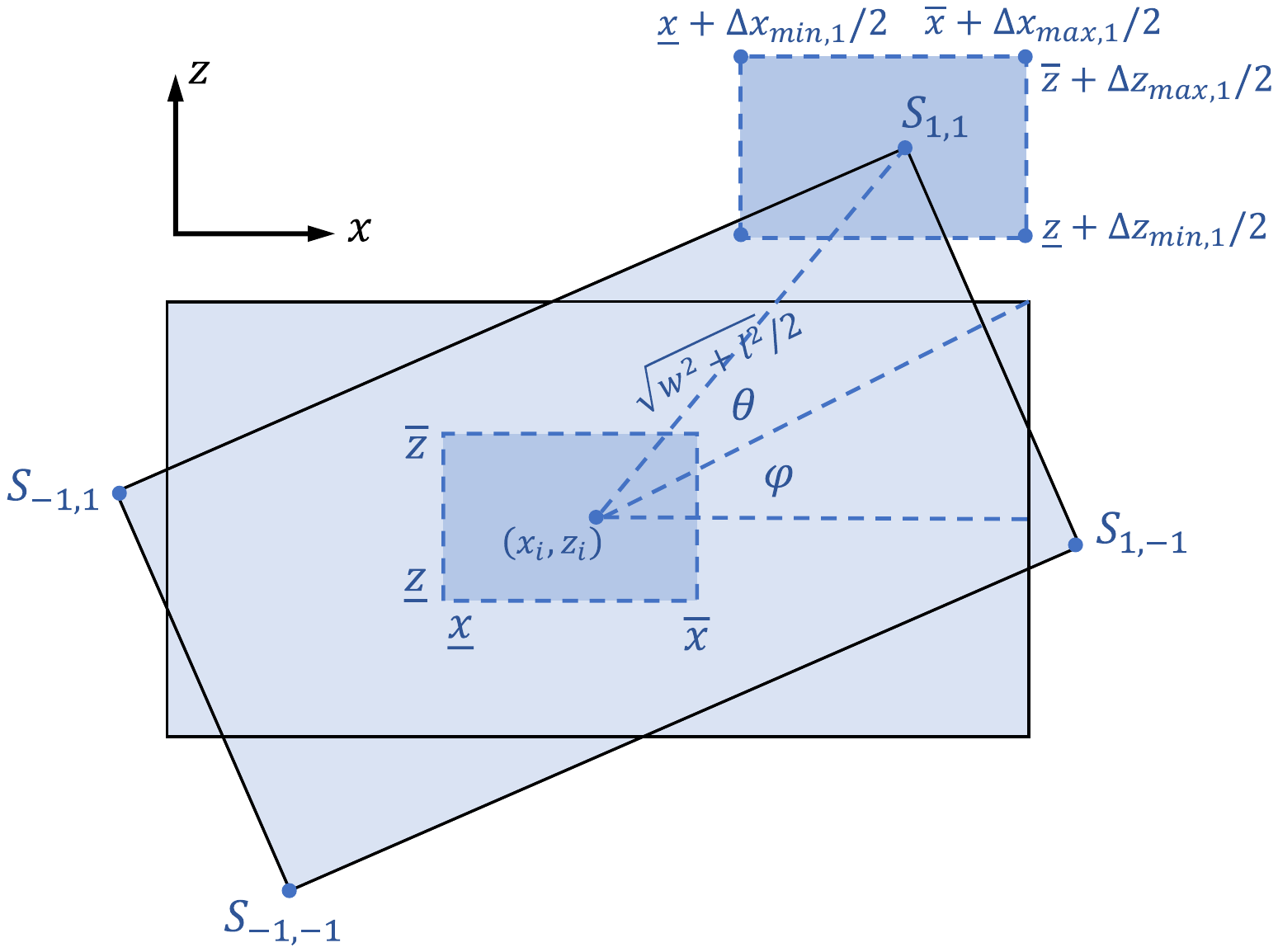}
    \caption{Illustration of $C(\underline x,\underline z,\underline r, \bar x,\bar z,\bar r, w,l)$ on $x-z$ plane.}
    \label{fig:append-iou_xz}
\end{figure}

As shown in \cref{fig:append-iou_xz}, we calculate the possible range for each of the four vertices of bounding box $S_i$. For example, \cref{fig:append-iou_xz} illustrates that the $x-z$ coordinate of the upper-right vertex is $s_{1,1}= (x_i + \frac{\sqrt{w^2 + l^2}}{2} \cos(\theta + \varphi), z_i + \frac{\sqrt{w^2+l^2}}{2}\sin (\theta + \varphi))$. Therefore, its $x,z$ coordinate of $s_{1,1} $satisfies
\begin{align}
    \underline x + \min_{\theta\in [\underline r,\bar r]} \frac{1}{2}\sqrt{w^2 + l^2} \cos(\theta + \varphi)\leq &x_{1,1} \leq \bar x + \max_{\theta\in [\underline r,\bar r]} \frac{1}{2}\sqrt{w^2+l^2} \cos (\theta + \varphi)\\
    \underline z + \min_{\theta\in [\underline r, \bar r]} \frac{1}{2}\sqrt{w^2+l^2}\sin(\theta +\varphi) \leq & z_{1,1} \leq \bar z + \max_{\theta\in [\underline r, \bar r]} \frac{1}{2} \sqrt{w^2 + l^2} \sin(\theta + \varphi).
\end{align}
which means $s_{1,1} = (x_{1,1},z_{1,1})$ is contained by the rectangle formed by the four points in $S_{1,1}$, i.e., $s_{1,1}\in \mathrm{Conv}(P_{1,1})$. Similar arguments also hold for rest of vertices: $s_{1,-1}\in \mathrm {Conv}(P_{1,-1})$, $s_{-1,1}\in \mathrm {Conv}(P_{-1,1})$, and $s_{-1,-1}\in \mathrm {Conv}(P_{-1,-1})$. Finally, we conclude that
\begin{equation}
    S_{i} = \mathrm{Conv}(s_{1,1},s_{1,-1},s_{-1,1},s_{-1,-1}) \subseteq \mathrm{Conv}(P_{1,1},P_{1,-1},P_{-1,1},P_{-1,-1}) \label{eq:append:conv}
\end{equation}
Note that $\underline S = C(\underline x,\underline z, \underline r, \bar x,\bar z,\bar r,\underline w,\underline l)$, and $\bar S = C(\underline x,\underline z, \underline r, \bar x,\bar z,\bar r,\bar w,\bar l)$. By \cref{eq:append:iouxz,eq:append:conv},
\begin{equation}
    \min_{x,z,r} \vol(S_i(x,z,\underline w,\underline l,r)\cap S_{gt}) \geq\underline w\underline l - \vol(\underline S \backslash S_{gt})
\end{equation}
and
\begin{equation}
    \min_{x,z,r} \vol(S_i(x,z,\bar w,\bar l,r)\cap S_{gt}) \geq\bar w\bar l - \vol(\bar S \backslash S_{gt}).
\end{equation}
Combining the above, we have
\begin{align}
    \mathrm{IoU}(B_i,B_{gt}) \geq \frac{h_1(\underline w\underline l - \vol(\underline S \backslash S_{gt}))}{hwl+\bar h\bar w \bar l - h_2(\bar w\bar l - \vol(\bar S\backslash S_{gt}))}.
\end{align}
\end{proof}

\section{Additional Details of Certification Strategies}
\label{append:methods}
In this section, we present the details of our detection (\Cref{subsec:detection cert}) and IoU (\Cref{subsec:IoU cert}) certification algorithms for camera and LiDAR fusion models, including the pseudocode and some details of our implementation. Note that our algorithms can be adapted to any fusion framework and single-modality model by doing smoothing inference and certification for corresponding modules. 

\subsection{Detection Certification}

\label{subsec:detection cert}
\Cref{alg:0/1 detection} presents the pseudocode of our detection certification (\textbf{function} \textsc{Certify}) and median smoothing detection (\textbf{function} \textsc{Inference}). In our implementation, we directly sample on the left endpoints of each interval and do smoothing inference, which can also be replaced by a random point in each small interval. 

\begin{algorithm}[H]
\algsetup{linenosize=\scriptsize}
\scriptsize
\DontPrintSemicolon
\caption{0/1 Detection and Certification}\label{alg:0/1 detection}
\KwIn{clean image input $\vx_0 \in \gX \subseteq \sR^d$, clean point cloud input  $\vp_0 \in \gP \subseteq \sR^{3\times N}$, multi-sensor fusion pipeline $F: \gX \times \gP \to ([C] \times [0, 1])_n$, ground truth label $y$, image Gaussian noise variance $\sigma_x^2$, point Gaussian noise variance $\sigma_p^2$, transformation space $[\vz_{l}, \vz_{u}]$, transformation function $T$, detection threshold $\gamma_0$}
\KwOut{smoothed confidence score of ground truth label $c_{median}$, lower bound and upper bound of confidence score of ground truth $c_u, c_l$, whether the object is robustly detected as a boolean variable $b$}
\SetKwFunction{Fdetect}{\textsc{Inference}}
\SetKwFunction{FAddGaussianNoise}{\textsc{AddGaussianNoise}}
\SetKwFunction{FComputeConfidenceScores}{\textsc{ComputeConfidenceScores}}
\SetKwFunction{FSort}{\textsc{Sort}}
\SetKwProg{Fn}{function}{:}{}
\Fn{\Fdetect{$F, \vx_0, \vp_0, y_k, \sigma_x, \sigma_p, n,\gamma_0$}}{
    $\hat{x} \leftarrow$ \FAddGaussianNoise$(\vx_0, \sigma_x, n)$\;
    $\hat{p} \leftarrow$ \FAddGaussianNoise$(\vp_0, \sigma_p, n)$\;
    $\hat{c} \leftarrow$ \FComputeConfidenceScores$(F, \hat{x}, \hat{p}, y_k)$\;
    $\hat{c} \leftarrow$ \FSort$(\hat{c})$\;
    $c_{median} \leftarrow \hat{c}_{\lfloor 0.5n \rfloor}$\;
    $b \gets \1[c_{median} \ge \gamma_0]$\;
    \textbf{return} $c_{median}, b$
}
\SetKwFunction{Fcert}{\textsc{Certify}}
\SetKwFunction{FSplitInterval}{\textsc{SplitInterval}}
\SetKwFunction{FComputeEps}{\textsc{ComputeEps}}
\SetKwFunction{FGetEmpiricalPerc}{\textsc{GetEmpiricalPerc}}
\Fn{\Fcert{$F, \vx_0, \vp_0, y_k, \sigma_x, \sigma_p, n, [\vz_{l}, \vz_{u}], c, \alpha, \gamma_0$}}{
$interval\_list \leftarrow$ \FSplitInterval($[\vz_{l}, \vz_{u}]$)\;
$c_{median}\_list$, $c_u\_list$, $c_l\_list \leftarrow$ [], [], []\;
\For{$[z_{sl}, z_{su}] \in$ interval\_list}{
    $\epsilon \leftarrow$ \FComputeEps($T, \vx_0, \vp_0, \sigma_x, \sigma_p, \vz_{sl}, \vz_{su}$) \Comment*[r]{\scriptsize compute $\eps$ according to equation (13)}\;
    $q_u, q_l \leftarrow$ \FGetEmpiricalPerc$(n, \epsilon, c, \alpha/|\text{interval\_list}|)$ \Comment*[r]{\scriptsize compute $q_u, q_l$ according to equation (10)}\;
    $\hat{x} \leftarrow$ \FAddGaussianNoise$(T_x(\vx_0, \vz_{sl}), \sigma_x, n)$ \;
    \Comment*[r]{\scriptsize sample $n$ Gaussian noises $\vdelta_x \sim \gN(0, \sigma_x^2 \mI_d)$ and add to $T_x(\vx_0, \vz_{sl})$ to get $n$ noisy samples $\hat{x}$}\;
    $\hat{p} \leftarrow$ \FAddGaussianNoise$(T_p(\vp_0, \vz_{sl}), \sigma_p, n)$ \;
    \Comment*[r]{\scriptsize sample $n$ Gaussian noises $\vdelta_p \sim \gN(0, \sigma_p^2 \mI_{3\times N})$ and add to $(T_p(\vp_0, \vz_{sl})$ to get $n$ noisy samples $\hat{p}$}\;
    $\hat{c} \leftarrow$ \FComputeConfidenceScores$(F, \hat{x}, \hat{p}, y)$ \;
    \Comment*[r]{\scriptsize collect the confidence score of $y$ for each $(\vx,\vp) \in (\hat{x}, \hat{p})$ based on $F(\vx,\vp) = \max_{1\le k\le \ell: y_k = y} c_k$}\;
    $\hat{c} \leftarrow$ \FSort$(\hat{c})$\;
    $c_{median} \leftarrow \hat{c}_{\lfloor 0.5n \rfloor}$\;
    \eIf {$q_l = -1$}{
        $c_l \leftarrow -\infty$ \Comment*[r]{\scriptsize $-\infty$ \text{means cannot certify}}
    } {
        $c_l \leftarrow \hat{c}_{q_l}$
    }
    \eIf {$q_u = \infty$}{
        $c_u \leftarrow \infty$ \Comment*[r]{\scriptsize $\infty$ \text{means cannot certify}}
    } {
        $c_u \leftarrow \hat{c}_{q_u}$
    }
    $c_{median}\_list$ \textbf{add} $c_{median}$, $c_u\_list$ \textbf{add} $c_u$, $c_l\_list$ \textbf{add} $c_l$
}
    $b \gets \1[\min c_l\_list \ge \gamma_0]$\;
    \textbf{return} $\min c_l\_list$, $\max c_u\_list$, $b$\;
}
\SetKwFunction{FBinarySearchLeft}{\textsc{BinarySearchLeft}}
\SetKwFunction{FBinarySearchRight}{\textsc{BinarySearchRight}}
\Fn{\FGetEmpiricalPerc$(n, \epsilon, c, \alpha)$}{
    $\underline{p} \leftarrow \Phi \left( \Phi ^{-1}( p) - \epsilon \right)$ \;
    $q_l \leftarrow \FBinarySearchLeft(BinomialCDF(n, \underline{p}), 1-\alpha) $\Comment*[r]{\scriptsize do binary search and choose the left endpoint}\;
    \If {$\text{BinomialCDF}(q_l, \underline{p}) > 1 - \alpha$} {
        $q_l \leftarrow -1$
    } 
    $\overline{p} \leftarrow \Phi \left( \Phi ^{-1}( p) + \epsilon \right)$ \;
    $q_u \leftarrow \FBinarySearchRight(BinomialCDF(n, \overline{p}), \alpha) $\Comment*[r]{\scriptsize do binary search and choose the right endpoint}\;
    \If {$\text{BinomialCDF}(q_u, \overline{p}) < \alpha$} {
        $q_u \leftarrow \infty$
    }
    \textbf{return} $q_u, q_l$
}
\end{algorithm}

\vspace{-1em}
\subsection{IoU Certification}
\label{subsec:IoU cert}
\Cref{alg:bounding box detection} presents the pseudocode of our IoU certification (\textbf{function} \textsc{Certify}) and smoothing bounding box detection (\textbf{function} \textsc{Inference}). As in the detection certification framework, we also do sampling on the lower endpoint of each small interval. \Cref{alg:lower bound of IoU} presents the pseudocode of IoU lower bound computation given the bounding box parameter intervals and the ground truth bounding box, and \Cref{alg:intervals of corners' xz} presents the pseudocode for computing the $x,z$ coordinate intervals for bounding box endpoints. Note that our framework focuses on the 7-parameter representation of bounding boxes ($x,y,z$, height, width, length, rotation angle), which can be adapted to other representation formats easily (e.g. 8 parameter representation, endpoint representation). 

\begin{algorithm}[H]
\algsetup{linenosize=\scriptsize}
\scriptsize
\DontPrintSemicolon
\caption{Bounding Box Detection and Certification}\label{alg:bounding box detection}
\KwIn{clean image input $\vx_0 \in \gX \subseteq \sR^d$, clean point cloud input  $\vp_0 \in \gP \subseteq \sR^{3\times N}$, multi-sensor fusion pipeline $F: \gX \times \gP \to ([X] \times [Y] \times [Z] \times [W] \times [H] \times [L] \times [R] \times [C] \times [0,1])_n$, image Gaussian noise $\vdelta_x \sim \gN(0, \sigma_x^2 \mI_d)$, point Gaussian noise $\vdelta_p \sim \gN(0, \sigma_p^2 \mI_{3\times N})$, transformation space $[\vz_{l}, \vz_{u}]$.}
% \KwOut{$bbox, l, \overline{bbox}, \overline{l}, \underline{bbox}, \underline{l} \quad (bbox \in ([X] \times [Y] \times [Z] \times [W] \times [H] \times [L] \times [R]))$}
\KwOut{smoothed prediction $bbox_{median}$, lower bound of IoU between predicted bounding boxes and ground truth bounding boxes $\underline{IoU}$}
\SetKwFunction{Fdetect}{\textsc{Inference}}
\SetKwFunction{FComputeBBoxParams}{\textsc{ComputeBBoxParams}}
\SetKwProg{Fn}{function}{:}{}
\Fn{\Fdetect{$F, \vx_0, \vp_0, \sigma_x, \sigma_p, n$}}{
    $\hat{x} \leftarrow$ \FAddGaussianNoise$(\vx_0, \sigma_x, n)$ \;
    $\hat{p} \leftarrow$ \FAddGaussianNoise$(\vp_0, \sigma_p, n)$ \;
    % $\hat{y} \leftarrow$ toRegression$f(\hat{x},\hat{p})$ \;
    $\hat{bbox} \leftarrow$ \FComputeBBoxParams($F(\hat{x},\hat{p}))$ \;
    $\hat{bbox} \leftarrow$ \FSort$(\hat{bbox})$ \Comment*[r]{\scriptsize sort on each parameter}\;
    $bbox_{median} \leftarrow \hat{bbox}_{\lfloor 0.5n \rfloor}$ \Comment*[r]{\scriptsize take the median of each parameter}\;
    % $bbox,l \leftarrow$ toDetection$(y_{median})$ \;
    \textbf{return} $bbox_{median}$
}
\SetKwFunction{Fcert}{\textsc{Certify}}
\SetKwFunction{FIoU}{\textsc{IoULowerBound}}
\Fn{\Fcert{$F, \vx_0, \vp_0, \sigma_x, \sigma_p, n, [\vz_{sl}, \vz_{su}], c, \alpha$}}{
    $interval\_list \leftarrow$ \FSplitInterval($[\vz_{sl}, \vz_{su}]$) \;
    % $y_{median}\_list$, $y_u\_list$, $y_l\_list \leftarrow$ [], [], []\;
    $IoU\_list\leftarrow$ []\;
    \For{$[z_{sl}, z_{su}] \in$ interval\_list}{
    $\epsilon \leftarrow$ \FComputeEps($T, \vx_0, \vp_0, \sigma_x, \sigma_p, \vz_{l}, \vz_{u}$) \Comment*[r]{\scriptsize compute $\eps$ according to equation (13)}\;
    $q_u, q_l \leftarrow$ \FGetEmpiricalPerc$(n, \epsilon, c, \alpha/|\text{interval\_list}|)$ \Comment*[r]{\scriptsize compute $q_u, q_l$ according to equation (10)}\;
    $\hat{x} \leftarrow$ \FAddGaussianNoise$(T_x(\vx_0, \vz_{sl}), \sigma_x, n)$ \;
    \Comment*[r]{\scriptsize sample $n$ Gaussian noises $\vdelta_x \sim \gN(0, \sigma_x^2 \mI_d)$ and add to $T_x(\vx_0, \vz_{sl})$ to get $n$ noisy samples $\hat{x}$}\;
    $\hat{p} \leftarrow$ \FAddGaussianNoise$(T_p(\vp_0, \vz_{sl}), \sigma_p, n)$ \;
    \Comment*[r]{\scriptsize sample $n$ Gaussian noises $\vdelta_p \sim \gN(0, \sigma_p^2 \mI_{3\times N})$ and add to $(T_p(\vp_0, \vz_{sl})$ to get $n$ noisy samples $\hat{p}$}\;
    % $\hat{y} \leftarrow$ toRegression($F(\hat{x},\hat{p}))$ \;
    $\hat{bbox} \leftarrow$ \FComputeBBoxParams($F(\hat{x},\hat{p}))$ \;
    $\hat{bbox} \leftarrow$ \FSort$(\hat{bbox})$ \Comment*[r]{\scriptsize sort on each parameter}\;
    % $bbox_{median} \leftarrow \hat{bbox}_{\lfloor 0.5n \rfloor}$ \;
    \eIf {$q_l = -1$}{
        $\underline{IoU} \leftarrow -\infty$ \Comment*[r]{\scriptsize $-\infty$ \text{means cannot certify}} \;
        \textbf{return} $\underline{IoU}$
    } {
        $bbox_l \leftarrow \hat{bbox}_{q_l}$
    }
    \eIf {$q_u = \infty$}{
        $\underline{IoU} \leftarrow \infty$ \Comment*[r]{\scriptsize $\infty$ \text{means cannot certify}} \;
        \textbf{return} $\underline{IoU}$
    } {
        $bbox_u \leftarrow \hat{bbox}_{q_u}$
    }
    % $y_u, y_l \leftarrow \hat{y}_{q_u}, \hat{y}_{q_l}$ \;
    % $bbox, l, \overline{bbox}, \overline{l}, \underline{bbox}, \underline{l} \leftarrow$ toDetection$(y_{median}, y_u, y_l)$ \;
    $\underline{IoU} \leftarrow$ \FIoU($\underline{bbox},\overline{bbox},bbox$) \;
    % $y_{median}\_list$ \textbf{add} $y_{median}$, $y_u\_list$ \textbf{add} $y_u$, $y_l\_list$ \textbf{add} $y_l$
    $IoU\_list$ \textbf{add} $\underline{IoU}$
}
    % $\overline{bbox}, \overline{l}, \underline{bbox}, \underline{l} \leftarrow$ toDetection$(\max y_u, \min y_l)$ \;
    % \textbf{return} $ \overline{bbox}, \overline{l}, \underline{bbox}, \underline{l}$
    \textbf{return} $\min(IoU\_list)$
}
\Fn{\FGetEmpiricalPerc$(n, \epsilon, c, \alpha)$}{
    $\underline{p} \leftarrow \Phi \left( \Phi ^{-1}( p) - \epsilon \right)$ \;
    $q_l \leftarrow \FBinarySearchLeft(BinomialCDF(n, \underline{p}), 1-\alpha) $\Comment*[r]{\scriptsize do binary search and choose the left endpoint}\;
    \If {$\text{BinomialCDF}(q_l, \underline{p}) > 1 - \alpha$} {
        $q_l \leftarrow -1$
    } 
    $\overline{p} \leftarrow \Phi \left( \Phi ^{-1}( p) + \epsilon \right)$ \;
    $q_u \leftarrow \FBinarySearchRight(BinomialCDF(n, \overline{p}), \alpha) $\Comment*[r]{\scriptsize do binary search and choose the right endpoint}\;
    \If {$\text{BinomialCDF}(q_u, \overline{p}) < \alpha$} {
        $q_u \leftarrow \infty$
    }
    \textbf{return} $q_u, q_l$
}
\end{algorithm}

\newpage

\begin{algorithm}[H]
    \algsetup{linenosize=\scriptsize}
    \scriptsize
    \DontPrintSemicolon
    \caption{IoU Lower Bound}\label{alg:lower bound of IoU}
    \begin{multicols}{2}
    \KwIn{upper bound of bounding boxes' parameters $\overline{bbox}$, lower bound of bounding boxes' parameters $\underline{bbox}$, ground truth bounding boxes' parameters $bbox$}
    \KwOut{lower bound of IoU between predicted bounding boxes and ground truth bounding boxes $\underline{IoU}$}
    \SetKwFunction{FIoU}{IoULowerBound}
    \SetKwFunction{FIntersectionLowerBound}{IntersectionLowerBound}
    \SetKwFunction{FUnionUpperBound}{UnionUpperBound}
    \SetKwFunction{FCornerXZIntervals}{CornerXZIntervals}
    \SetKwFunction{FAreaDiffUpper}{AreaDiffUpper}
    \SetKwFunction{FOverlapAreaLower}{OverlapAreaLower}
    \SetKwFunction{FComputeXZ}{ComputeXZ}
    \SetKwFunction{FComputeXZInterval}{ComputeXZInterval}
    \SetKwProg{Fn}{function}{:}{}
    \Fn{\FIoU{$\overline{bbox},\underline{bbox},bbox$}}{
        $\underline{V_I} \leftarrow$ \FIntersectionLowerBound($\overline{bbox},\underline{bbox},bbox$) \;
        $\overline{V_U} \leftarrow$ \FUnionUpperBound($\overline{bbox},\underline{bbox},bbox$) \;
        \textbf{return} $\underline{V_I}/\overline{V_U}$
    }
    % \SetKwFunction{FIntersectionLowerBound}{IntersectionLowerBound}
    \Fn{\FIntersectionLowerBound$(\overline{bbox},\underline{bbox},bbox)$}{
        \eIf {$y<(\overline{y}+\underline{y})/2$}{
            $y_l = \overline{y}$
        } {
            $y_l = \underline{y}$
        }
        \eIf {$y\leq y_l-\underline{h}$ $\mathbf{or}$ $y_l\leq y-h$}{
            \textbf{return} $0$
        } {
            $h_{I}=\min(y_l,y)-\max(y_l-\underline{h},y-h)$
        }
        $xz_l,xz_u\leftarrow$ \FCornerXZIntervals($\underline{l},\underline{w},\underline{x},\underline{z},\underline{r},\overline{x},\overline{z},\overline{r}$) \;
        % $\overline{S_{diff}}\leftarrow$ \FAreaDiffUpper($xz_l,xz_u$) \;
        % $\underline{S_{overlap}}\leftarrow$ \FOverlapAreaLower($xz_l,xz_u,\overline{S_{diff}},bbox$) \;
        $\underline{S_{overlap}}\leftarrow$ \FOverlapAreaLower($xz_l,xz_u,bbox$) \;
        \eIf{$\underline{S_{overlap}}\le 0$}{
            \textbf{return} $0$
        }{
            \textbf{return} $h_{I}\cdot \underline{S_{overlap}}$
        }
    }
    \Fn{\FUnionUpperBound$(\overline{bbox},\underline{bbox},bbox)$}{
        \eIf {$y<(\overline{y}+\underline{y})/2$}{
            $y_u = \underline{y}$
        } {
            $y_u = \overline{y}$
        }
        \eIf {$y\leq y_u-\overline{h}$ $\mathbf{or}$ $y_u\leq y-h$}{
            \textbf{return} $0$
        } {
            $h_{I}=\min(y_u,y)-\max(y_u-\overline{h},y-h)$
        }
        $xz_l,xz_u\leftarrow$ \FCornerXZIntervals($\overline{l},\overline{w},\underline{x},\underline{z},\underline{r},\overline{x},\overline{z},\overline{r}$) \;
        % $\overline{S_{diff}}\leftarrow$ \FAreaDiffUpper($xz_l,xz_u$) \;
        % $\underline{S_{overlap}}\leftarrow$ \FOverlapAreaLower($xz_l,xz_u,\overline{S_{diff}},bbox$) \;
        $\underline{S_{overlap}}\leftarrow$ \FOverlapAreaLower($xz_l,xz_u,bbox$) \;
        \eIf{$\underline{S_{overlap}}\le 0$}{
            \textbf{return} $h\cdot w\cdot l + \overline{h}\cdot\overline{w}\cdot\overline{l}$
        }{
            \textbf{return} $h\cdot w\cdot l + \overline{h}\cdot\overline{w}\cdot\overline{l}-h_{I}\cdot \underline{S_{overlap}}$
        }
    }
    \end{multicols}
\vspace{1em}
    \label{alg:iou_lower_bound}
\end{algorithm}
% \vspace{-10pt}

\begin{algorithm}[H]
\algsetup{linenosize=\scriptsize}
\scriptsize
\DontPrintSemicolon
\caption{Corners' $x,z$ Intervals}\label{alg:intervals of corners' xz}
    \KwIn{bounding box length $l$, width $w$, $x$ lower bound $\underline{x}$, $z$ lower bound $\underline{z}$, rotation angel lower bound $\underline{r}$, $x$ upper bound $\overline{x}$, $z$ upper bound $\overline{z}$, rotation upper bound $\overline{r}$}
    \KwOut{bounding box corners' $xz$ coordinates lower bound $xz_l$, bounding box corners' xz coordinates upper bound $xz_u$}
    \SetKwProg{Fn}{function}{:}{}
    \Fn{\FCornerXZIntervals($l,w,\underline{x},\underline{z},\underline{r},\overline{x},\overline{z},\overline{r}$)}{
        $xzCoorsList\leftarrow \left[\right]$ \;
        \For{$x,z$ $\mathbf{in}$ $\left[\left[\underline{x},\underline{z}\right],\left[\overline{x},\overline{z}\right]\right]$ }{
            \For{$r$ $\mathbf{in}$ $\underline{r},\overline{r}$}{
                $xzCoors\leftarrow$ \FComputeXZ($x,z,r$) \;
                $xzCoorsList$ \textbf{add} $xzCoors$                    
            }
        }
        $xz_l,xz_u\leftarrow$ \FComputeXZInterval($xzCoorsList$) \Comment*[r]{\scriptsize compute $x,z$ intervals roughly}\;
        \Comment*[l]{\scriptsize consider extremum cases}
        $\alpha\leftarrow\arctan(w/l)$ \;
        $d\leftarrow\sqrt{\left(l/2\right)^2+\left(w/2\right)^2}$ \;
        \If{$\pi-\alpha\geq\underline{r}$ $\mathbf{and}$ $\pi-\alpha\leq\overline{r}$}{
            $xz_l[0]\leftarrow\underline{x}-d$ \;
            $xz_u[4]\leftarrow\overline{x}+d$
        }
        \If{$2\pi-\alpha\geq\underline{r}$ $\mathbf{and}$ $2\pi-\alpha\leq\overline{r}$}{
            $xz_u[0]\leftarrow\overline{x}+d$ \;
            $xz_l[4]\leftarrow\underline{x}-d$
        }
        \If{$\pi/2-\alpha\geq\underline{r}$ $\mathbf{and}$ $\pi/2-\alpha\leq\overline{r}$}{
            $xz_u[1]\leftarrow\overline{z}+d$ \;
            $xz_l[5]\leftarrow\underline{z}-d$
        }
        \If{$3\pi/2-\alpha\geq\underline{r}$ $\mathbf{and}$ $3\pi/2-\alpha\leq\overline{r}$}{
            $xz_l[1]\leftarrow\underline{z}-d$ \;
            $xz_u[5]\leftarrow\overline{z}+d$
        }
        \If{$\alpha\geq\underline{r}$ $\mathbf{and}$ $\alpha\leq\overline{r}$}{
            $xz_u[2]\leftarrow\overline{x}+d$ \;
            $xz_l[6]\leftarrow\underline{x}-d$
        }
        \If{$\pi+\alpha\geq\underline{r}$ $\mathbf{and}$ $\pi+\alpha\leq\overline{r}$}{
            $xz_l[2]\leftarrow\underline{x}-d$ \;
            $xz_u[6]\leftarrow\overline{x}+d$
        }
        \If{$\pi/2+\alpha\geq\underline{r}$ $\mathbf{and}$ $\pi/2+\alpha\leq\overline{r}$}{
            $xz_u[3]\leftarrow\overline{z}+d$ \;
            $xz_l[7]\leftarrow\underline{z}-d$
        }
        \If{$3\pi/2+\alpha\geq\underline{r}$ $\mathbf{and}$ $3\pi/2+\alpha\leq\overline{r}$}{
            $xz_l[3]\leftarrow\underline{z}-d$ \;
            $xz_u[7]\leftarrow\overline{z}+d$
        }
    }
    \textbf{return} $xz_l,xz_u$
\end{algorithm}

\section{Additional Experimental Details}
\label{append:exp}

% \subsection{Ablation Studies}
% \zijian{Maybe move to appendix?}

\subsection{Dataset}
\label{appendix:dataset}

\begin{table}[!t]
%  \begin{wraptable}{r}{0.5\linewidth}
    \centering
    % \vspace{-2em}
    \caption{Certification data. 
    % This table shows the structure of our certification data. 
    Each row stands for one setting, and the columns ``vehicle color", ``building", ``pedestrian", and ``amount" represent the color of the car, whether buildings exist, whether a pedestrian exists, and the number of corresponding data respectively.}
    % \vspace{2mm}
    \scalebox{1.0}{
    \begin{tabular}{c|c|c|c}
    \hline
        vehicle color & building & pedestrian & amount \\\hline\hline
        blue & yes & no & 15 \\
        red & yes & no & 4 \\
        red & yes & yes & 4 \\
        black & yes & no & 4 \\
        black & yes & yes & 4 \\
        blue & no & no & 15 \\
        red & no & no & 4 \\
        red & no & yes & 4 \\
        black & no & no & 4 \\
        black & no & yes & 4 \\
    \hline
    \end{tabular}}
    \label{tab:certification data}
%  \vspace{-2em}
\end{table}
% \end{wraptable}

As introduced in \Cref{sec:exp}, we generate the certification data via spawning the ego vehicle at a few randomly chosen spawn points.
For settings with 15 spawn points in \Cref{tab:certification data}, the randomly-chosen spawn point index in the CARLA Town01 map is 15, 30, 43, 46, 57, 86, 102, 11, 136, 14, 29, 6, 61, 81, and 88. 
For settings with 4 spawn points in \Cref{tab:certification data}, the randomly-chosen spawn point index is 15, 30, 43, and 46.

\paragraph{Certification setup.}
As reflected in \Cref{lem:upper-bound-mx-mp}, to certify the robustness, we need to partition the transformation's parameter space.
For rotation certification, we split the rotation angle interval $[-30^\circ, 30^\circ]$ uniformly into 600 tiny intervals of 0.1 degree; 
for shifting certification, we split the distance interval $[10, 15]$ uniformly into 500 tiny intervals of 0.01 meter.

\paragraph{Empirical attack setup.}
In \Cref{sec:exp}, we evaluate the empirical robustness, i.e., the robustness under attacks, to show vanilla models' vulnerability and estimate the certification tightness.
For rotation, we conduct the attack by enumerating the lowest detection rate and IoU score among 6000 parameters uniformly sampled with a distance 0.01 degree~(distance = $60^\circ / 6000 = 0.01^\circ$); 
for shifting, we conduct the attack by enumerating the lowest detection rate and IoU score among 5000 parameters uniformly sampled with a distance of 0.001 meter~(distance = $5 / 5000 = 0.001$).

% In this section, we provide more details about the dataset in our experimental evaluation. 

% \begin{itemize}[noitemsep,leftmargin=*]
%     \item \textbf{Gaussian noise augmentation.} To adapt models with Gaussian noises, we augment our training and testing data with Gaussian noises with $\sigma=0.25$ and $\sigma=0.5$. 
%     \item \textbf{Certification data.} The randomly chosen spawn point index in the CARLA Town01 map is 15, 30, 43, 46, 57, 86, 102, 11, 136, 14, 29, 6, 61, 81, and 88. For those scenarios which only use 4 spawn points in \Cref{tab:certification data}, the spawn point index is 15, 30, 43 and 46. For rotation certification data, we use 600 tiny intervals of 0.1 degrees to cover the big interval from $-30^\circ$ to $30^\circ$; for shifting certification data, we use 500 tiny intervals of 0.01 meters to cover the big interval from 10 meters to 15 meters. 
%     \item \textbf{Empirical attack data.} For rotation empirical attack data, we use 6000 tiny intervals of 0.01 degrees to approximate the empirical attack result; for shifting empirical attack data, we use 5000 tiny intervals of 0.001 meters to approximate the empirical attack result. 
% \end{itemize}

\subsection{Detailed Experimental Evaluation}
\label{subsubsec: detailed-experiments}
In this section, we present the complete experimental results, which include rotation transformation (\Cref{appendix-tab:rotation overview 1}) and shifting transformation (\Cref{appendix-tab:shift overview}) considering different thresholds for detection and IoU certification.

\paragraph{Certification against rotation transformation.} As shown in \Cref{appendix-tab:rotation overview 1} and discussed in \Cref{subsec:rotation certification}, the order of robustness against rotation transformation in the detection metric is $\text{FocalsConv}>\text{MonoCon}>\text{CLOCs}>\text{SECOND}$, but the most robust model in the IoU metric is CLOCs. Moreover, with experimental results in different thresholds (\Cref{appendix-tab:rotation overview 1(b)}, \Cref{tab:bound-rotation} and \Cref{fig:rotation-bounds-th}), the performance of models all drop when the threshold or attack radius increase, and all models' certified IoU drop to 0 when $\text{TH}_\text{IoU}\approx0.6$ , which reflects the problem of models' robustness against rotation transformation. We can also find that the performance difference between models increases when the threshold or attack radius increases.

\paragraph{Certification against shifting transformation.} As shown in \Cref{appendix-tab:shift overview} and discussed in \Cref{subsec:shifting certification}, the order of robustness against shifting transformation in the detection metric is $\text{CLOCs}>\text{SECOND}\approx\text{MonoCon}>\text{FocalsConv}$, and the order in the IoU metric is $\text{CLOCs}>\text{MonoCon}>\text{SECOND}>\text{FocalsConv}$. This shows the advantage of fusion models (e.g. CLOCs) on the one hand but also shows this makes the attack space larger on another hand (e.g. FocalsConv), which leads to a new question of robust fusion mechanism design.

\begin{table*}[ht]
    % \tiny
    \fontsize{5.5}{5.5}\selectfont
    \centering
    \caption{Overview of rotation transformation experiment results (smoothing $\sigma=0.25$). Each row represents the corresponding model and attack radius. ``Benign'', ``Adv (Vanilla)'', ``Adv (Smoothed)'', and ``Certification'' stands for benign performance, vanilla models' performance under attacks, smoothed models' performance under attacks, and certified lower bound of smoothed model performance under attacks. Each column represents the results under different thresholds.}
    \begin{subtable}[]{\linewidth}\centering
    \caption{Detection rate under rotation transformation}
    \resizebox{\linewidth}{!}{
    \begin{tabular}{c|c|ccc|ccc|ccc|ccc}
        \hline
        \multirow{2}{*}{Model} & \multirow{2}{*}{Attack Radius} & \multicolumn{3}{c|}{Benign} & \multicolumn{3}{c|}{Adv (Vanilla)} & \multicolumn{3}{c|}{Adv (Smoothed)} & \multicolumn{3}{c}{Certification}\\
        & & \textbf{Det@20} & \textbf{Det@50} & \textbf{Det@80} & \textbf{Det@20} & \textbf{Det@50} & \textbf{Det@80} & \textbf{Det@20} & \textbf{Det@50} & \textbf{Det@80} & \textbf{Det@20} & \textbf{Det@50} & \textbf{Det@80} \\\hline\hline
        \multirowcell{5}{MonoCon\\\cite{liu2022learning}} & $|r|\leq10^\circ$ & \multirow{5}{*}{100.00\%} & \multirow{5}{*}{100.00\%} & \multirow{5}{*}{100.00\%} & 70.97\% & 70.97\% & 58.06\% & 98.39\% & 98.39\% & 80.65\% & 95.16\% & 95.16\% & 75.81\% \\
                                                        & $|r|\leq15^\circ$ & & & & 70.97\% & 70.97\% & 58.06\% & 98.39\% & 98.39\% & 80.65\% & 95.16\% & 95.16\% & 75.81\% \\
                                                        & $|r|\leq20^\circ$ & & & & 70.97\% & 70.97\% & 58.06\% & 98.39\% & 98.39\% & 80.65\% & 95.16\% & 95.16\% & 75.81\% \\
                                                        & $|r|\leq25^\circ$ & & & & 70.97\% & 70.97\% & 45.16\% & 98.39\% & 98.39\% & 80.65\% & 95.16\% & 95.16\% & 75.81\% \\
                                                        & $|r|\leq30^\circ$ & & & & 70.97\% & 70.97\% & 32.26\% & 96.77\% & 96.77\% & 80.65\% & 91.94\% & 91.94\% & 75.81\% \\\hline
        \multirowcell{5}{SECOND\\\cite{yan2018second}} & $|r|\leq10^\circ$ & \multirow{5}{*}{100.00\%} & \multirow{5}{*}{100.00\%} & \multirow{5}{*}{100.00\%} & 100.00\% & 100.00\% & 0.00\% & 100.00\% & 100.00\% & 0.00\% & 100.00\% & 100.00\% & 0.00\% \\
                                                     & $|r|\leq15^\circ$ & & & & 100.00\% & 100.00\% & 0.00\% & 100.00\% & 100.00\% & 0.00\% & 100.00\% & 100.00\% & 0.00\% \\
                                                     & $|r|\leq20^\circ$ & & & & 100.00\% & 100.00\% & 0.00\% & 100.00\% & 100.00\% & 0.00\% & 100.00\% & 100.00\% & 0.00\% \\
                                                     & $|r|\leq25^\circ$ & & & & 100.00\% & 12.90\% & 0.00\% & 100.00\% & 100.00\% & 0.00\% & 100.00\% & 100.00\% & 0.00\% \\
                                                     & $|r|\leq30^\circ$ & & & & 67.74\% & 3.23\% & 0.00\% & 100.00\% & 100.00\% & 0.00\% & 100.00\% & 62.90\% & 0.00\% \\\hline
        \multirowcell{5}{CLOCs\\\cite{pang2020clocs}} & $|r|\leq10^\circ$ & \multirow{5}{*}{100.00\%} & \multirow{5}{*}{100.00\%} & \multirow{5}{*}{100.00\%} & 100.00\% & 100.00\% & 100.00\% & 100.00\% & 100.00\% & 88.71\% & 100.00\% & 100.00\% & 88.71\% \\
                                                    & $|r|\leq15^\circ$ & & & & 100.00\% & 100.00\% & 100.00\% & 98.39\% & 79.03\% & 66.13\% & 98.39\% & 77.42\% & 66.13\% \\
                                                    & $|r|\leq20^\circ$ & & & & 100.00\% & 100.00\% & 100.00\% & 98.39\% & 69.35\% & 50.00\% & 98.39\% & 67.74\% & 50.00\% \\
                                                    & $|r|\leq25^\circ$ & & & & 100.00\% & 91.94\% & 20.97\% & 98.39\% & 69.35\% & 50.00\% & 98.39\% & 67.74\% & 50.00\% \\
                                                    & $|r|\leq30^\circ$ & & & & 100.00\% & 74.19\% & 3.23\% & 98.39\% & 69.35\% & 50.00\% & 98.39\% & 67.74\% & 50.00\% \\\hline
        \multirowcell{5}{FocalsConv\\\cite{chen2022focal}} & $|r|\leq10^\circ$ & \multirow{5}{*}{100.00\%} & \multirow{5}{*}{100.00\%} & \multirow{5}{*}{100.00\%} & 100.00\% & 100.00\% & 100.00\% & 100.00\% & 100.00\% & 100.00\% & 100.00\% & 100.00\%  & 100.00\% \\
                                                         & $|r|\leq15^\circ$ & & & & 100.00\% & 100.00\% & 100.00\% & 100.00\% & 100.00\% & 100.00\% & 100.00\% & 100.00\%  & 100.00\% \\
                                                         & $|r|\leq20^\circ$ & & & & 100.00\% & 100.00\% & 100.00\% & 100.00\% & 100.00\% & 100.00\% & 100.00\% & 100.00\%  & 100.00\% \\
                                                         & $|r|\leq25^\circ$ & & & & 100.00\% & 100.00\% & 100.00\% & 100.00\% & 100.00\% & 100.00\% & 100.00\% & 100.00\%  & 100.00\% \\
                                                         & $|r|\leq30^\circ$ & & & & 100.00\% & 100.00\% & 98.39\% & 100.00\% & 100.00\% & 100.00\% & 100.00\% & 100.00\% & 100.00\% \\\hline
        \end{tabular}
        }
    \label{appendix-tab:rotation overview 1(a)}
    \end{subtable}
    \begin{subtable}[]{\linewidth}\centering
    \caption{IoU with ground truth under rotation transformation}
    \resizebox{\linewidth}{!}{
    \begin{tabular}{c|c|ccc|ccc|ccc|ccc}
        \hline
        \multirow{2}{*}{Model} & \multirow{2}{*}{Attack Radius} & \multicolumn{3}{c|}{Benign} & \multicolumn{3}{c|}{Adv (Vanilla)} & \multicolumn{3}{c|}{Adv (Smoothed)} & \multicolumn{3}{c}{Certification}\\
        & & \textbf{AP@30} & \textbf{AP@50} & \textbf{AP@80} & \textbf{AP@30} & \textbf{AP@50} & \textbf{AP@80} & \textbf{AP@30} & \textbf{AP@50} & \textbf{AP@80} & \textbf{AP@30} & \textbf{AP@50} & \textbf{AP@80} \\\hline\hline
        \multirowcell{5}{MonoCon\\\cite{liu2022learning}} & $|r|\leq10^\circ$ & \multirow{5}{*}{100.00\%} & \multirow{5}{*}{100.00\%} & \multirow{5}{*}{100.00\%} & 56.45\% & 56.45\% & 0.00\% & 82.26\% & 82.26\% & 0.00\% & 75.81\% & 0.00\% & 0.00\% \\
                                                        & $|r|\leq15^\circ$ & & & & 54.84\% & 54.84\% & 0.00\% & 82.26\% & 82.26\% & 0.00\% & 74.19\% & 0.00\% & 0.00\% \\
                                                        & $|r|\leq20^\circ$ & & & & 54.84\% & 53.23\% & 0.00\% & 82.26\% & 74.19\% & 0.00\% & 6.45\% & 0.00\% & 0.00\% \\
                                                        & $|r|\leq25^\circ$ & & & & 51.61\% & 16.13\% & 0.00\% & 80.65\% & 16.13\% & 0.00\% & 0.00\% & 0.00\% & 0.00\% \\
                                                        & $|r|\leq30^\circ$ & & & & 46.77\% & 0.00\% & 0.00\% & 79.03\% & 3.23\% & 0.00\% & 0.00\% & 0.00\% & 0.00\% \\\hline
        \multirowcell{5}{SECOND\\\cite{yan2018second}} & $|r|\leq10^\circ$ & \multirow{5}{*}{100.00\%} & \multirow{5}{*}{100.00\%} & \multirow{5}{*}{100.00\%} & 96.77\% & 96.77\% & 0.00\% & 100.00\% & 100.00\% & 100.00\% & 100.00\% & 100.00\% & 0.00\% \\
                                                     & $|r|\leq15^\circ$ & & & & 96.77\% & 96.77\% & 0.00\% & 100.00\% & 100.00\% & 100.00\% & 100.00\% & 100.00\% & 0.00\% \\
                                                     & $|r|\leq20^\circ$ & & & & 96.77\% & 96.77\% & 0.00\% & 100.00\% & 100.00\% & 54.84\% & 100.00\% & 100.00\% & 0.00\% \\
                                                     & $|r|\leq25^\circ$ & & & & 83.87\% & 83.87\% & 0.00\% & 100.00\% & 96.77\% & 0.00\% & 100.00\% & 0.00\% & 0.00\% \\
                                                     & $|r|\leq30^\circ$ & & & & 51.61\% & 51.61\% & 0.00\% & 54.84\% & 54.84\% & 0.00\% & 11.29\% & 0.00\% & 0.00\% \\\hline
        \multirowcell{5}{CLOCs\\\cite{pang2020clocs}} & $|r|\leq10^\circ$ & \multirow{5}{*}{100.00\%} & \multirow{5}{*}{100.00\%} & \multirow{5}{*}{100.00\%} & 90.32\% & 90.32\% & 90.32\% & 100.00\% & 100.00\% & 98.39\% & 100.00\% & 100.00\% & 0.00\%\\
                                                    & $|r|\leq15^\circ$ & & & & 90.32\% & 90.32\% & 90.32\% & 98.39\% & 98.39\% & 85.48\% & 98.39\% & 87.10\% & 0.00\%\\
                                                    & $|r|\leq20^\circ$ & & & & 88.71\% & 88.71\% & 77.42\% & 98.39\% & 98.39\% & 67.74\% & 98.39\% & 69.35\% & 0.00\%\\
                                                    & $|r|\leq25^\circ$ & & & & 87.10\% & 87.10\% & 0.00\% & 98.39\% & 98.39\% & 67.74\% & 98.39\% & 67.74\% & 0.00\%\\
                                                    & $|r|\leq30^\circ$ & & & & 80.65\% & 80.65\% & 0.00\% & 98.39\% & 98.39\% & 67.74\% & 98.39\% & 53.23\% & 0.00\%\\\hline
        \multirowcell{5}{FocalsConv\\\cite{chen2022focal}} & $|r|\leq10^\circ$ & \multirow{5}{*}{100.00\%} & \multirow{5}{*}{100.00\%} & \multirow{5}{*}{100.00\%} & 96.77\% & 96.77\% & 0.00\% & 100.00\% & 100.00\% & 0.00\% & 100.00\% & 0.00\% & 0.00\% \\
                                                         & $|r|\leq15^\circ$ & & & & 96.77\% & 0.00\% & 0.00\% & 100.00\% & 0.00\% & 0.00\% & 0.00\% & 0.00\% & 0.00\% \\
                                                         & $|r|\leq20^\circ$ & & & & 96.77\% & 0.00\% & 0.00\% & 100.00\% & 0.00\% & 0.00\% & 0.00\% & 0.00\% & 0.00\% \\
                                                         & $|r|\leq25^\circ$ & & & & 96.77\% & 0.00\% & 0.00\% & 100.00\% & 0.00\% & 0.00\% & 0.00\% & 0.00\% & 0.00\% \\
                                                         & $|r|\leq30^\circ$ & & & & 96.77\% & 0.00\% & 0.00\% & 100.00\% & 0.00\% & 0.00\% & 0.00\% & 0.00\% & 0.00\% \\\hline
    \end{tabular}
    }
    \label{appendix-tab:rotation overview 1(b)}
    \end{subtable}
    \label{appendix-tab:rotation overview 1}
\end{table*}

\begin{table*}[ht]
    % \tiny
    \fontsize{5.5}{5.5}\selectfont
    \centering
    \caption{Overview of shifting transformation experiment results. Each row represents the corresponding model and attack radius. ``Benign'', ``Adv (Vanilla)'', ``Adv (Smoothed)'', and ``Certification'' stands for benign performance, vanilla models' performance under attacks, smoothed models' performance under attacks, and certified lower bound of smoothed model performance under attacks. Each column represents the results under different thresholds.}
    \begin{subtable}[]{\linewidth}\centering
    \caption{Detection rate under shifting transformation}
    \resizebox{\linewidth}{!}{
    \begin{tabular}{c|c|ccc|ccc|ccc|ccc}
        \hline
        \multirow{2}{*}{Model} & \multirow{2}{*}{Attack Radius} & \multicolumn{3}{c|}{Benign} & \multicolumn{3}{c|}{Adv (Vanilla)} & \multicolumn{3}{c|}{Adv (Smoothed)} & \multicolumn{3}{c}{Certification}\\
        & & \textbf{Det@20} & \textbf{Det@50} & \textbf{Det@80} & \textbf{Det@20} & \textbf{Det@50} & \textbf{Det@80} & \textbf{Det@20} & \textbf{Det@50} & \textbf{Det@80} & \textbf{Det@20} & \textbf{Det@50} & \textbf{Det@80} \\\hline\hline
        \multirowcell{5}{MonoCon\\\cite{liu2022learning}} & $10\leq z\leq 11$ & \multirow{5}{*}{100.00\%} & \multirow{5}{*}{100.00\%} & \multirow{5}{*}{100.00\%} & 87.10\% & 87.10\% & 66.13\% & 87.10\% & 87.10\% & 66.13\% & 83.87\% & 83.87\% & 64.52\% \\
                                                        & $10\leq z\leq 12$ & & & & 85.48\% &85.48\% & 62.90\% & 85.48\% &85.48\% & 62.90\% & 75.81\% & 75.81\% & 61.29\% \\
                                                        & $10\leq z\leq 13$ & & & & 82.26\% & 82.26\% & 56.45\% & 82.26\% & 82.26\% & 56.45\% & 72.58\% & 72.58\% & 51.61\% \\
                                                        & $10\leq z\leq 14$ & & & & 75.81\% & 75.81\% & 46.77\% & 75.81\% & 75.81\% & 46.77\% & 66.13\% & 66.13\% & 41.94\% \\
                                                        & $10\leq z\leq 15$ & & & & 50.00\% & 50.00\% & 27.42\% & 50.00\% & 50.00\% & 27.42\% & 48.39\% & 48.39\% & 27.42\% \\\hline
        \multirowcell{5}{SECOND\\\cite{yan2018second}} & $10\leq z\leq 11$ & \multirow{5}{*}{100.00\%} & \multirow{5}{*}{100.00\%} & \multirow{5}{*}{100.00\%} & 100.00\% & 70.97\% & 0.00\% & 100.00\% & 70.97\% & 0.00\% & 100.00\% & 0.00\% & 0.00\% \\
                                                     & $10\leq z\leq 12$ & & & & 100.00\% & 70.97\% & 0.00\% & 100.00\% & 70.97\% & 0.00\% & 100.00\% & 0.00\% & 0.00\% \\
                                                     & $10\leq z\leq 13$ & & & & 100.00\% & 12.90\% & 0.00\% & 100.00\% & 70.97\% & 0.00\% & 100.00\% & 0.00\% & 0.00\% \\
                                                     & $10\leq z\leq 14$ & & & & 100.00\% & 12.90\% & 0.00\% & 100.00\% & 70.97\% & 0.00\% & 100.00\% & 0.00\% & 0.00\% \\
                                                     & $10\leq z\leq 15$ & & & & 100.00\% & 12.90\% & 0.00\% & 100.00\% & 70.97\% & 0.00\% & 100.00\% & 0.00\% & 0.00\% \\\hline
        \multirowcell{5}{CLOCs\\\cite{pang2020clocs}} & $10\leq z\leq 11$ & \multirow{5}{*}{100.00\%} & \multirow{5}{*}{100.00\%} & \multirow{5}{*}{100.00\%} & 100.00\% & 100.00\% & 93.54\% & 100.00\% & 100.00\% & 93.54\% & 100.00\% & 100.00\% & 67.74\% \\
                                                    & $10\leq z\leq 12$ & & & & 100.00\% & 100.00\% & 93.54\% & 100.00\% & 100.00\% & 93.54\% & 100.00\% & 100.00\% & 66.13\% \\
                                                    & $10\leq z\leq 13$ & & & & 100.00\% & 100.00\% & 85.45\% & 100.00\% & 100.00\% & 88.71\% & 100.00\% & 100.00\% & 64.52\% \\
                                                    & $10\leq z\leq 14$ & & & & 100.00\% & 100.00\% & 64.52\% & 100.00\% & 100.00\% & 85.48\% & 100.00\% & 100.00\% & 62.90\% \\
                                                    & $10\leq z\leq 15$ & & & & 100.00\% & 100.00\% & 64.52\% & 100.00\% & 100.00\% & 83.87\% & 100.00\% & 100.00\% & 61.29\% \\\hline
        \multirowcell{5}{FocalsConv\\\cite{chen2022focal}} & $10\leq z\leq 11$ & \multirow{5}{*}{100.00\%} & \multirow{5}{*}{100.00\%} & \multirow{5}{*}{100.00\%} & 100.00\% & 100.00\% & 96.77\% & 100.00\% & 100.00\% & 96.77\% & 91.94\% & 88.71\% & 54.84\% \\
                                                         & $10\leq z\leq 12$ & & & & 100.00\% & 100.00\% & 96.77\% & 100.00\% & 100.00\% & 96.77\% & 87.10\% & 79.03\% & 4.84\% \\
                                                         & $10\leq z\leq 13$ & & & & 82.26\% & 22.58\% & 0.00\% & 82.26\% & 22.58\% & 0.00\% & 4.84\% & 0.00\% & 0.00\% \\
                                                         & $10\leq z\leq 14$ & & & & 14.52\% & 0.00\% & 0.00\% & 14.52\% & 0.00\% & 0.00\% & 0.00\% & 0.00\% & 0.00\% \\
                                                         & $10\leq z\leq 15$ & & & & 8.06\% & 0.00\% & 0.00\% & 8.06\% & 0.00\% & 0.00\% & 0.00\% & 0.00\% & 0.00\% \\\hline
    \end{tabular}
    }
    \label{appendix-tab:shift overview (a)}
    \end{subtable}
    \begin{subtable}[]{\linewidth}\centering
    \caption{IoU with ground truth under shifting transformation}
    \resizebox{\linewidth}{!}{
    \begin{tabular}{c|c|ccc|ccc|ccc|ccc}
        \hline
        \multirow{2}{*}{Model} & \multirow{2}{*}{Attack Radius} & \multicolumn{3}{c|}{Benign} & \multicolumn{3}{c|}{Adv (Vanilla)} & \multicolumn{3}{c|}{Adv (Smoothed)} & \multicolumn{3}{c}{Certification}\\
        & & \textbf{AP@30} & \textbf{AP@50} & \textbf{AP@80} & \textbf{AP@30} & \textbf{AP@50} & \textbf{AP@80} & \textbf{AP@30} & \textbf{AP@50} & \textbf{AP@80} & \textbf{AP@30} & \textbf{AP@50} & \textbf{AP@80} \\\hline\hline
        \multirowcell{5}{MonoCon\\\cite{liu2022learning}} & $10\leq z\leq 11$ & \multirow{5}{*}{100.00\%} & \multirow{5}{*}{100.00\%} & \multirow{5}{*}{100.00\%} & 77.42\% & 77.42\% & 41.94\% & 77.42\% & 77.42\% & 41.94\% & 74.19\% & 41.94\% & 0.00\% \\
                                                        & $10\leq z\leq 12$ & & & & 74.19\% & 74.19\% & 0.00\% & 74.19\% & 74.19\% & 0.00\% & 69.35\% & 1.61\% & 0.00\% \\
                                                        & $10\leq z\leq 13$ & & & & 72.58\% & 72.58\% & 0.00\% & 72.58\% & 72.58\% & 0.00\% & 61.29\% & 0.00\% & 0.00\% \\
                                                        & $10\leq z\leq 14$ & & & & 40.32\% & 33.87\% & 0.00\% & 40.32\% & 33.87\% & 0.00\% & 20.97\% & 0.00\% & 0.00\% \\
                                                        & $10\leq z\leq 15$ & & & & 6.45\% & 1.61\% & 0.00\% & 6.45\% & 1.61\% & 0.00\% & 0.00\% & 0.00\% & 0.00\% \\\hline
        \multirowcell{5}{SECOND\\\cite{yan2018second}} & $10\leq z\leq 11$ & \multirow{5}{*}{100.00\%} & \multirow{5}{*}{100.00\%} & \multirow{5}{*}{100.00\%} & 93.55\% & 93.55\% & 93.55\% & 100.00\% & 100.00\% & 0.00\% & 0.00\% & 0.00\% & 0.00\% \\
                                                     & $10\leq z\leq 12$ & & & & 93.55\% & 93.55\% & 93.55\% & 100.00\% & 100.00\% & 0.00\% & 0.00\% & 0.00\% & 0.00\% \\
                                                     & $10\leq z\leq 13$ & & & & 87.10\% & 87.10\% & 87.10\% & 100.00\% & 100.00\% & 0.00\% & 0.00\% & 0.00\% & 0.00\% \\
                                                     & $10\leq z\leq 14$ & & & & 87.10\% & 87.10\% & 87.10\% & 100.00\% & 100.00\% & 0.00\% & 0.00\% & 0.00\% & 0.00\% \\
                                                     & $10\leq z\leq 15$ & & & &87.10\% & 87.10\% & 87.10\% & 100.00\% & 100.00\% & 0.00\% & 0.00\% & 0.00\% & 0.00\% \\\hline
        \multirowcell{5}{CLOCs\\\cite{pang2020clocs}} & $10\leq z\leq 11$ & \multirow{5}{*}{100.00\%} & \multirow{5}{*}{100.00\%} & \multirow{5}{*}{100.00\%} & 93.55\% & 93.55\% & 93.55\% & 100.00\% & 100.00\% & 100.00\% & 79.03\% & 79.03\% & 0.00\% \\
                                                    & $10\leq z\leq 12$ & & & & 80.65\% & 80.65\% & 80.65\% & 80.65\% & 80.65\% & 80.65\% & 51.61\% & 51.61\% & 0.00\% \\
                                                    & $10\leq z\leq 13$ & & & & 80.65\% & 80.65\% & 77.42\% & 80.65\% & 80.65\% & 77.42\% & 51.61\% & 48.39\% & 0.00\% \\
                                                    & $10\leq z\leq 14$ & & & & 80.65\% & 80.65\% & 77.42\% & 80.65\% & 80.65\% & 77.42\% & 51.61\% & 48.39\% & 0.00\% \\
                                                    & $10\leq z\leq 15$ & & & & 80.65\% & 80.65\% & 77.42\% & 80.65\% & 80.65\% & 77.42\% & 51.61\% & 48.39\% & 0.00\% \\\hline
        \multirowcell{5}{FocalsConv\\\cite{chen2022focal}} & $10\leq z\leq 11$ & \multirow{5}{*}{100.00\%} & \multirow{5}{*}{100.00\%} & \multirow{5}{*}{100.00\%} & 97.77\% & 0.00\% & 0.00\% & 100.00\% & 100.00\% & 0.00\% & 85.48\% & 0.00\% & 0.00\% \\
                                                         & $10\leq z\leq 12$ & & & & 0.00\% & 0.00\% & 0.00\% & 100.00\% & 100.00\% & 0.00\% & 83.87\% & 0.00\% & 0.00\% \\
                                                         & $10\leq z\leq 13$ & & & & 0.00\% & 0.00\% & 0.00\% & 82.26\% & 82.26\% & 0.00\% & 4.84\% & 0.00\% & 0.00\% \\
                                                         & $10\leq z\leq 14$ & & & & 0.00\% & 0.00\% & 0.00\% & 14.52\% & 14.52\% & 0.00\% & 0.00\% & 0.00\% & 0.00\% \\
                                                         & $10\leq z\leq 15$ & & & & 0.00\% & 0.00\% & 0.00\% & 8.06\% & 8.06\% & 0.00\% & 0.00\% & 0.00\% & 0.00\% \\\hline
    \end{tabular}
    }
    \label{appendix-tab:shift overview (b)}
    \end{subtable}
    \label{appendix-tab:shift overview}
\end{table*}

{
\renewcommand{\thesubfigure}{\alph{subfigure}}
% {\refstepcounter{subfigure}\textbf{(\thesubfigure) }{\ignorespaces #1}}

\begin{figure*}[ht]
% \vspace{-2em}

% \newlength{\utilheightc}
\settoheight{\utilheightc}{\includegraphics[width=.160\linewidth]{rotation_conf_0.2.pdf}}%

% \newlength{\utilheightd}
\settoheight{\utilheightd}{\includegraphics[width=.165\linewidth]{rotation_IoU_0.3.pdf}}%

% \newlength{\utilheightaa}
% \settoheight{\utilheightaa}{\includegraphics[width=.162\linewidth]{Freeway_adasearch_sigma-0.05.pdf}}%

% \newlength{\utilheightcp}
% \settoheight{\utilheightcp}{\includegraphics[width=.160\linewidth]{Pong_global_mean_sigma-0.001.pdf}}%

% \newlength{\utilheightdp}
% \settoheight{\utilheightdp}{\includegraphics[width=.165\linewidth]{Pong_global_median_sigma-0.001.pdf}}%

% \newlength{\utilheightaap}
% \settoheight{\utilheightaap}{\includegraphics[width=.162\linewidth]{Pong_adasearch_sigma-0.05.pdf}}%

% \newlength{\utilheight}
% \settoheight{\utilheight}{\includegraphics[width=.138\linewidth]{twitch-DE: low degree.pdf}}%

% \newlength{\attackheightb}
% \settoheight{\attackheightb}{\includegraphics[width=.138\linewidth]{twitch-DE: high degree.pdf}}%

% \newlength{\legendheightb}
\setlength{\legendheightb}{0.48\utilheightc}%

\newcommand{\rownamec}[1]% #1 = text
{\rotatebox{90}{\makebox[\utilheightc][c]{\tiny #1}}}

\newcommand{\rownamed}[1]% #1 = text
{\rotatebox{90}{\makebox[\utilheightd][c]{\tiny #1}}}

\centering

{
\renewcommand{\tabcolsep}{10pt}

\begin{subtable}[]{\linewidth}
\begin{tabular}{l}
\includegraphics[height=\legendheightb]{legend_marker.pdf}
\end{tabular}
\end{subtable}
\vspace{-2pt}
\begin{subtable}{\linewidth}
\centering
\resizebox{\linewidth}{!}{%
\begin{tabular}{@{}p{3mm}@{}c@{}c@{}c@{}c@{}c@{}c@{}}
        & \makecell{\tiny{$\text{TH}_\text{conf}=0.2$}}
        & \makecell{\tiny{$\text{TH}_\text{conf}=0.5$}}
        & \makecell{\tiny{$\text{TH}_\text{conf}=0.8$}}
        & \makecell{\tiny{$\text{TH}_\text{IoU}=0.3$}}
        & \makecell{\tiny{$\text{TH}_\text{IoU}=0.5$}}
        & \makecell{\tiny{$\text{TH}_\text{IoU}=0.8$}}
        \vspace{-2pt}\\
\rownamec{\makecell{Detection}}&
\includegraphics[height=\utilheightc]{rotation_conf_0.2.pdf}&
\includegraphics[height=\utilheightc]{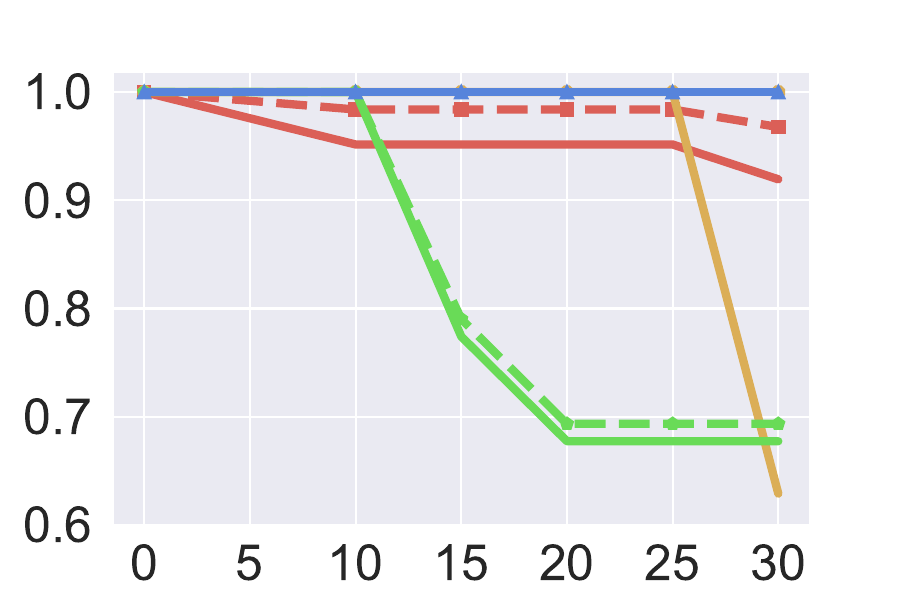}&
\includegraphics[height=\utilheightc]{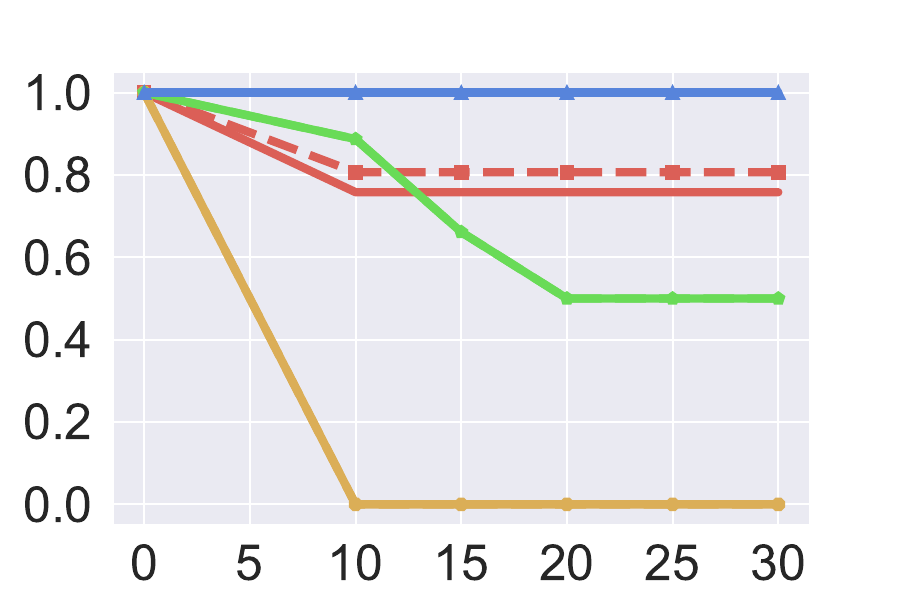}
\rownamec{\makecell{IoU}}&
\includegraphics[height=\utilheightd]{rotation_IoU_0.3.pdf}&
\includegraphics[height=\utilheightd]{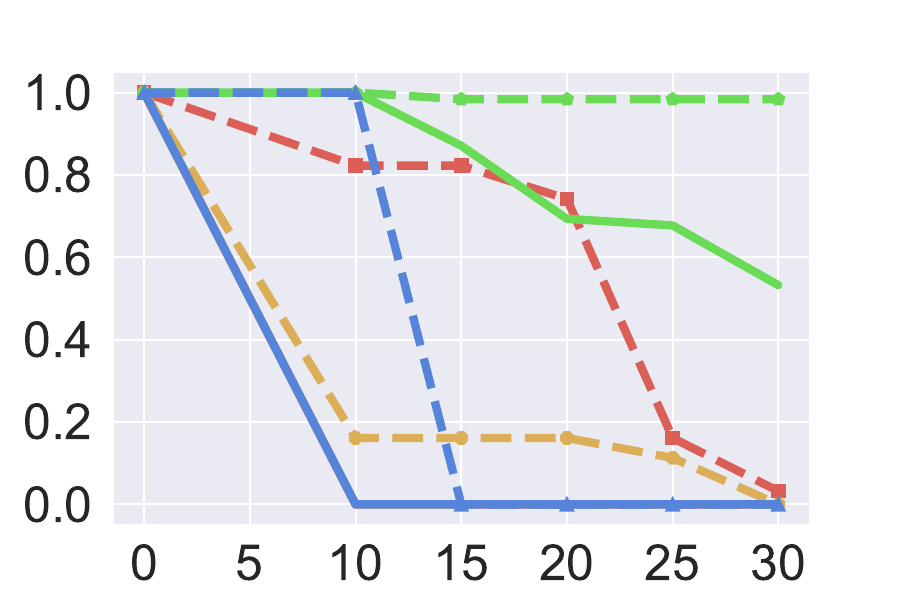}&
\includegraphics[height=\utilheightd]{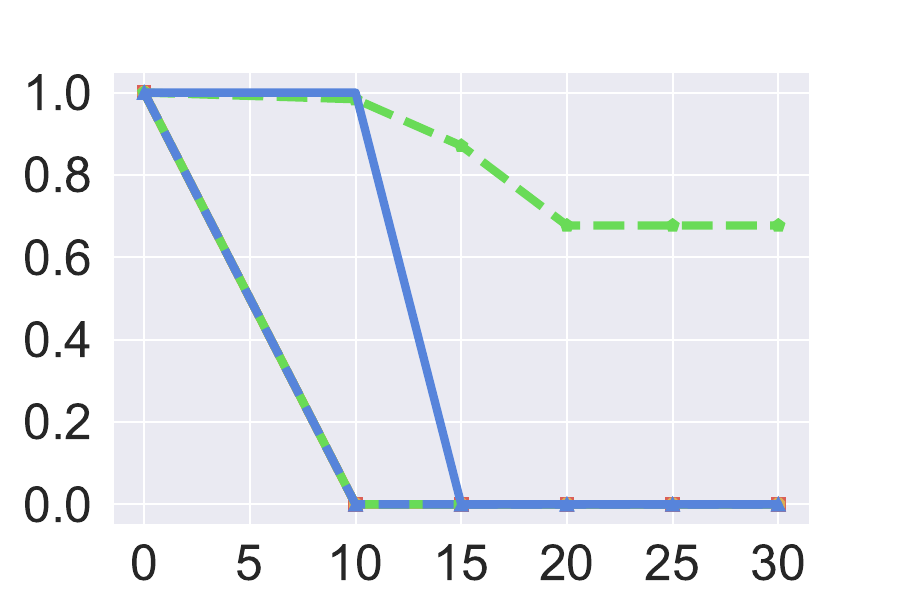}\\[-1.2ex]
% \rownamed{\makecell{detection}}&
% \includegraphics[height=\utilheightd]{rotation_IoU_0.3.pdf}&
% \includegraphics[height=\utilheightd]{rotation_IoU_0.5.pdf}&
% \includegraphics[height=\utilheightd]{Freeway_global_median_sigma-0.1_cmp.pdf}&
% \includegraphics[height=\utilheightd]{Freeway_global_median_sigma-0.5_cmp.pdf}&
% \includegraphics[height=\utilheightd]{Freeway_global_median_sigma-0.75_cmp.pdf}&
% \includegraphics[height=\utilheightd]{rotation_IoU_0.7.pdf}\\[-1.2ex]
% \rownamed{\makecell{(Local) $\uj$}}&
% \includegraphics[height=\utilheightaa]{Freeway_adasearch_sigma-0.005_cmp.pdf}&
% \includegraphics[height=\utilheightaa]{Freeway_adasearch_sigma-0.05_cmp.pdf}&
% \includegraphics[height=\utilheightaa]{Freeway_adasearch_sigma-0.1_cmp.pdf}&
% \includegraphics[height=\utilheightaa]{Freeway_adasearch_sigma-0.5_cmp.pdf}&
% \includegraphics[height=\utilheightaa]{Freeway_adasearch_sigma-0.75_cmp.pdf}&
% \includegraphics[height=\utilheightaa]{Freeway_adasearch_sigma-1.0_cmp.pdf}\\[-1.2ex]
        & \makecell{\tiny{attack $\epsilon$}}
        & \makecell{\tiny{attack $\epsilon$}}
        & \makecell{\tiny{attack $\epsilon$}}
        & \makecell{\tiny{attack $\epsilon$}}
        & \makecell{\tiny{attack $\epsilon$}}
        & \makecell{\tiny{attack $\epsilon$}}
\end{tabular}
}
% \vspace{-3mm}
\caption{\small Robustness certification for rotation transformation (smoothing $\sigma=0.25$)}\label{tab:bound-rotation}
\end{subtable}

\begin{subtable}{\linewidth}
\centering
\resizebox{\linewidth}{!}{%
\begin{tabular}{@{}p{3mm}@{}c@{}c@{}c@{}c@{}c@{}c@{}}
        & \makecell{\tiny{$\text{TH}_\text{conf}=0.2$}}
        & \makecell{\tiny{$\text{TH}_\text{conf}=0.5$}}
        & \makecell{\tiny{$\text{TH}_\text{conf}=0.8$}}
        & \makecell{\tiny{$\text{TH}_\text{IoU}=0.3$}}
        & \makecell{\tiny{$\text{TH}_\text{IoU}=0.5$}}
        & \makecell{\tiny{$\text{TH}_\text{IoU}=0.8$}}\\
        % \vspace{-2pt}\\
\rownamec{\makecell{Detection}}&
\includegraphics[height=\utilheightc]{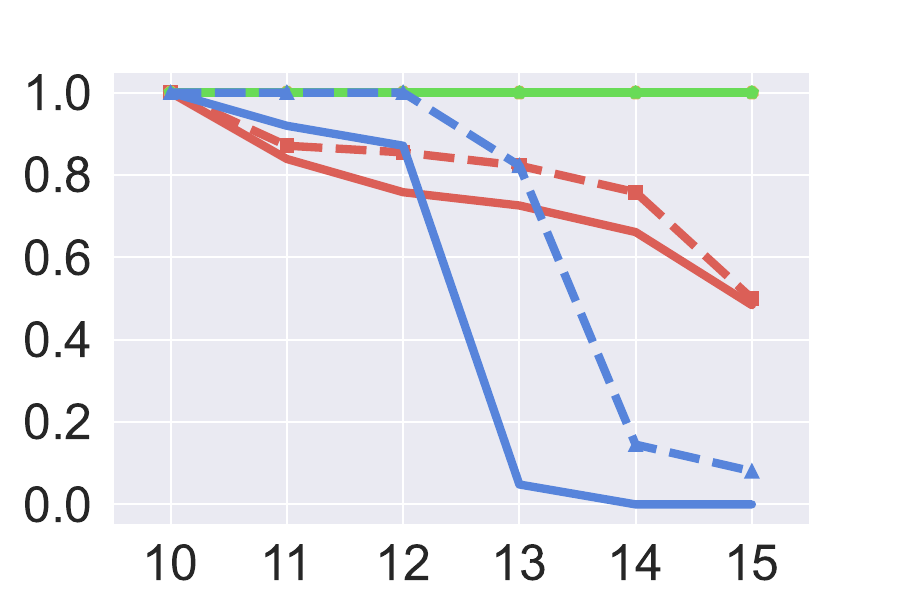}&
\includegraphics[height=\utilheightc]{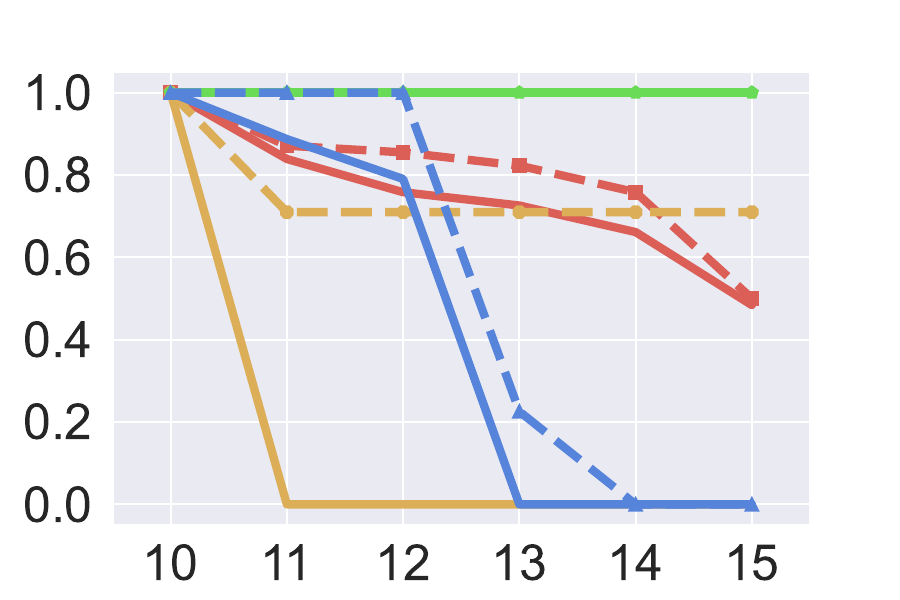}&
\includegraphics[height=\utilheightc]{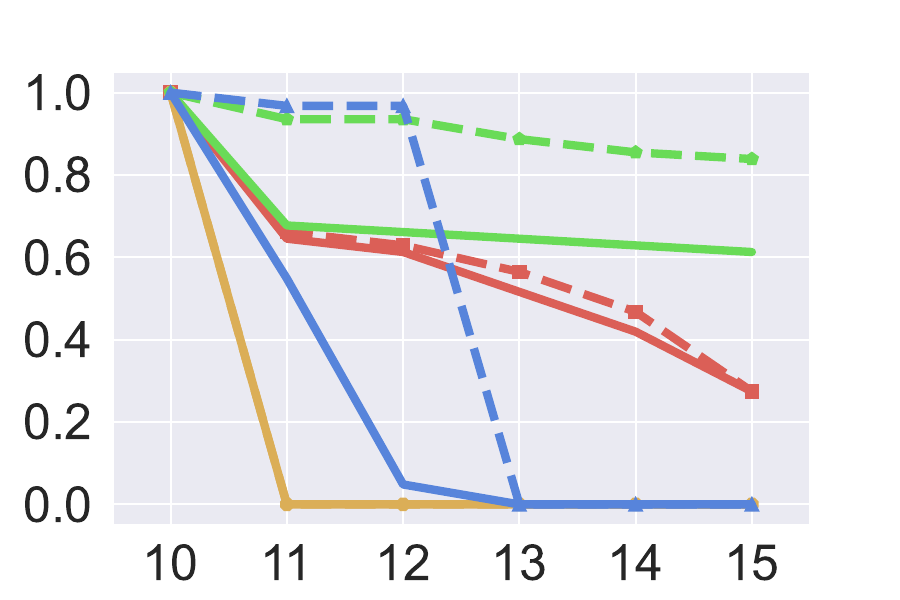}
\rownamec{\makecell{IoU}}&
\includegraphics[height=\utilheightd]{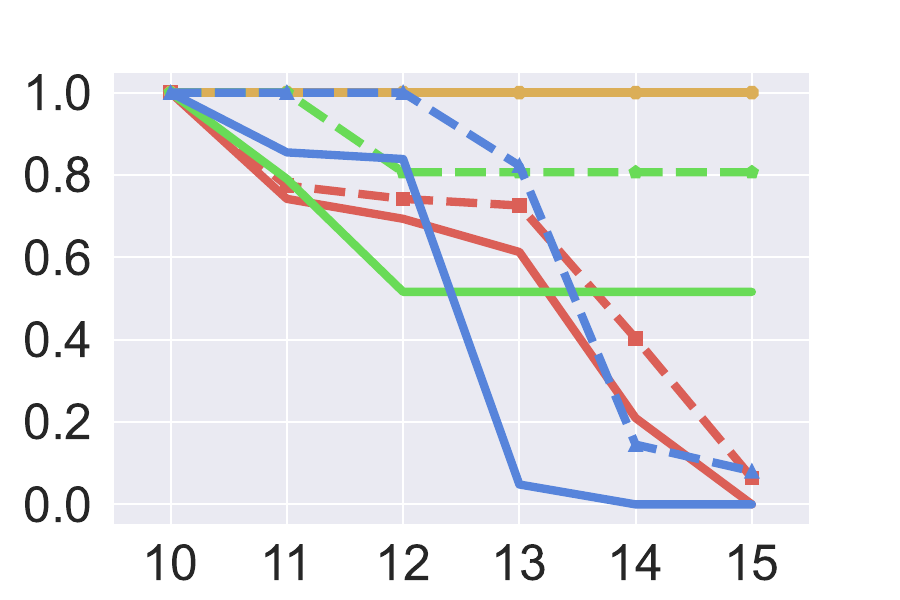}&
\includegraphics[height=\utilheightd]{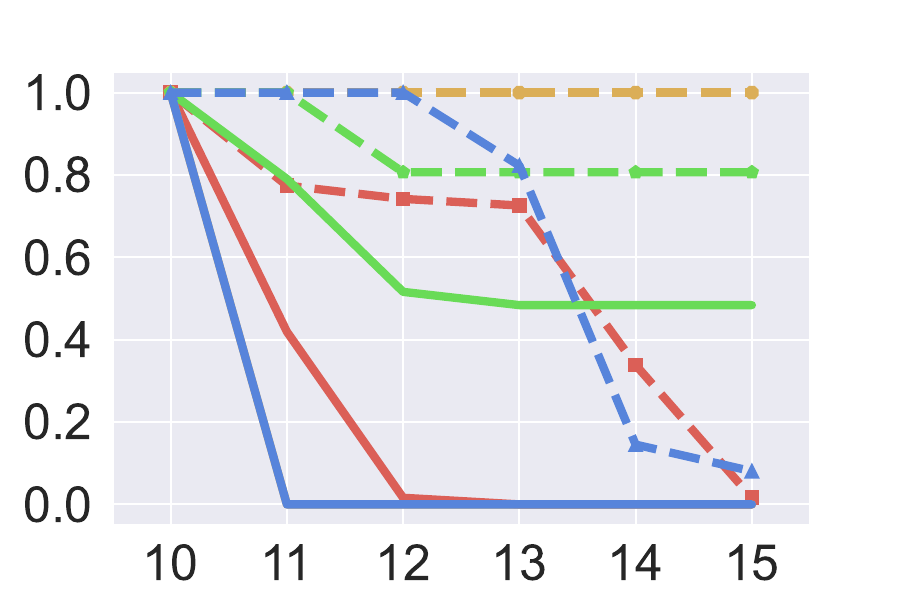}&
\includegraphics[height=\utilheightd]{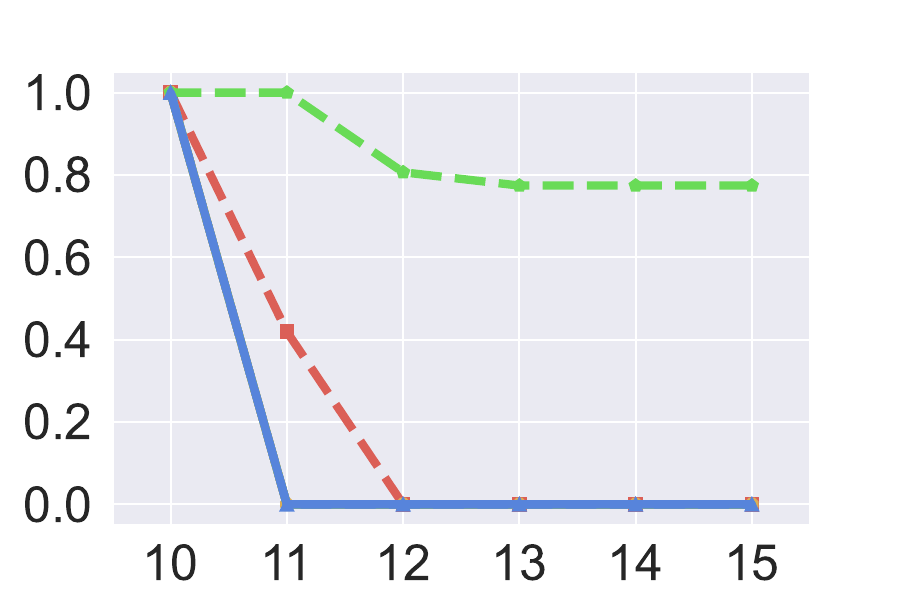}\\[-1.2ex]
        & \makecell{\tiny{attack $\epsilon$}}
        & \makecell{\tiny{attack $\epsilon$}}
        & \makecell{\tiny{attack $\epsilon$}}
        & \makecell{\tiny{attack $\epsilon$}}
        & \makecell{\tiny{attack $\epsilon$}}
        & \makecell{\tiny{attack $\epsilon$}}
\end{tabular}
}
% \vspace{-3mm}
\caption{\small Robustness certification for shift transformation}\label{tab:bound-shift}
\end{subtable}

\begin{subtable}{\linewidth}
\centering
\resizebox{\linewidth}{!}{%
\begin{tabular}{@{}p{3mm}@{}c@{}c@{}c@{}c@{}c@{}c@{}}
        & \makecell{\tiny{$\text{TH}_\text{conf}=0.2$}}
        & \makecell{\tiny{$\text{TH}_\text{conf}=0.5$}}
        & \makecell{\tiny{$\text{TH}_\text{conf}=0.8$}}
        & \makecell{\tiny{$\text{TH}_\text{IoU}=0.3$}}
        & \makecell{\tiny{$\text{TH}_\text{IoU}=0.5$}}
        & \makecell{\tiny{$\text{TH}_\text{IoU}=0.8$}}\\
        % \vspace{-2pt}\\
\rownamec{\makecell{Detection}}&
\includegraphics[height=\utilheightc]{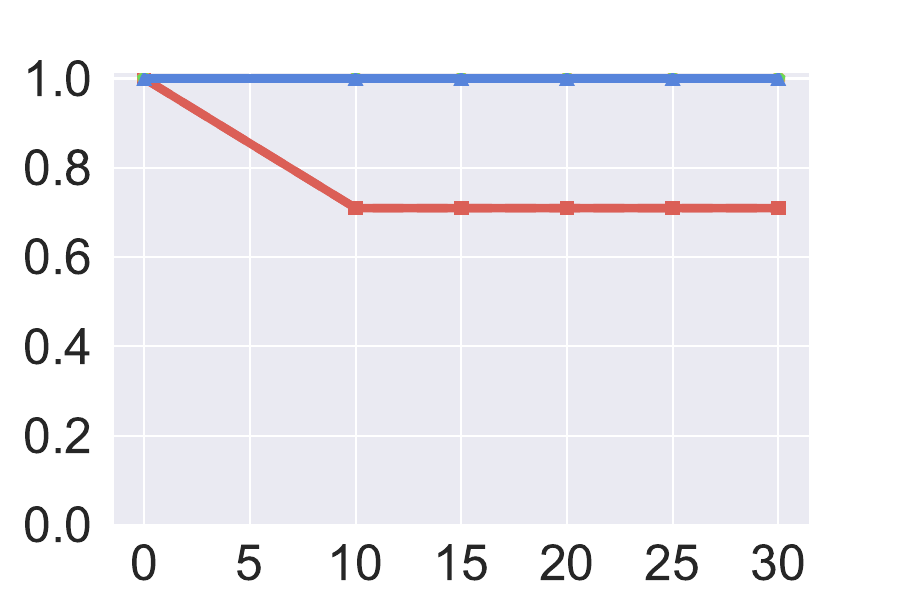}&
\includegraphics[height=\utilheightc]{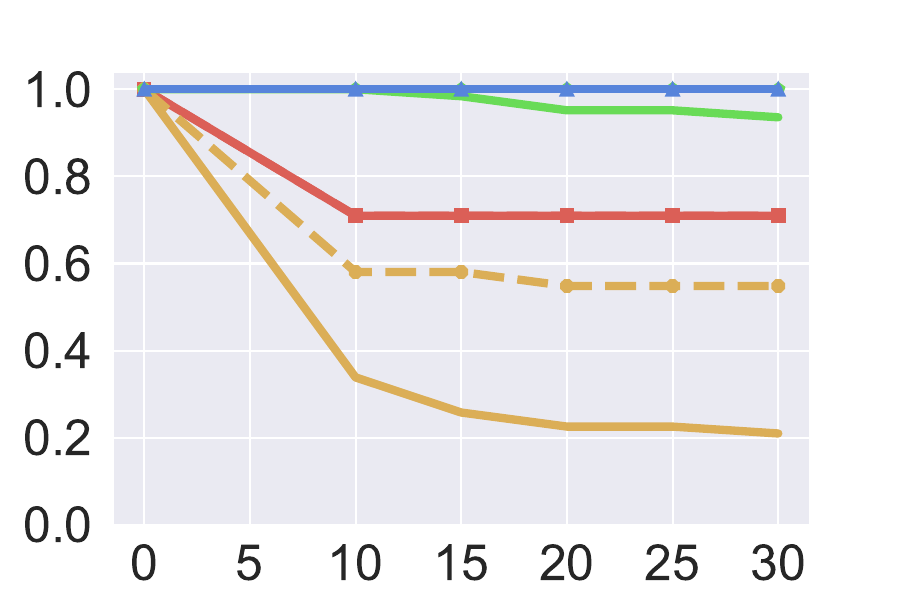}&
\includegraphics[height=\utilheightc]{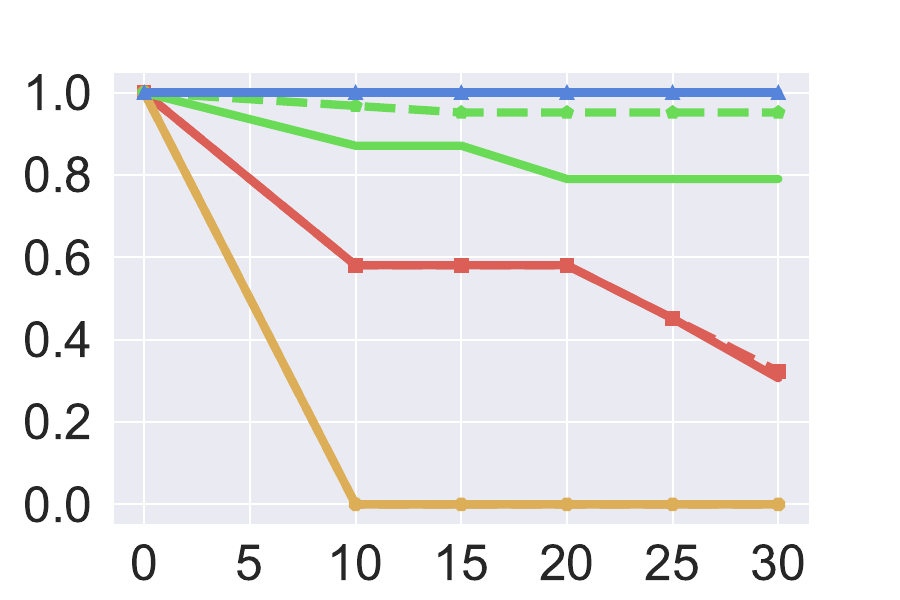}
\rownamec{\makecell{IoU}}&
\includegraphics[height=\utilheightd]{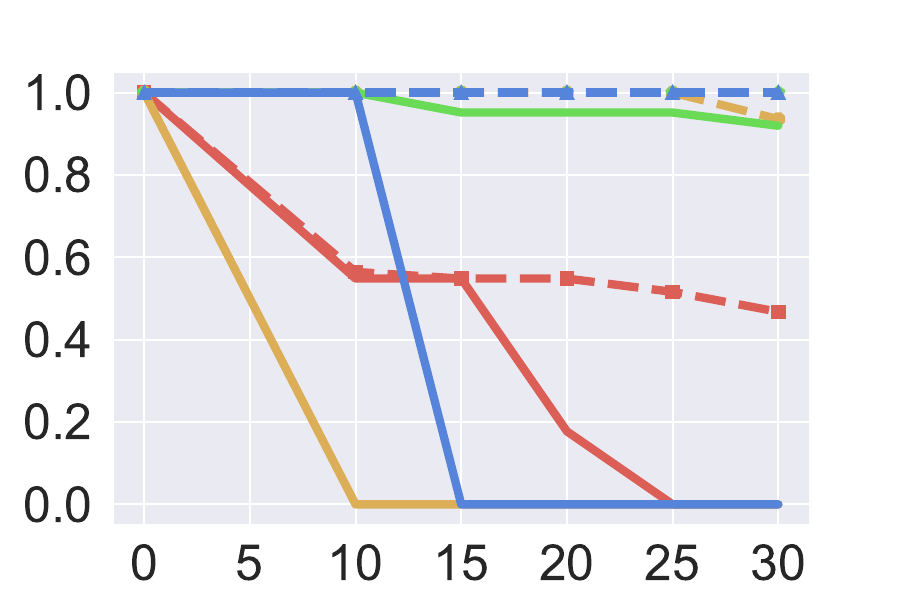}&
\includegraphics[height=\utilheightd]{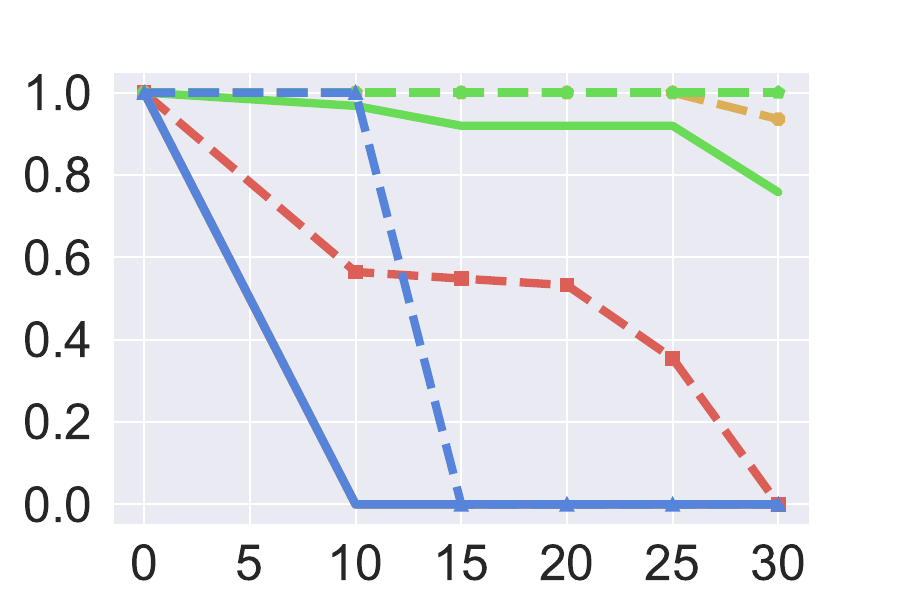}&
\includegraphics[height=\utilheightd]{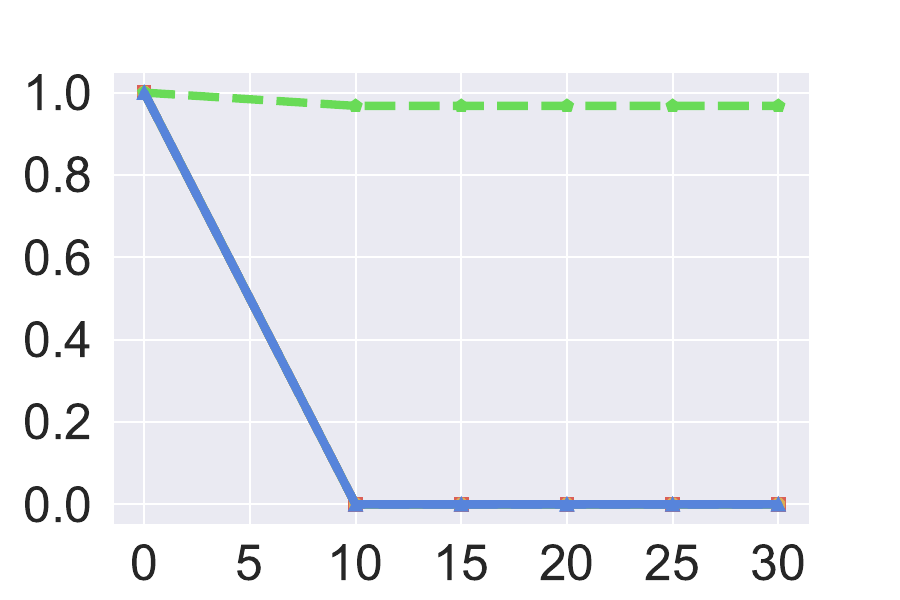}\\[-1.2ex]
        & \makecell{\tiny{attack $\epsilon$}}
        & \makecell{\tiny{attack $\epsilon$}}
        & \makecell{\tiny{attack $\epsilon$}}
        & \makecell{\tiny{attack $\epsilon$}}
        & \makecell{\tiny{attack $\epsilon$}}
        & \makecell{\tiny{attack $\epsilon$}}
\end{tabular}
}
% \vspace{-3mm}
\caption{\small Robustness certification for rotation transformation (smoothing $\sigma=0.5$)}\label{tab:bound-rotation-0.5}
\end{subtable}

}

\caption{\small Robustness certification for rotation and shifting transformation, including detection rate bound and IoU bound. Solid lines represent the certified bounds, and dashed lines show the empirical performance under PGD.
% \caption{\small Robustness certification for cumulative reward, including \textit{expectation bound} $\uje$, \textit{percentile bound} $\ujp$ ($p=50\%$), and \textit{absolute lower bound} $\uj$. Each column corresponds to one smoothing variance. Solid lines represent the certified reward bounds, and dashed lines show the empirical performance under PGD.
% shows the change of the lower bounds w.r.t. the attack magnitude $\eps$.
% The solid curves represent the certified lower bounds, while the dotted lines represent the empirical values.
}%
\label{fig:all-bounds}
% \vspace{-3em}
\end{figure*}
}

{
\renewcommand{\thesubfigure}{\alph{subfigure}}
% {\refstepcounter{subfigure}\textbf{(\thesubfigure) }{\ignorespaces #1}}

\begin{figure*}[ht]
% \vspace{-2em}

% \newlength{\utilheightc}
\settoheight{\utilheightc}{\includegraphics[width=.160\linewidth]{rotation_conf_0.2.pdf}}%

% \newlength{\utilheightd}
\settoheight{\utilheightd}{\includegraphics[width=.165\linewidth]{rotation_IoU_0.3.pdf}}%

% \newlength{\utilheightaa}
% \settoheight{\utilheightaa}{\includegraphics[width=.162\linewidth]{Freeway_adasearch_sigma-0.05.pdf}}%

% \newlength{\utilheightcp}
% \settoheight{\utilheightcp}{\includegraphics[width=.160\linewidth]{Pong_global_mean_sigma-0.001.pdf}}%

% \newlength{\utilheightdp}
% \settoheight{\utilheightdp}{\includegraphics[width=.165\linewidth]{Pong_global_median_sigma-0.001.pdf}}%

% \newlength{\utilheightaap}
% \settoheight{\utilheightaap}{\includegraphics[width=.162\linewidth]{Pong_adasearch_sigma-0.05.pdf}}%

% \newlength{\utilheight}
% \settoheight{\utilheight}{\includegraphics[width=.138\linewidth]{twitch-DE: low degree.pdf}}%

% \newlength{\attackheightb}
% \settoheight{\attackheightb}{\includegraphics[width=.138\linewidth]{twitch-DE: high degree.pdf}}%

% \newlength{\legendheightb}
\setlength{\legendheightb}{0.48\utilheightc}%

\newcommand{\rownamec}[1]% #1 = text
{\rotatebox{90}{\makebox[\utilheightc][c]{\tiny #1}}}

\newcommand{\rownamed}[1]% #1 = text
{\rotatebox{90}{\makebox[\utilheightd][c]{\tiny #1}}}

\centering

{
\renewcommand{\tabcolsep}{10pt}

\begin{subtable}[]{\linewidth}
\begin{tabular}{l}
\includegraphics[height=\legendheightb]{legend_marker.pdf}
\end{tabular}
\end{subtable}
\vspace{-2pt}

\begin{subtable}{\linewidth}
\centering
\resizebox{\linewidth}{!}{%
\begin{tabular}{@{}p{3mm}@{}c@{}c@{}c@{}c@{}c@{}}
        & \makecell{\tiny{$10\leq z\leq11$}}
        & \makecell{\tiny{$10\leq z\leq12$}}
        & \makecell{\tiny{$10\leq z\leq13$}}
        & \makecell{\tiny{$10\leq z\leq14$}}
        & \makecell{\tiny{$10\leq z\leq15$}}
        % & \makecell{\tiny{$\text{TH}_\text{IoU}=0.8$}}
        \vspace{-2pt}\\
\rownamec{\makecell{Detection}}&
\includegraphics[height=\utilheightc]{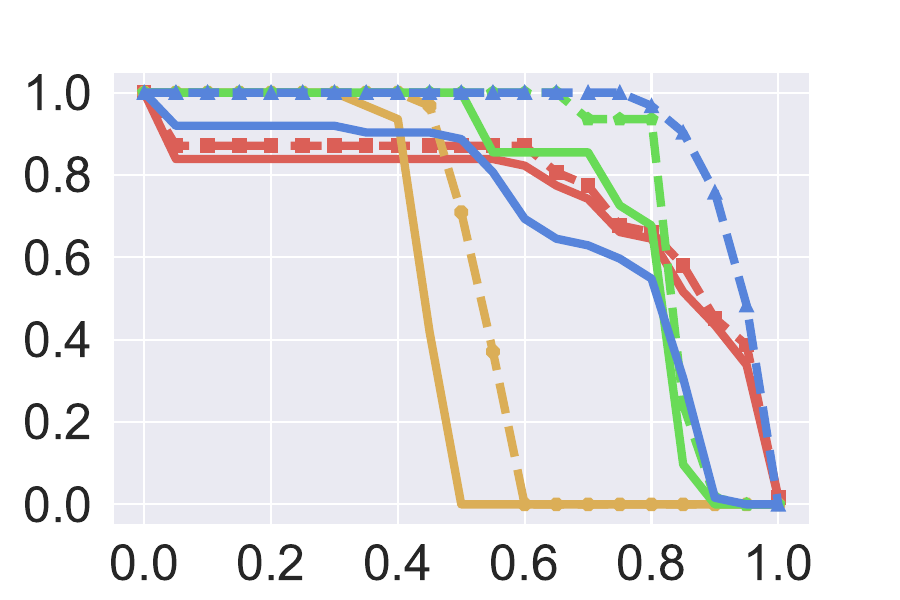}&
\includegraphics[height=\utilheightc]{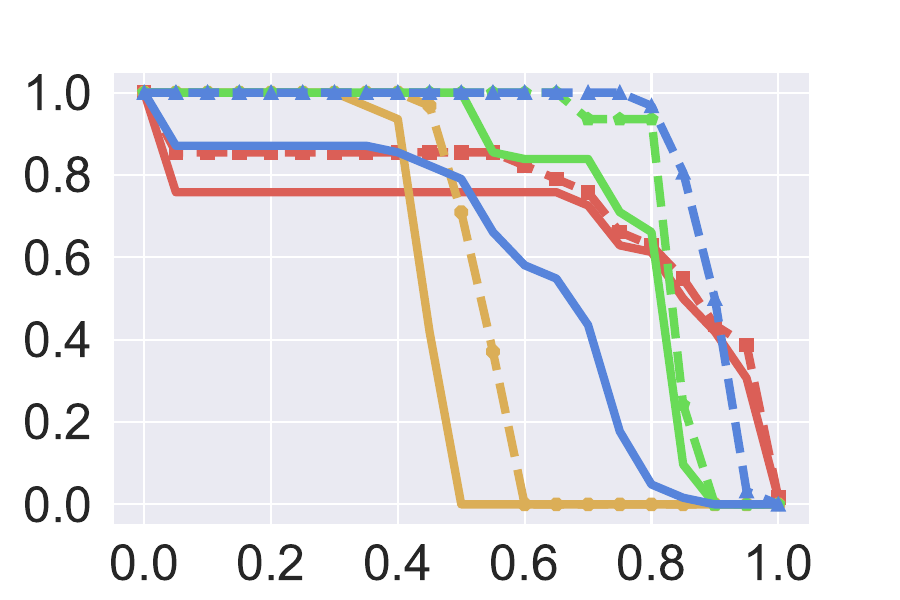}&
\includegraphics[height=\utilheightc]{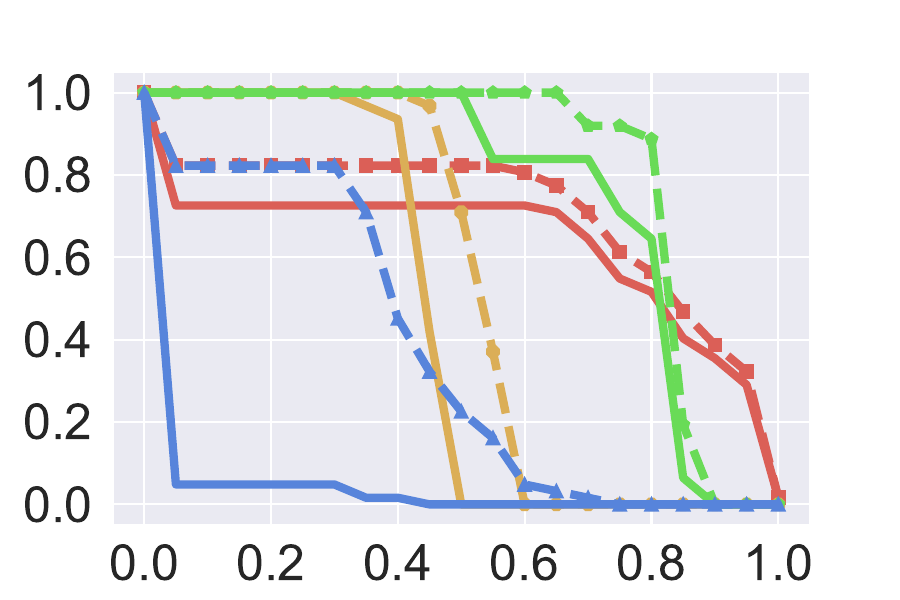}&
\includegraphics[height=\utilheightc]{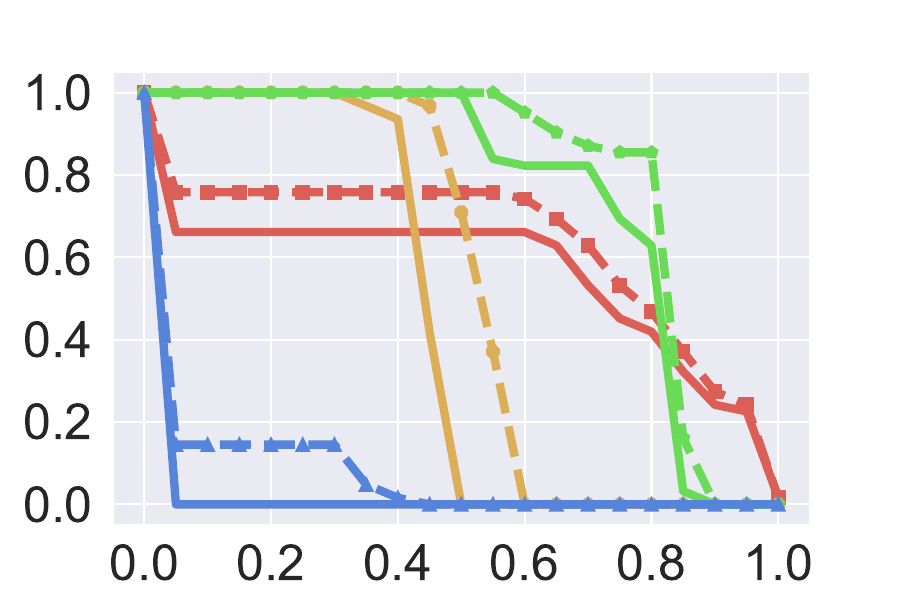}&
\includegraphics[height=\utilheightc]{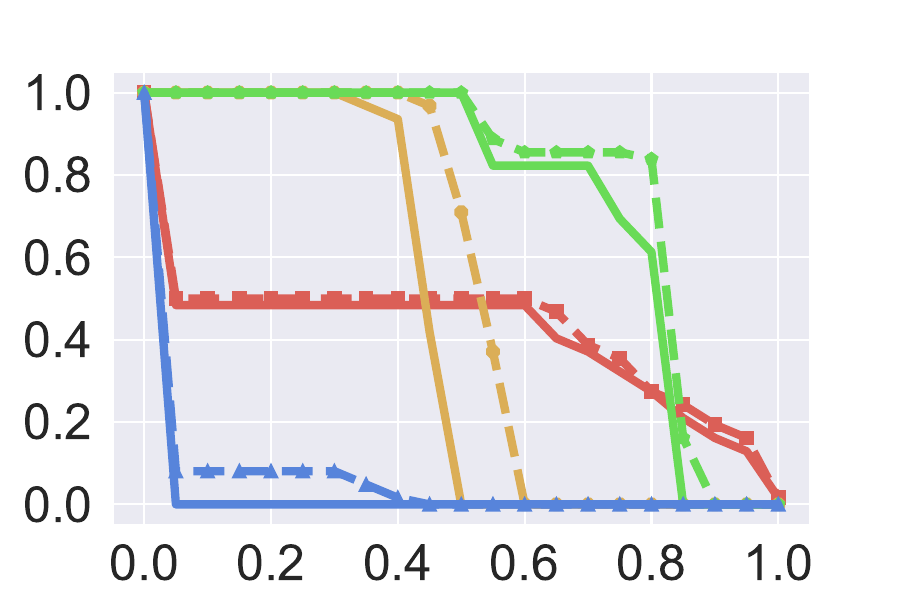}\\[-1.2ex]
        & \makecell{\tiny{$\text{TH}_\text{conf}$}}
        & \makecell{\tiny{$\text{TH}_\text{conf}$}}
        & \makecell{\tiny{$\text{TH}_\text{conf}$}}
        & \makecell{\tiny{$\text{TH}_\text{conf}$}}
        & \makecell{\tiny{$\text{TH}_\text{conf}$}}\\
\rownamec{\makecell{IoU}}&
\includegraphics[height=\utilheightc]{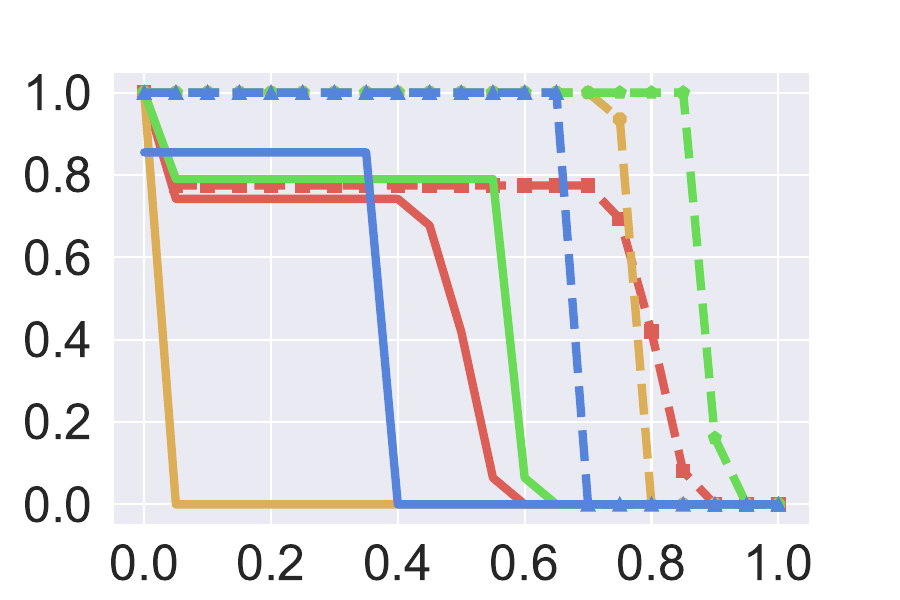}&
\includegraphics[height=\utilheightc]{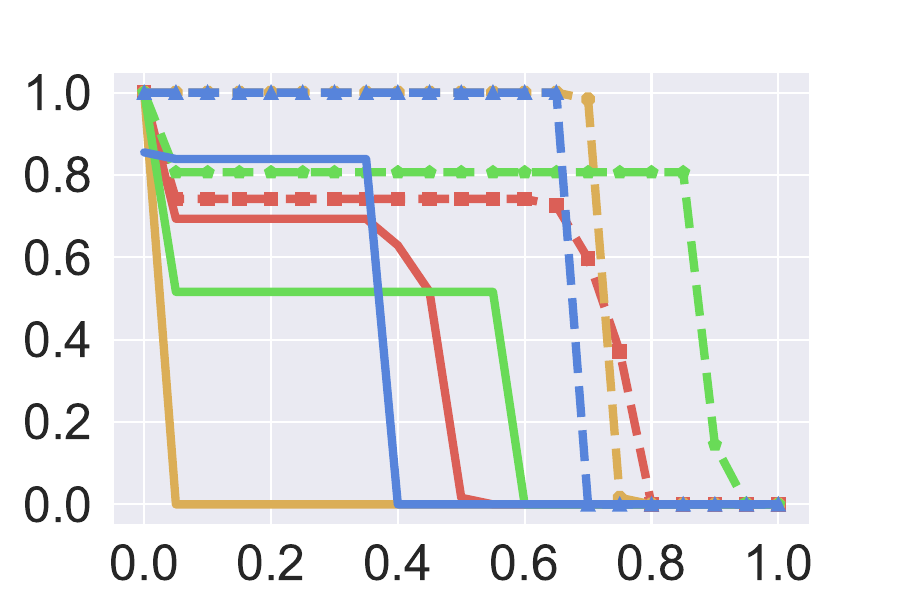}&
\includegraphics[height=\utilheightc]{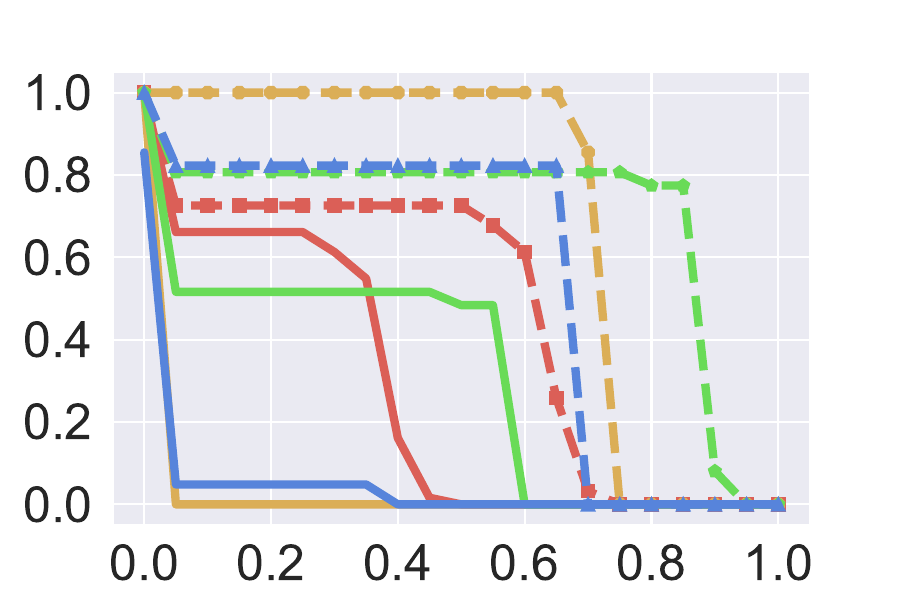}&
\includegraphics[height=\utilheightc]{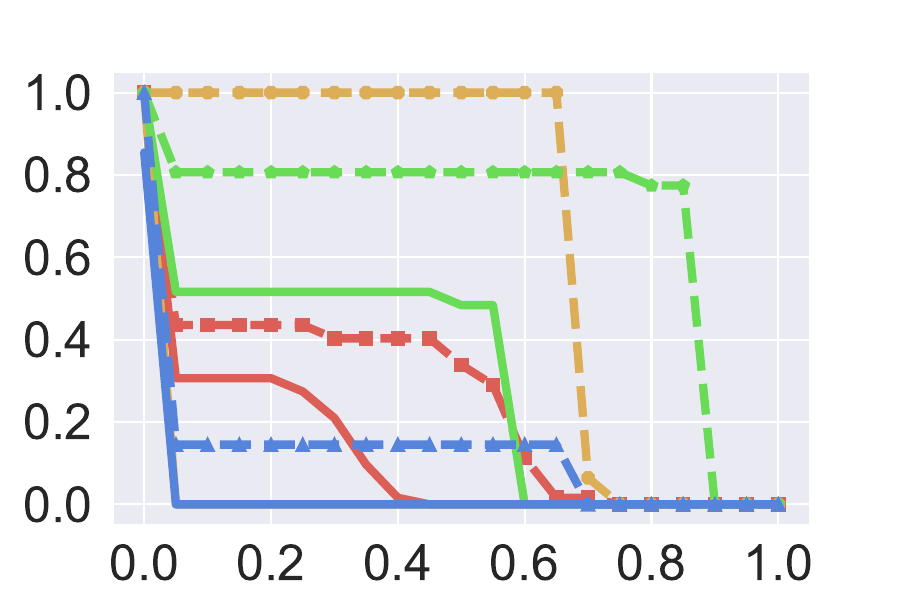}&
\includegraphics[height=\utilheightc]{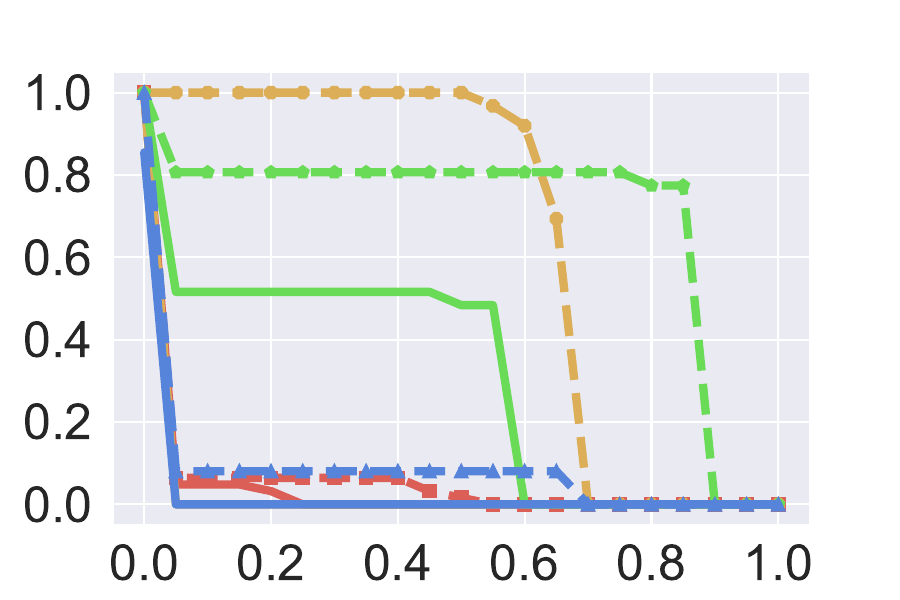}\\[-1.2ex]
        & \makecell{\tiny{$\text{TH}_\text{IoU}$}}
        & \makecell{\tiny{$\text{TH}_\text{IoU}$}}
        & \makecell{\tiny{$\text{TH}_\text{IoU}$}}
        & \makecell{\tiny{$\text{TH}_\text{IoU}$}}
        & \makecell{\tiny{$\text{TH}_\text{IoU}$}}
\end{tabular}
}
% \vspace{-3mm}
% \caption{\small Robustness certification for shifting transformation}\label{tab:bound-shift-th}
\end{subtable}

}

\caption{\small Robustness certification for shifting transformation, including detection rate bound and IoU bound. Solid lines represent the certified bounds, and dashed lines show the empirical performance under PGD.
% \caption{\small Robustness certification for cumulative reward, including \textit{expectation bound} $\uje$, \textit{percentile bound} $\ujp$ ($p=50\%$), and \textit{absolute lower bound} $\uj$. Each column corresponds to one smoothing variance. Solid lines represent the certified reward bounds, and dashed lines show the empirical performance under PGD.
% shows the change of the lower bounds w.r.t. the attack magnitude $\eps$.
% The solid curves represent the certified lower bounds, while the dotted lines represent the empirical values.
}%
\label{fig:shift-bounds-th}
% \vspace{-3em}
\end{figure*}
}

\subsection{Effect of Sample Strategy}
\label{subsubsec: sample-strategy}
To study the effect of sampling strategies, we compare two different sampling strategies and show the results in \Cref{tab:rotation sample strategy}. The first strategy is fixing the size of intervals (e.g. $0.1^\circ$ rotation intervals in out case), which is named ``Certification (sparse)'' in \Cref{tab:rotation sample strategy}, and the second strategy is fixing small interval number (e.g. 600 small intervals in each big rotation interval), which is called ``Certification (dense)" in \Cref{tab:rotation sample strategy}.

From certified detection rate \Cref{tab:rotation sample strategy (a)} and certified IoU \Cref{tab:rotation sample strategy (b)}, we can find that the ``Certification (sparse)" is already tight enough when the sample number in each small interval stays the same since the certified detection rate and IoU in ``Certification (sparse)" is almost the same as those in ``Certification (dense)". 

\begin{table*}[h]
    % \tiny
    \fontsize{5.5}{5.5}\selectfont
    \centering
    \caption{Overview of rotation transition experiment results with different sampling strategies under the condition ``color 1 with buildings without pedestrian''. Each row represents the corresponding model and attack radius. Columns ``Certification (sparse)'', and ``Certification (dense)'' represent the certified lower bound of performance under attacks with intervals of $0.1^\circ$ and 600 intervals respectively.}
    \begin{subtable}[]{\linewidth}\centering
    \caption{Certified rotation detection rate.}
    \begin{tabular}{c|c|ccc|ccc}
        \hline
        \multirow{2}{*}{Model} & \multirow{2}{*}{Attack Radius} & \multicolumn{3}{c|}{Certification (sparse)} & \multicolumn{3}{c}{Certification (dense)}\\
        & & \textbf{Det@20} & \textbf{Det@50} & \textbf{Det@80} & \textbf{Det@20} & \textbf{Det@50} & \textbf{Det@80} \\\hline\hline
        \multirowcell{3}{MonoCon\\\cite{liu2022learning}} & $r\pm10^\circ$ & 93.33\% & 93.33\% & 86.67\% & 93.33\% & 93.33\% & 86.67\% \\
                                                        & $r\pm20^\circ$ & 93.33\% & 93.33\% & 86.67\% & 93.33\% & 93.33\% & 86.67\% \\
                                                        & $r\pm30^\circ$ & 86.67\% & 86.67\% & 86.67\% & 86.67\% & 86.67\% & 86.67\% \\\hline
        \multirowcell{3}{SECOND\\\cite{yan2018second}} & $r\pm10^\circ$ & 100.00\% & 100.00\% & 0.00\% & 100.00\% & 100.00\% & 0.00\% \\
                                                     & $r\pm20^\circ$ & 100.00\% & 100.00\% & 0.00\% & 100.00\% & 100.00\% & 0.00\% \\
                                                     & $r\pm30^\circ$ & 100.00\% & 66.67\% & 0.00\% & 100.00\% & 66.67\% & 0.00\% \\\hline
        \multirowcell{3}{CLOCs\\\cite{pang2020clocs}} & $r\pm10^\circ$ & 100.00\% & 100.00\% & 100.00\% & 100.00\% & 100.00\% & 100.00\% \\
                                                    & $r\pm20^\circ$ & 100.00\% & 80.00\% & 73.33\% & 100.00\% & 80.00\% & 73.33\% \\
                                                    & $r\pm30^\circ$ & 100.00\% & 80.00\% & 73.33\% & 100.00\% & 80.00\% & 73.33\% \\\hline
        \multirowcell{3}{FocalsConv\\\cite{chen2022focal}} & $r\pm10^\circ$ & 100.00\% & 100.00\%  & 100.00\% & 100.00\% & 100.00\%  & 100.00\% \\
                                                         & $r\pm20^\circ$ & 100.00\% & 100.00\%  & 100.00\% & 100.00\% & 100.00\%  & 100.00\% \\
                                                         & $r\pm30^\circ$ & 100.00\% & 100.00\% & 100.00\% & 100.00\% & 100.00\% & 100.00\% \\\hline
    \end{tabular}
    \label{tab:rotation sample strategy (a)}
    \end{subtable}
    \begin{subtable}[]{\linewidth}\centering
    \caption{Certified rotation IoU.}
    \begin{tabular}{c|c|ccc|ccc}
        \hline
        \multirow{2}{*}{Model} & \multirow{2}{*}{Attack Radius} & \multicolumn{3}{c}{Certification (sparse)} & \multicolumn{3}{c}{Certification (dense)}\\
        & & \textbf{AP@30} & \textbf{AP@50} & \textbf{AP@70 } & \textbf{AP@30} & \textbf{AP@50} & \textbf{AP@70 } \\\hline\hline
        \multirowcell{3}{MonoCon\\\cite{liu2022learning}} & $r\pm10^\circ$ & 86.67\% & 0.00\% & 0.00\% & 86.67\% & 0.00\% & 0.00\% \\
                                                        & $r\pm20^\circ$ & 13.33\% & 0.00\% & 0.00\% & 13.33\% & 0.00\% & 0.00\% \\
                                                        & $r\pm30^\circ$ & 0.00\% & 0.00\% & 0.00\% & 0.00\% & 0.00\% & 0.00\% \\\hline
        \multirowcell{3}{SECOND\\\cite{yan2018second}} & $r\pm10^\circ$ & 0.00\% & 0.00\% & 0.00\% & 0.00\% & 0.00\% & 0.00\% \\
                                                     & $r\pm20^\circ$ & 0.00\% & 0.00\% & 0.00\% & 0.00\% & 0.00\% & 0.00\% \\
                                                     & $r\pm30^\circ$ & 0.00\% & 0.00\% & 0.00\% & 0.00\% & 0.00\% & 0.00\% \\\hline
        \multirowcell{3}{CLOCs\\\cite{pang2020clocs}} & $r\pm10^\circ$ & 100.00\% & 100.00\% & 0.00\% & 100.00\% & 100.00\% & 0.00\% \\
                                                    & $r\pm20^\circ$ & 100.00\% & 93.33\% & 0.00\% & 100.00\% & 93.33\% & 0.00\% \\
                                                    & $r\pm30^\circ$ & 100.00\% & 53.33\% & 0.00\% & 100.00\% & 53.33\% & 0.00\% \\\hline
        \multirowcell{3}{FocalsConv\\\cite{chen2022focal}} & $r\pm10^\circ$ & 100.00\% & 0.00\% & 0.00\% & 100.00\% & 0.00\% & 0.00\% \\
                                                         & $r\pm20^\circ$ & 0.00\% & 0.00\% & 0.00\% & 13.33\% & 0.00\% & 0.00\% \\
                                                         & $r\pm30^\circ$ & 0.00\% & 0.00\% & 0.00\% & 0.00\% & 0.00\% & 0.00\% \\\hline
    \end{tabular}
    \label{tab:rotation sample strategy (b)}
    \end{subtable}
    \label{tab:rotation sample strategy}
% \vspace{-15pt}
\end{table*}

\subsection{Effect of Smoothing Parameter $\sigma$}
\label{subsubsec: smoothing-parameter}
To study the effect of smoothing parameter $\sigma$, we test our approach with random Gaussian noise whose $\sigma=0.5$ in our rotation setting to compare the previous results on rotation transformation with noises whose $\sigma=0.25$ (\Cref{appendix-tab:rotation overview 1}). 

As shown in \Cref{tab:bound-rotation-0.5}, \Cref{fig:rotation-bounds-th-0.5} and \Cref{appendix-tab:rotation overview 2}, the models' performance might degrade in small attack radius with larger smoothing $\sigma$, but smoothing with larger $\sigma$ can also improve the robustness of models against large attack radius, especially for the model which is more stable than others (CLOCs in our case).

{
\renewcommand{\thesubfigure}{\alph{subfigure}}
% {\refstepcounter{subfigure}\textbf{(\thesubfigure) }{\ignorespaces #1}}

\begin{figure*}[!t]
% \vspace{-2em}

% \newlength{\utilheightc}
\settoheight{\utilheightc}{\includegraphics[width=.160\linewidth]{rotation_conf_0.2.pdf}}%

% \newlength{\utilheightd}
\settoheight{\utilheightd}{\includegraphics[width=.165\linewidth]{rotation_IoU_0.3.pdf}}%

% \newlength{\utilheightaa}
% \settoheight{\utilheightaa}{\includegraphics[width=.162\linewidth]{Freeway_adasearch_sigma-0.05.pdf}}%

% \newlength{\utilheightcp}
% \settoheight{\utilheightcp}{\includegraphics[width=.160\linewidth]{Pong_global_mean_sigma-0.001.pdf}}%

% \newlength{\utilheightdp}
% \settoheight{\utilheightdp}{\includegraphics[width=.165\linewidth]{Pong_global_median_sigma-0.001.pdf}}%

% \newlength{\utilheightaap}
% \settoheight{\utilheightaap}{\includegraphics[width=.162\linewidth]{Pong_adasearch_sigma-0.05.pdf}}%

% \newlength{\utilheight}
% \settoheight{\utilheight}{\includegraphics[width=.138\linewidth]{twitch-DE: low degree.pdf}}%

% \newlength{\attackheightb}
% \settoheight{\attackheightb}{\includegraphics[width=.138\linewidth]{twitch-DE: high degree.pdf}}%

% \newlength{\legendheightb}
\setlength{\legendheightb}{0.48\utilheightc}%

\newcommand{\rownamec}[1]% #1 = text
{\rotatebox{90}{\makebox[\utilheightc][c]{\tiny #1}}}

\newcommand{\rownamed}[1]% #1 = text
{\rotatebox{90}{\makebox[\utilheightd][c]{\tiny #1}}}

\centering

{
\renewcommand{\tabcolsep}{10pt}

\begin{subtable}[]{\linewidth}
\begin{tabular}{l}
\includegraphics[height=\legendheightb]{legend_marker.pdf}
\end{tabular}
\end{subtable}
% \vspace{-2pt}
\begin{subtable}{\linewidth}
\centering
\resizebox{\linewidth}{!}{%
\begin{tabular}{@{}p{3mm}@{}c@{}c@{}c@{}c@{}c@{}}
        & \makecell{\tiny{$|r|\leq 10^\circ$}}
        & \makecell{\tiny{$|r|\leq 15^\circ$}}
        & \makecell{\tiny{$|r|\leq 20^\circ$}}
        & \makecell{\tiny{$|r|\leq 25^\circ$}}
        & \makecell{\tiny{$|r|\leq 30^\circ$}}
        % & \makecell{\tiny{$\text{TH}_\text{IoU}=0.8$}}
        \vspace{-2pt}\\
\rownamec{\makecell{Detection}}&
\includegraphics[height=\utilheightc]{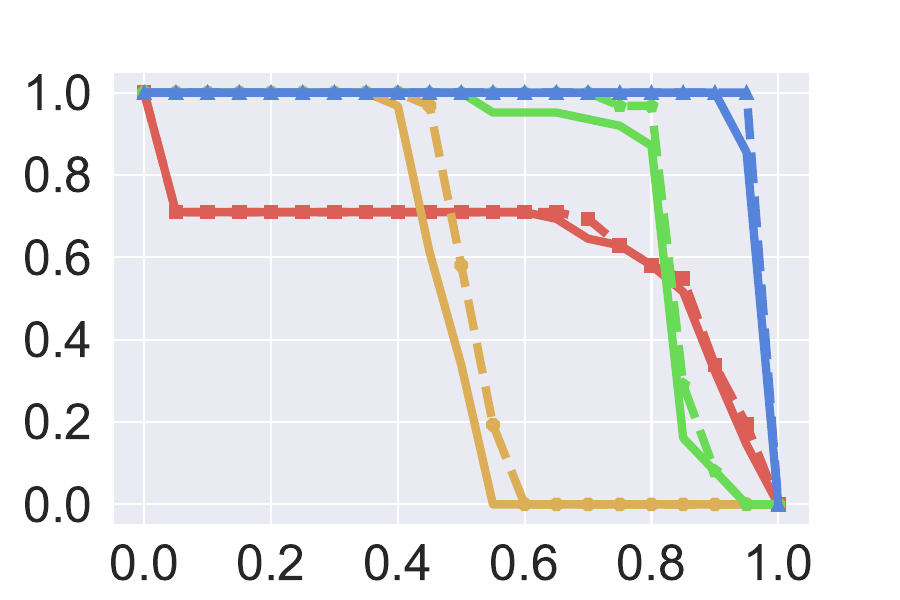}&
\includegraphics[height=\utilheightc]{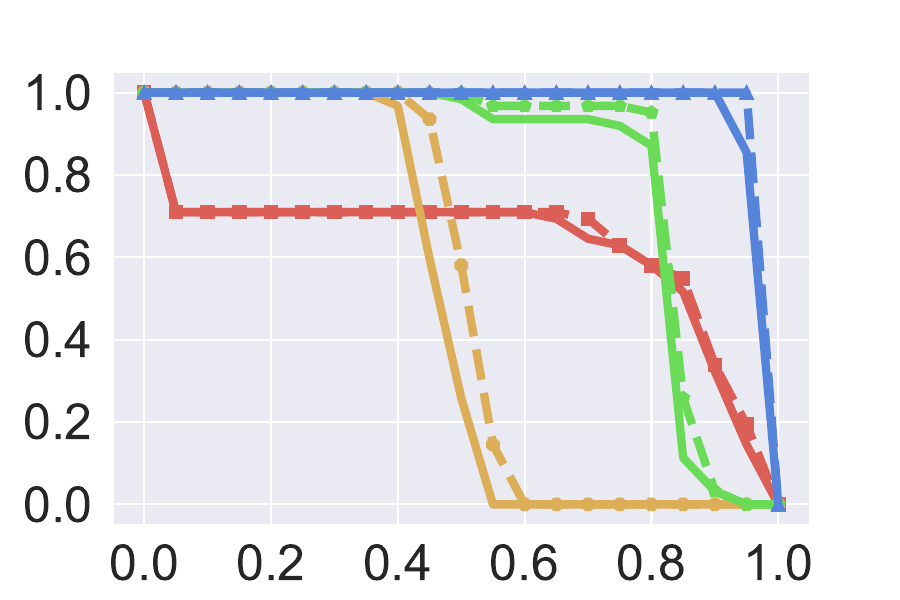}&
\includegraphics[height=\utilheightc]{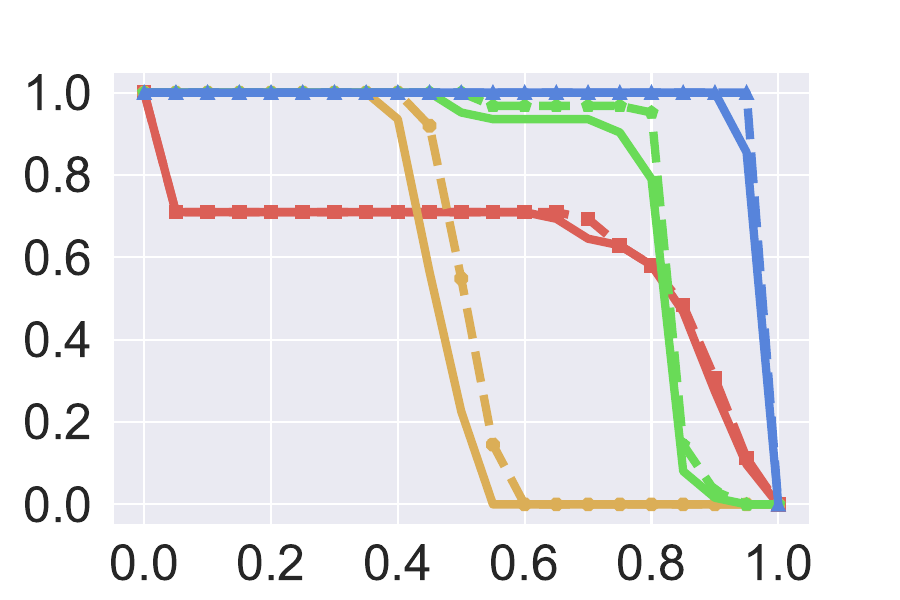}&
\includegraphics[height=\utilheightc]{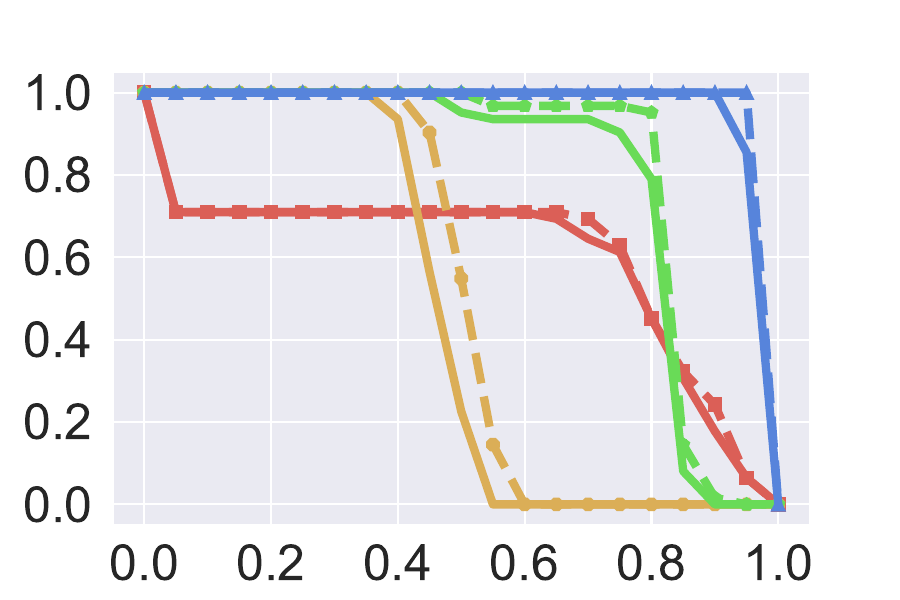}&
\includegraphics[height=\utilheightc]{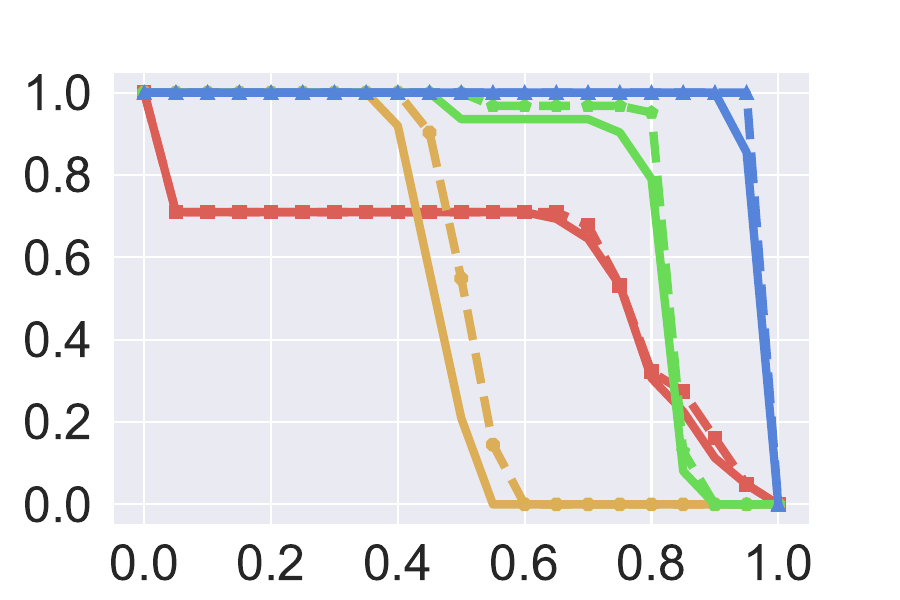}\\[-1.2ex]
        & \makecell{\tiny{$\text{TH}_\text{conf}$}}
        & \makecell{\tiny{$\text{TH}_\text{conf}$}}
        & \makecell{\tiny{$\text{TH}_\text{conf}$}}
        & \makecell{\tiny{$\text{TH}_\text{conf}$}}
        & \makecell{\tiny{$\text{TH}_\text{conf}$}}\\
\rownamec{\makecell{IoU}}&
\includegraphics[height=\utilheightc]{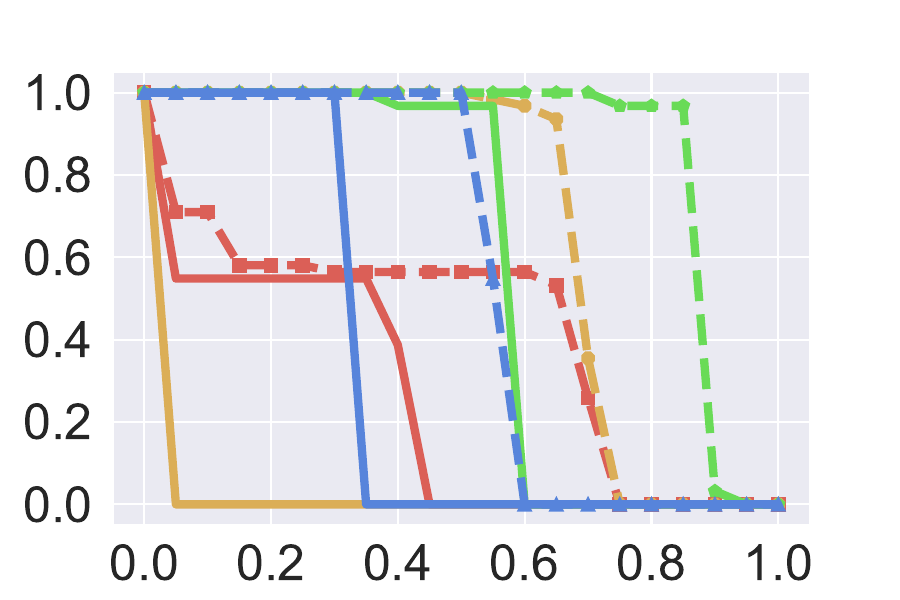}&
\includegraphics[height=\utilheightc]{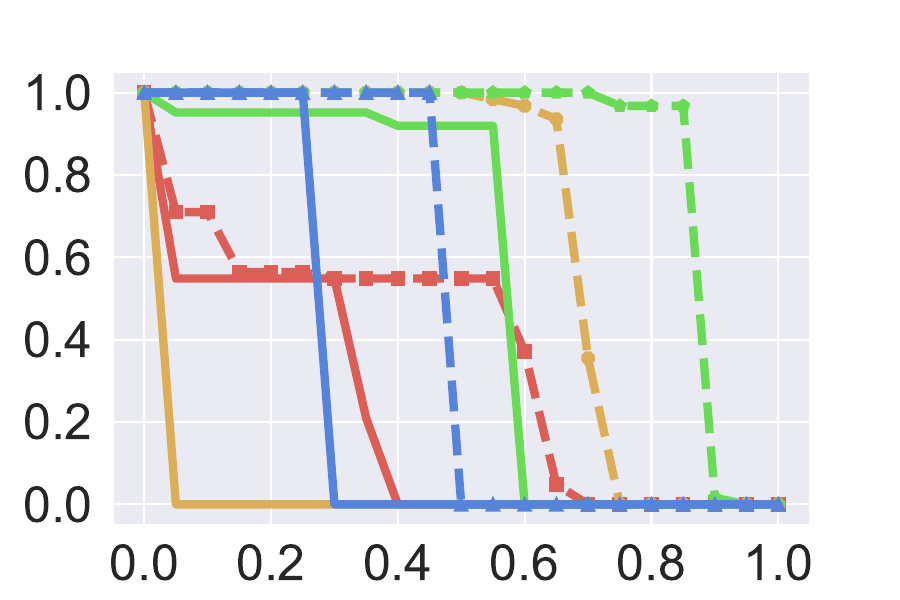}&
\includegraphics[height=\utilheightc]{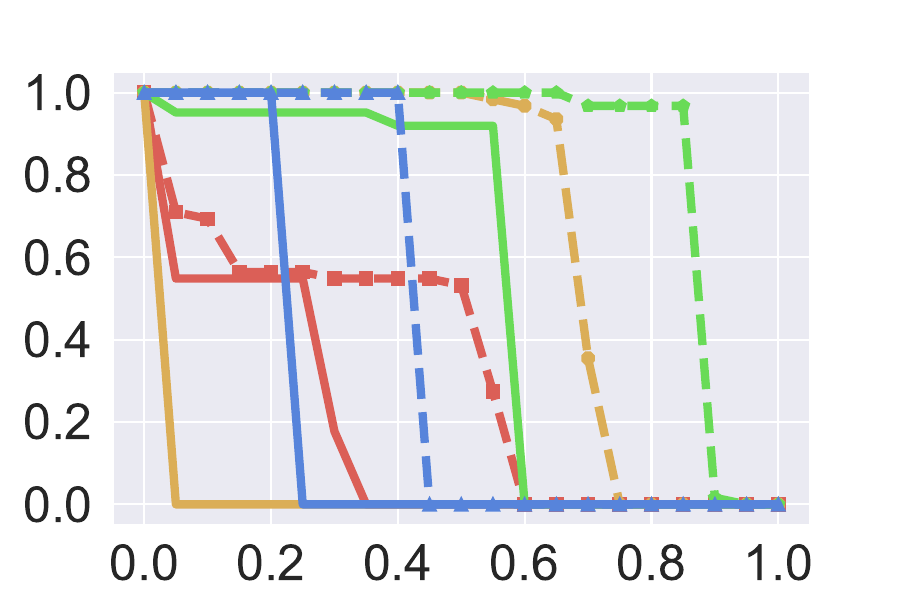}&
\includegraphics[height=\utilheightc]{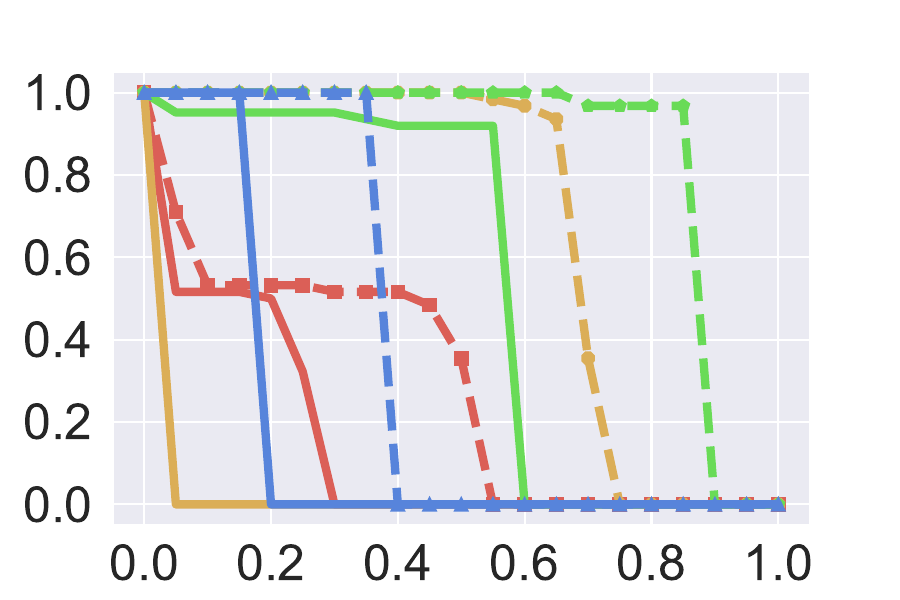}&
\includegraphics[height=\utilheightc]{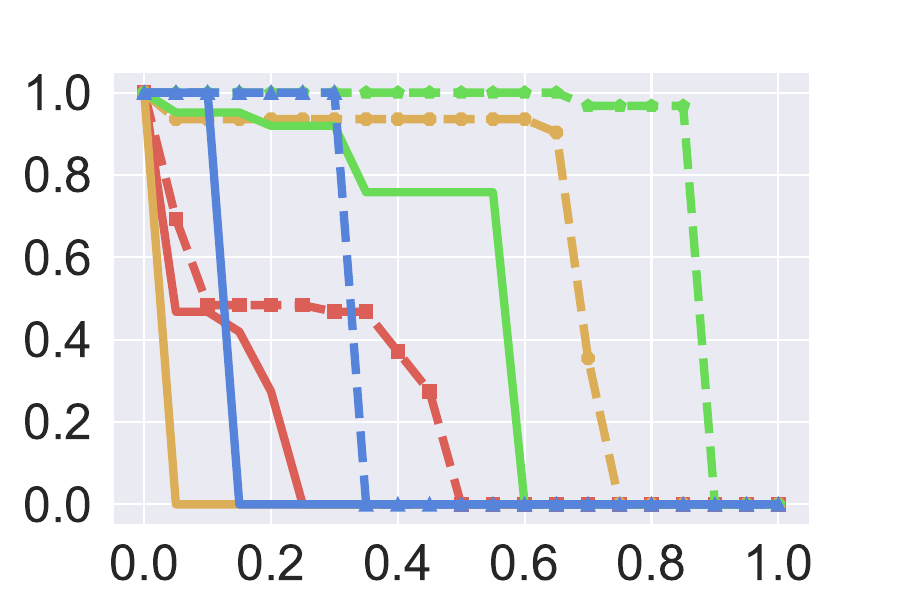}\\[-1.2ex]
        & \makecell{\tiny{$\text{TH}_\text{IoU}$}}
        & \makecell{\tiny{$\text{TH}_\text{IoU}$}}
        & \makecell{\tiny{$\text{TH}_\text{IoU}$}}
        & \makecell{\tiny{$\text{TH}_\text{IoU}$}}
        & \makecell{\tiny{$\text{TH}_\text{IoU}$}}
\end{tabular}
}
% \vspace{-3mm}
% \caption{\small Robustness certification for rotation transformation}
% \label{tab:bound-rotation-th}
\end{subtable}

}

\caption{Robustness certification for rotation transformation (smoothing $\sigma=0.5$), including detection rate bound and IoU bound. Solid lines represent the certified bounds, and dashed lines show the empirical performance under PGD. $x$-axis is the threshold for confidence score ($\text{TH}_\text{conf}$) and IoU score ($\text{TH}_\text{IoU}$), and $y$-axis is the ratio of detection whose confidence / IoU score is larger than the confidence / IoU threshold. 
% \linyi{What does $x$-axis mean in each sub-figure? What does $y$-axis mean in each sub-figure?}
% \caption{\small Robustness certification for cumulative reward, including \textit{expectation bound} $\uje$, \textit{percentile bound} $\ujp$ ($p=50\%$), and \textit{absolute lower bound} $\uj$. Each column corresponds to one smoothing variance. Solid lines represent the certified reward bounds, and dashed lines show the empirical performance under PGD.
% shows the change of the lower bounds w.r.t. the attack magnitude $\eps$.
% The solid curves represent the certified lower bounds, while the dotted lines represent the empirical values.
}%
\label{fig:rotation-bounds-th-0.5}
% \vspace{-8pt}
\end{figure*}
}

\begin{table*}[ht]
    % \tiny
    \fontsize{5.5}{5.5}\selectfont
    \centering
    \caption{Overview of rotation transformation experiment results (smoothing $\sigma=0.5$). Each row represents the corresponding model and attack radius. ``Benign'', ``Adv (Vanilla)'', ``Adv (Smoothed)'', and ``Certification'' stands for benign performance, vanilla models' performance under attacks, smoothed models' performance under attacks, and certified lower bound of smoothed model performance under attacks. Each column represents the results under different thresholds.}
    \begin{subtable}[]{\linewidth}\centering
    \caption{Detection rate under rotation transformation}
    \resizebox{\linewidth}{!}{
    \begin{tabular}{c|c|ccc|ccc|ccc|ccc}
        \hline
        \multirow{2}{*}{Model} & \multirow{2}{*}{Attack Radius} & \multicolumn{3}{c|}{Benign} & \multicolumn{3}{c|}{Adv (Vanilla)} & \multicolumn{3}{c|}{Adv (Smoothed)} & \multicolumn{3}{c}{Certification}\\
        & & \textbf{Det@20} & \textbf{Det@50} & \textbf{Det@80} & \textbf{Det@20} & \textbf{Det@50} & \textbf{Det@80} & \textbf{Det@20} & \textbf{Det@50} & \textbf{Det@80} & \textbf{Det@20} & \textbf{Det@50} & \textbf{Det@80} \\\hline\hline
        \multirowcell{5}{MonoCon\\\cite{liu2022learning}} & $|r|\leq10^\circ$ & \multirow{5}{*}{100.00\%} & \multirow{5}{*}{100.00\%} & \multirow{5}{*}{100.00\%} & 70.97\% & 70.97\% & 58.06\% & 70.97\% & 70.97\% & 58.06\% & 70.97\% & 70.97\% & 58.06\% \\
                                                        & $|r|\leq15^\circ$ & & & & 70.97\% & 70.97\% & 58.06\% & 70.97\% & 70.97\% & 58.06\% & 70.97\% & 70.97\% & 58.06\% \\
                                                        & $|r|\leq20^\circ$ & & & & 70.97\% & 70.97\% & 58.06\% & 70.97\% & 70.97\% & 58.06\% & 70.97\% & 70.97\% & 58.06\% \\
                                                        & $|r|\leq25^\circ$ & & & & 70.97\% & 70.97\% & 45.16\% & 70.97\% & 70.97\% & 45.16\% & 70.97\% & 70.97\% & 45.16\% \\
                                                        & $|r|\leq30^\circ$ & & & & 70.97\% & 70.97\% & 32.26\% & 70.97\% & 70.97\% & 32.26\% & 70.97\% & 70.97\% & 30.65\% \\\hline
        \multirowcell{5}{SECOND\\\cite{yan2018second}} & $|r|\leq10^\circ$ & \multirow{5}{*}{100.00\%} & \multirow{5}{*}{100.00\%} & \multirow{5}{*}{100.00\%} & 100.00\% & 100.00\% & 0.00\% & 100.00\% & 58.06\% & 0.00\% & 100.00\% & 33.87\% & 0.00\% \\
                                                     & $|r|\leq15^\circ$ & & & & 100.00\% & 100.00\% & 0.00\% & 100.00\% & 58.06\% & 0.00\% & 100.00\% & 25.81\% & 0.00\% \\
                                                     & $|r|\leq20^\circ$ & & & & 100.00\% & 100.00\% & 0.00\% & 100.00\% & 54.84\% & 0.00\% & 100.00\% & 22.58\% & 0.00\% \\
                                                     & $|r|\leq25^\circ$ & & & & 100.00\% & 12.90\% & 0.00\% & 100.00\% & 54.84\% & 0.00\% & 100.00\% & 22.58\% & 0.00\% \\
                                                     & $|r|\leq30^\circ$ & & & & 67.74\% & 3.23\% & 0.00\% & 100.00\% & 54.84\% & 0.00\% & 100.00\% & 20.97\% & 0.00\% \\\hline
        \multirowcell{5}{CLOCs\\\cite{pang2020clocs}} & $|r|\leq10^\circ$ & \multirow{5}{*}{100.00\%} & \multirow{5}{*}{100.00\%} & \multirow{5}{*}{100.00\%} & 100.00\% & 100.00\% & 100.00\% & 100.00\% & 100.00\% & 96.77\% & 100.00\% & 100.00\% & 87.10\% \\
                                                    & $|r|\leq15^\circ$ & & & & 100.00\% & 100.00\% & 100.00\% & 100.00\% & 100.00\% & 95.16\% & 100.00\% & 98.39\% & 87.10\% \\
                                                    & $|r|\leq20^\circ$ & & & & 100.00\% & 100.00\% & 100.00\% & 100.00\% & 100.00\% & 95.16\% & 100.00\% & 95.16\% & 79.03\% \\
                                                    & $|r|\leq25^\circ$ & & & & 100.00\% & 91.94\% & 20.97\% & 100.00\% & 100.00\% & 95.16\% & 100.00\% & 95.16\% & 79.03\% \\
                                                    & $|r|\leq30^\circ$ & & & & 100.00\% & 74.19\% & 3.23\% & 100.00\% & 100.00\% & 95.16\% & 100.00\% & 93.55\% & 79.03\% \\\hline
        \multirowcell{5}{FocalsConv\\\cite{chen2022focal}} & $|r|\leq10^\circ$ & \multirow{5}{*}{100.00\%} & \multirow{5}{*}{100.00\%} & \multirow{5}{*}{100.00\%} & 100.00\% & 100.00\% & 100.00\% & 100.00\% & 100.00\% & 100.00\% & 100.00\% & 100.00\%  & 100.00\% \\
                                                         & $|r|\leq15^\circ$ & & & & 100.00\% & 100.00\% & 100.00\% & 100.00\% & 100.00\% & 100.00\% & 100.00\% & 100.00\%  & 100.00\% \\
                                                         & $|r|\leq20^\circ$ & & & & 100.00\% & 100.00\% & 100.00\% & 100.00\% & 100.00\% & 100.00\% & 100.00\% & 100.00\%  & 100.00\% \\
                                                         & $|r|\leq25^\circ$ & & & & 100.00\% & 100.00\% & 100.00\% & 100.00\% & 100.00\% & 100.00\% & 100.00\% & 100.00\%  & 100.00\% \\
                                                         & $|r|\leq30^\circ$ & & & & 100.00\% & 100.00\% & 98.39\% & 100.00\% & 100.00\% & 100.00\% & 100.00\% & 100.00\% & 100.00\% \\\hline
        \end{tabular}
        }
    \label{appendix-tab:rotation overview 2(a)}
    \end{subtable}
    \begin{subtable}[]{\linewidth}\centering
    \caption{IoU with ground truth under rotation transformation}
    \resizebox{\linewidth}{!}{
    \begin{tabular}{c|c|ccc|ccc|ccc|ccc}
        \hline
        \multirow{2}{*}{Model} & \multirow{2}{*}{Attack Radius} & \multicolumn{3}{c|}{Benign} & \multicolumn{3}{c|}{Adv (Vanilla)} & \multicolumn{3}{c|}{Adv (Smoothed)} & \multicolumn{3}{c}{Certification}\\
        & & \textbf{AP@30} & \textbf{AP@50} & \textbf{AP@80} & \textbf{AP@30} & \textbf{AP@50} & \textbf{AP@80} & \textbf{AP@30} & \textbf{AP@50} & \textbf{AP@80} & \textbf{AP@30} & \textbf{AP@50} & \textbf{AP@80} \\\hline\hline
        \multirowcell{5}{MonoCon\\\cite{liu2022learning}} & $|r|\leq10^\circ$ & \multirow{5}{*}{100.00\%} & \multirow{5}{*}{100.00\%} & \multirow{5}{*}{100.00\%} & 56.45\% & 56.45\% & 0.00\% & 56.45\% & 56.45\% & 0.00\% & 54.84\% & 0.00\% & 0.00\% \\
                                                        & $|r|\leq15^\circ$ & & & & 54.84\% & 54.84\% & 0.00\% & 54.84\% & 54.84\% & 0.00\% & 54.84\% & 0.00\% & 0.00\% \\
                                                        & $|r|\leq20^\circ$ & & & & 54.84\% & 53.23\% & 0.00\% & 54.84\% & 53.23\% & 0.00\% & 17.74\% & 0.00\% & 0.00\% \\
                                                        & $|r|\leq25^\circ$ & & & & 51.61\% & 35.48\% & 0.00\% & 51.61\% & 35.48\% & 0.00\% & 0.00\% & 0.00\% & 0.00\% \\
                                                        & $|r|\leq30^\circ$ & & & & 46.77\% & 0.00\% & 0.00\% & 46.77\% & 0.00\% & 0.00\% & 0.00\% & 0.00\% & 0.00\% \\\hline
        \multirowcell{5}{SECOND\\\cite{yan2018second}} & $|r|\leq10^\circ$ & \multirow{5}{*}{100.00\%} & \multirow{5}{*}{100.00\%} & \multirow{5}{*}{100.00\%} & 96.77\% & 96.77\% & 0.00\% & 100.00\% & 100.00\% & 0.00\% & 0.00\% & 0.00\% & 0.00\% \\
                                                     & $|r|\leq15^\circ$ & & & & 96.77\% & 96.77\% & 0.00\% & 100.00\% & 100.00\% & 0.00\% & 0.00\% & 0.00\% & 0.00\% \\
                                                     & $|r|\leq20^\circ$ & & & & 96.77\% & 96.77\% & 0.00\% & 100.00\% & 100.00\% & 0.00\% & 0.00\% & 0.00\% & 0.00\% \\
                                                     & $|r|\leq25^\circ$ & & & & 83.87\% & 83.87\% & 0.00\% & 100.00\% & 100.00\% & 0.00\% & 0.00\% & 0.00\% & 0.00\% \\
                                                     & $|r|\leq30^\circ$ & & & & 51.61\% & 51.61\% & 0.00\% & 93.55\% & 93.55\% & 0.00\% & 0.00\% & 0.00\% & 0.00\% \\\hline
        \multirowcell{5}{CLOCs\\\cite{pang2020clocs}} & $|r|\leq10^\circ$ & \multirow{5}{*}{100.00\%} & \multirow{5}{*}{100.00\%} & \multirow{5}{*}{100.00\%} & 90.32\% & 90.32\% & 90.32\% & 100.00\% & 100.00\% & 96.77\% & 100.00\% & 96.77\% & 0.00\%\\
                                                    & $|r|\leq15^\circ$ & & & & 90.32\% & 90.32\% & 90.32\% & 100.00\% & 100.00\% & 96.77\% & 95.16\% & 91.93\% & 0.00\%\\
                                                    & $|r|\leq20^\circ$ & & & & 88.71\% & 88.71\% & 77.42\% & 100.00\% & 100.00\% & 96.77\% & 95.16\% & 91.93\% & 0.00\%\\
                                                    & $|r|\leq25^\circ$ & & & & 87.10\% & 87.10\% & 0.00\% & 100.00\% & 100.00\% & 96.77\% & 95.16\% & 91.93\% & 0.00\%\\
                                                    & $|r|\leq30^\circ$ & & & & 88.71\% & 88.71\% & 3.26\% & 98.39\% & 98.39\% & 67.74\% & 98.39\% & 53.23\% & 0.00\%\\\hline
        \multirowcell{5}{FocalsConv\\\cite{chen2022focal}} & $|r|\leq10^\circ$ & \multirow{5}{*}{100.00\%} & \multirow{5}{*}{100.00\%} & \multirow{5}{*}{100.00\%} & 96.77\% & 96.77\% & 0.00\% & 100.00\% & 100.00\% & 0.00\% & 100.00\% & 0.00\% & 0.00\% \\
                                                         & $|r|\leq15^\circ$ & & & & 96.77\% & 0.00\% & 0.00\% & 100.00\% & 0.00\% & 0.00\% & 0.00\% & 0.00\% & 0.00\% \\
                                                         & $|r|\leq20^\circ$ & & & & 96.77\% & 0.00\% & 0.00\% & 100.00\% & 0.00\% & 0.00\% & 0.00\% & 0.00\% & 0.00\% \\
                                                         & $|r|\leq25^\circ$ & & & & 96.77\% & 0.00\% & 0.00\% & 100.00\% & 0.00\% & 0.00\% & 0.00\% & 0.00\% & 0.00\% \\
                                                         & $|r|\leq30^\circ$ & & & & 96.77\% & 0.00\% & 0.00\% & 100.00\% & 0.00\% & 0.00\% & 0.00\% & 0.00\% & 0.00\% \\\hline
    \end{tabular}
    }
    \label{appendix-tab:rotation overview 2(b)}
    \end{subtable}
    \label{appendix-tab:rotation overview 2}
\end{table*}

\subsection{Examples}
\label{subsubsec: examples}

\begin{figure}[ht]
    \centering
    \subcaptionbox{Rotation failure cases 1\label{fig:rotation failure 1}}{
        \includegraphics[width=.2\linewidth]{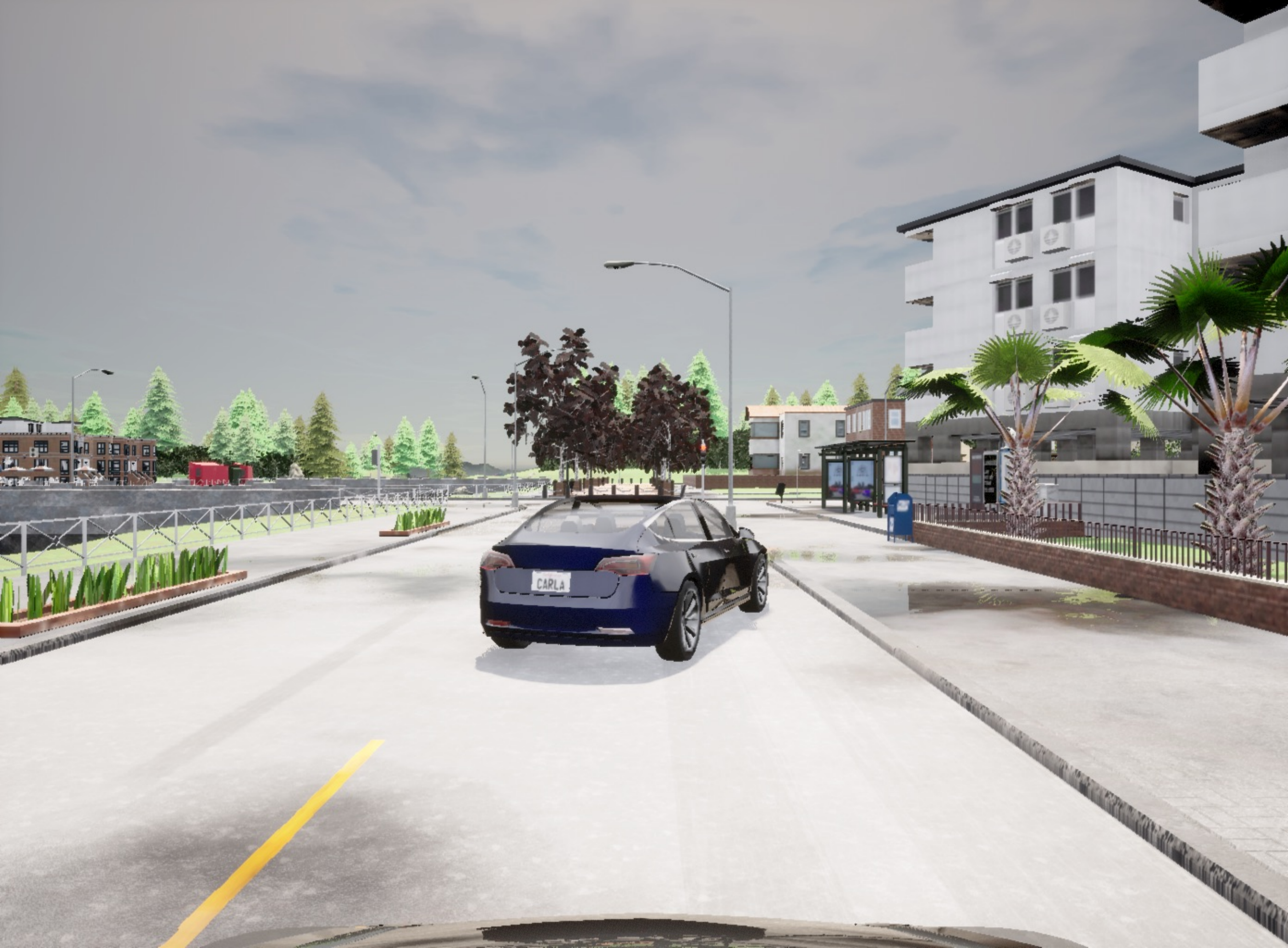}\quad
        \includegraphics[width=.2\linewidth]{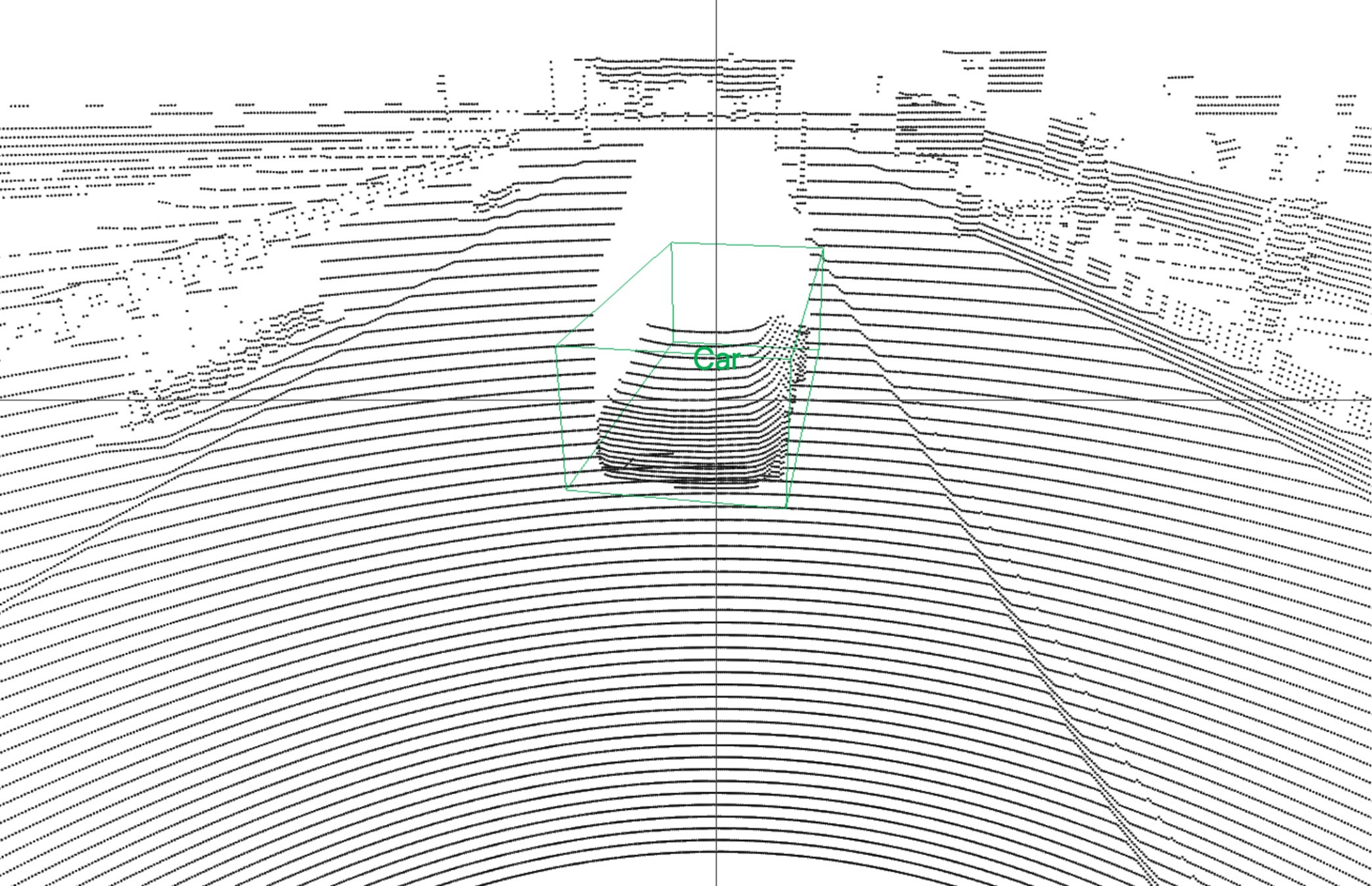}\quad
        \includegraphics[width=.2\linewidth]{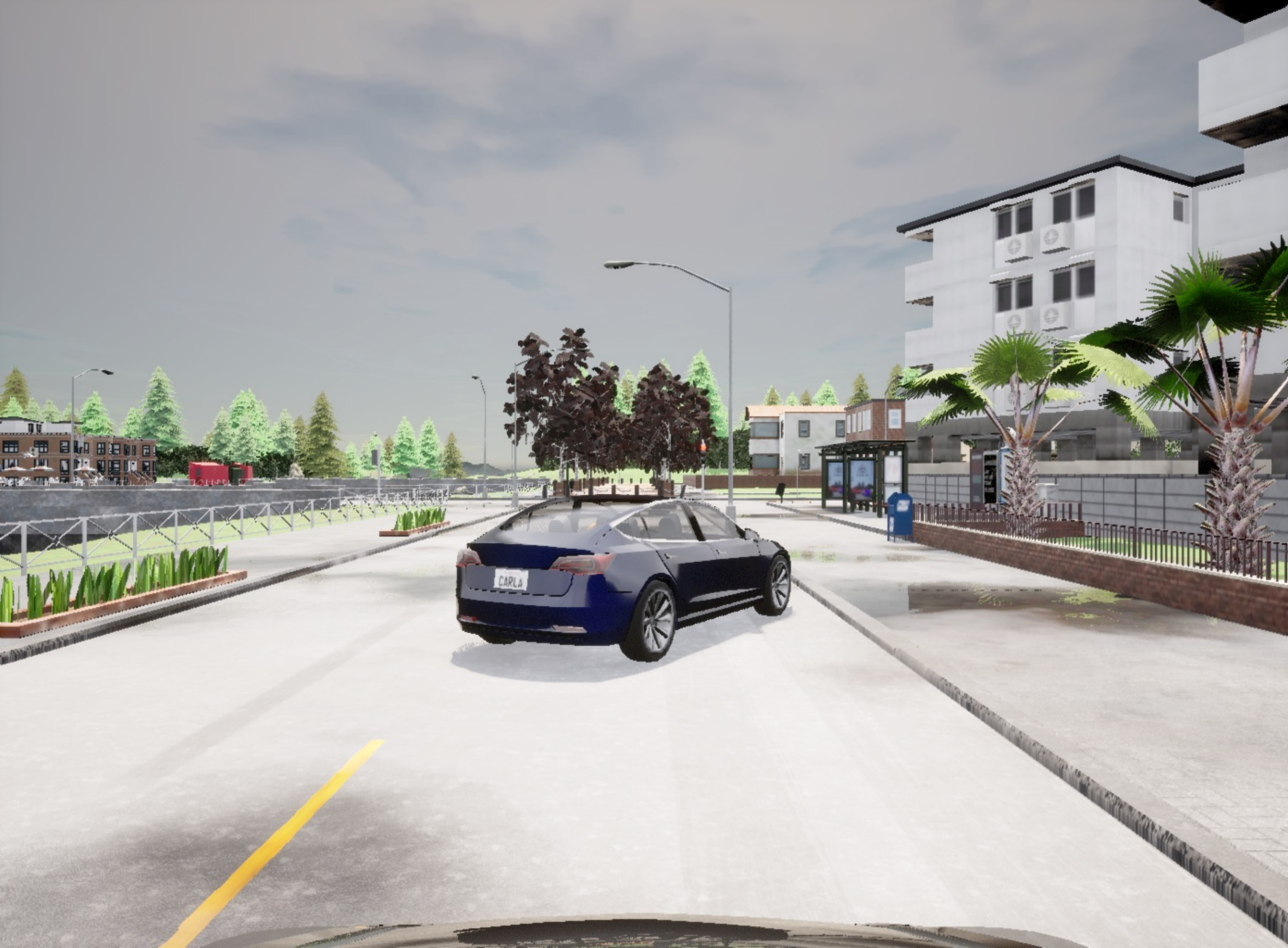}\quad
        \includegraphics[width=.2\linewidth]{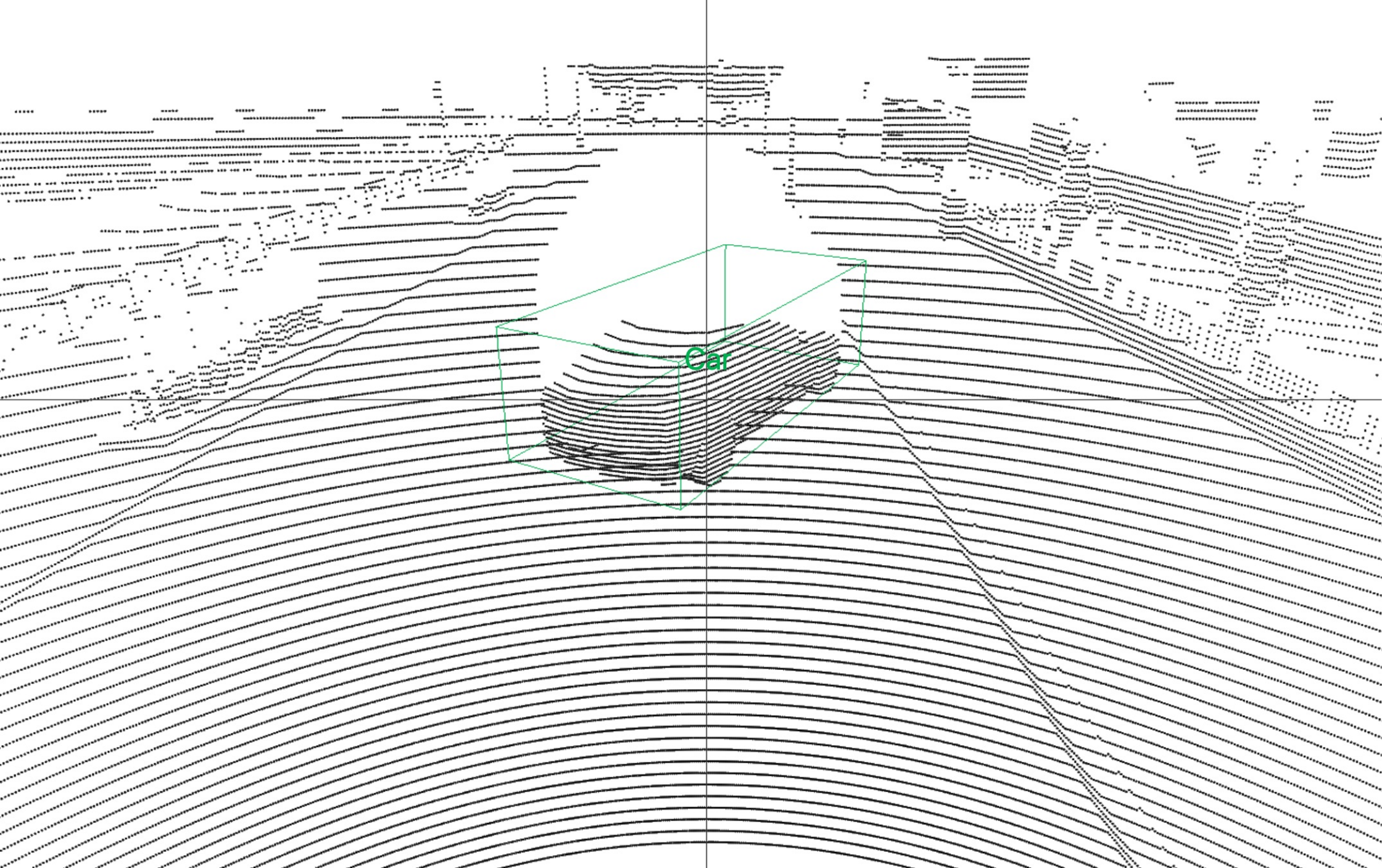}
    }
    \subcaptionbox{Rotation failure cases 2\label{fig:rotation failure 2}}{
        \includegraphics[width=.2\linewidth]{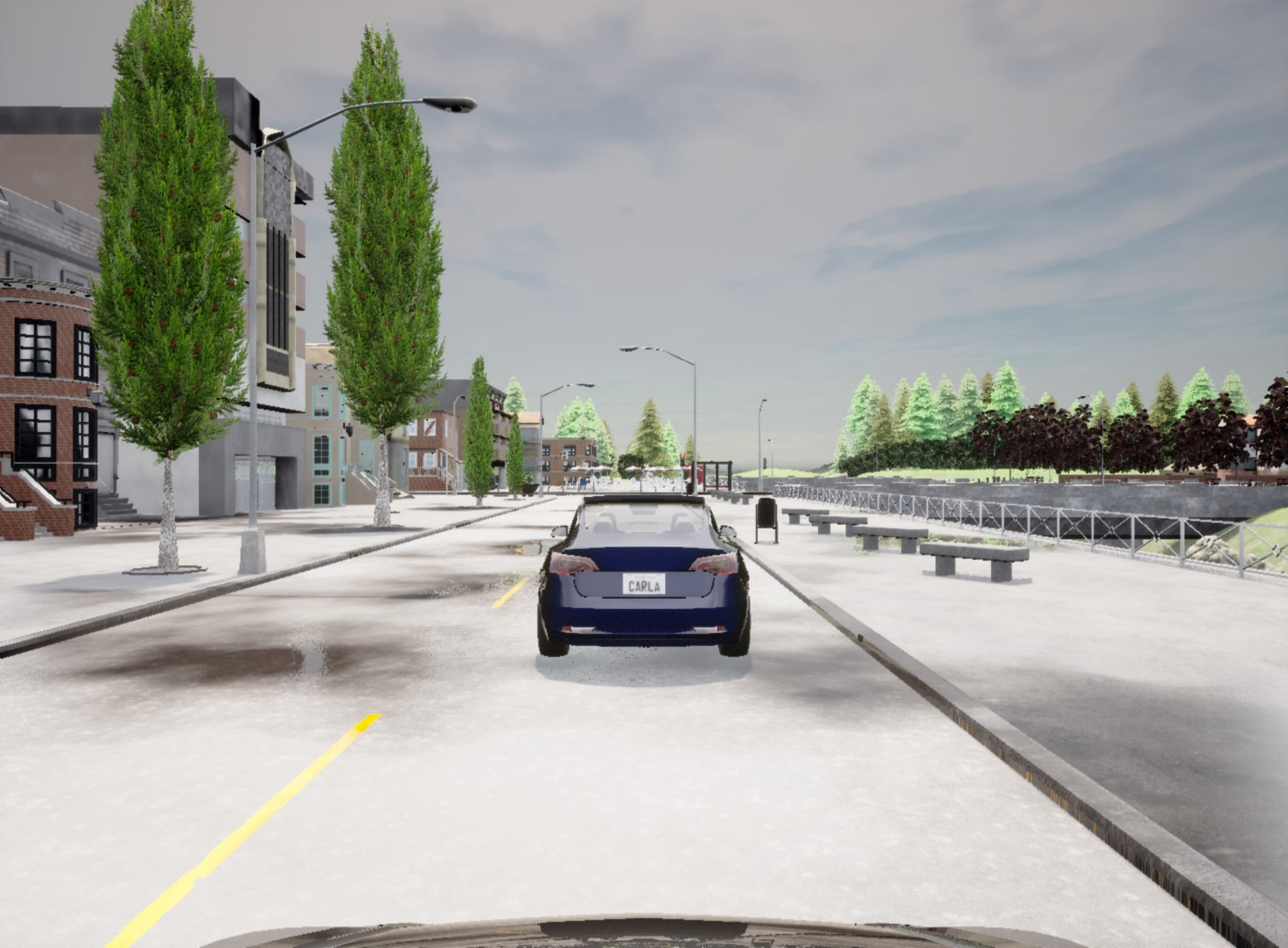}\quad
        \includegraphics[width=.2\linewidth]{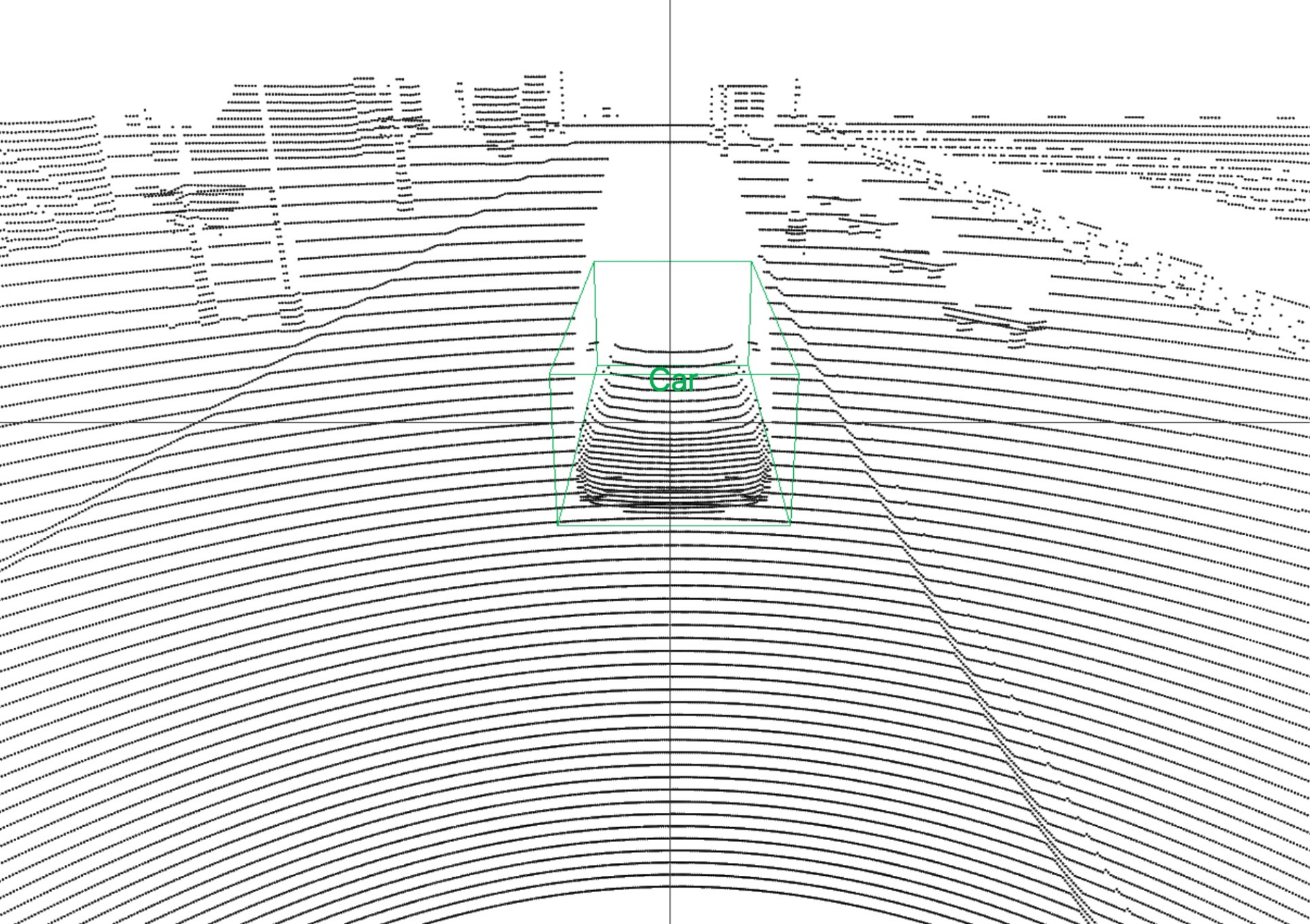}\quad
        \includegraphics[width=.2\linewidth]{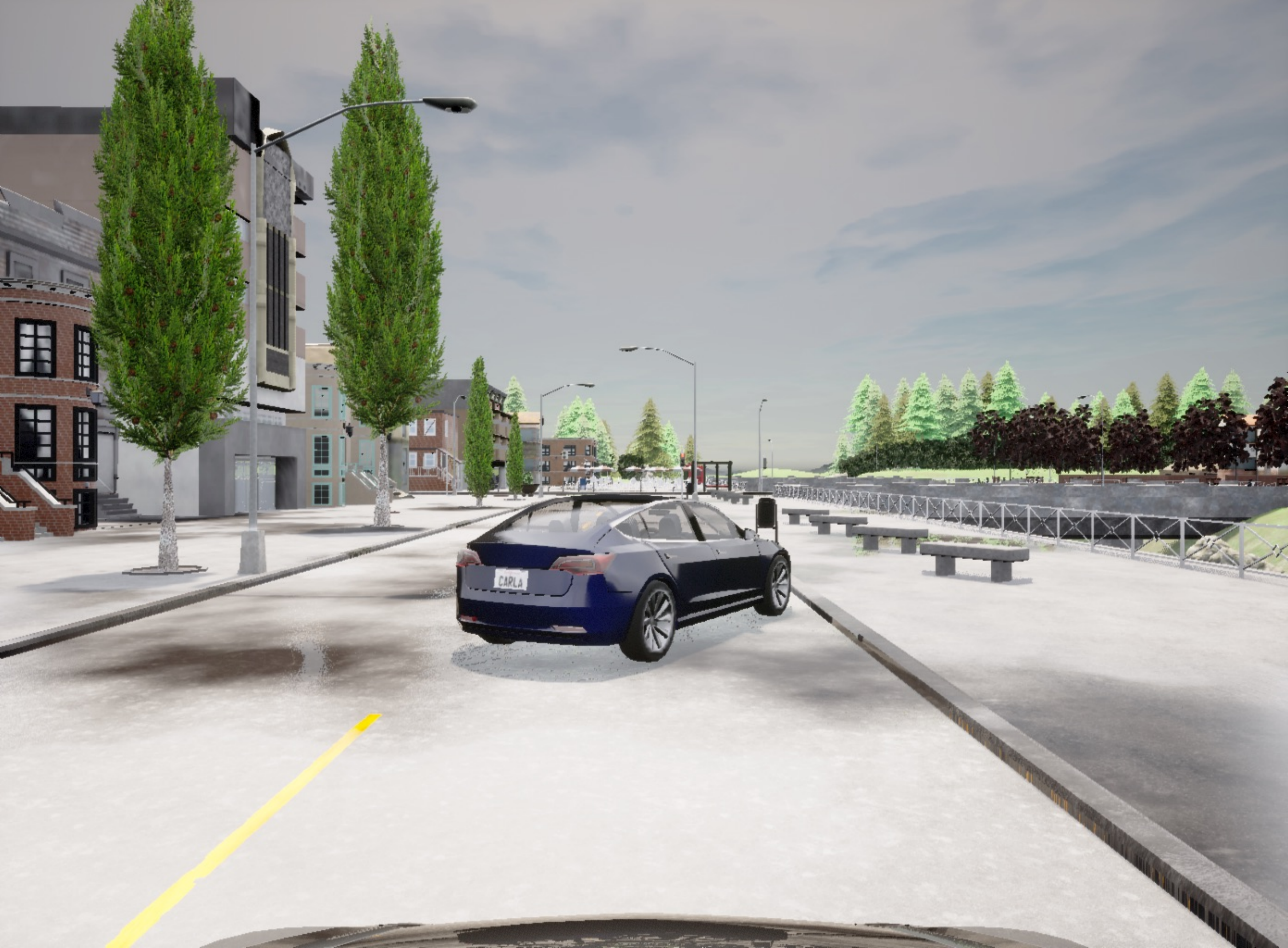}\quad
        \includegraphics[width=.2\linewidth]{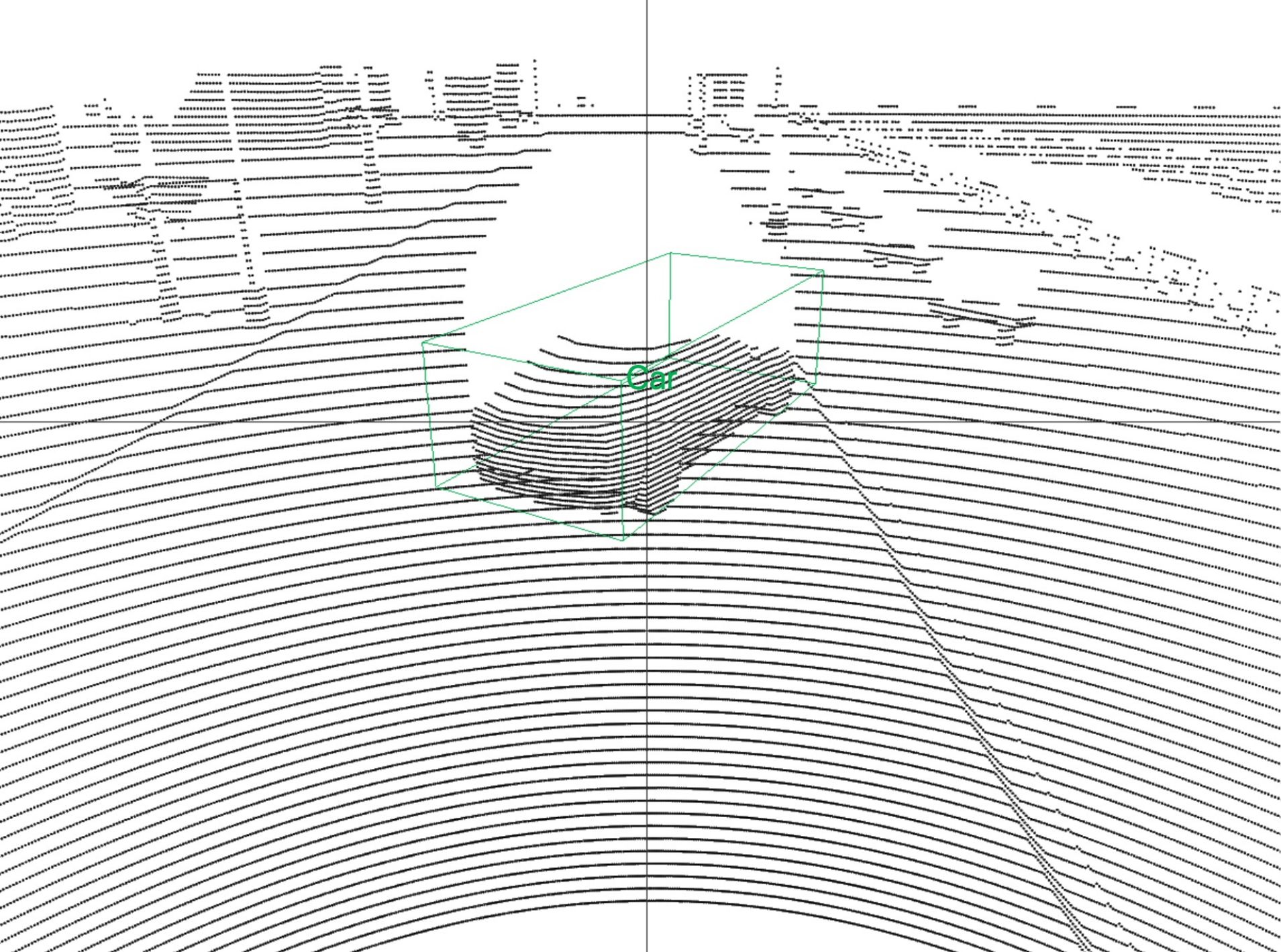}
    }
    \subcaptionbox{Shifting failure cases\label{fig:shift failure 1}}{
        \includegraphics[width=.2\linewidth]{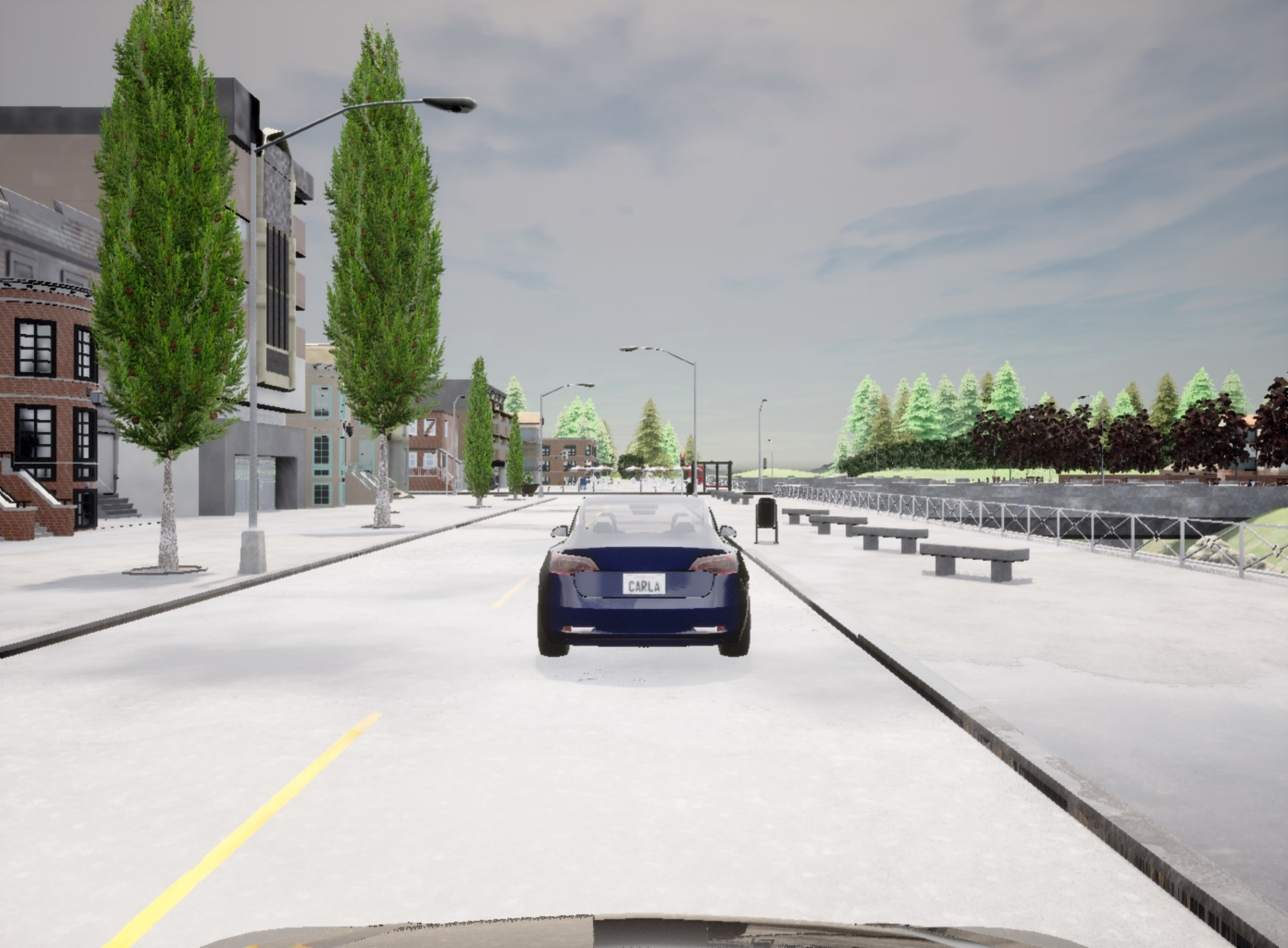}\quad
        \includegraphics[width=.2\linewidth]{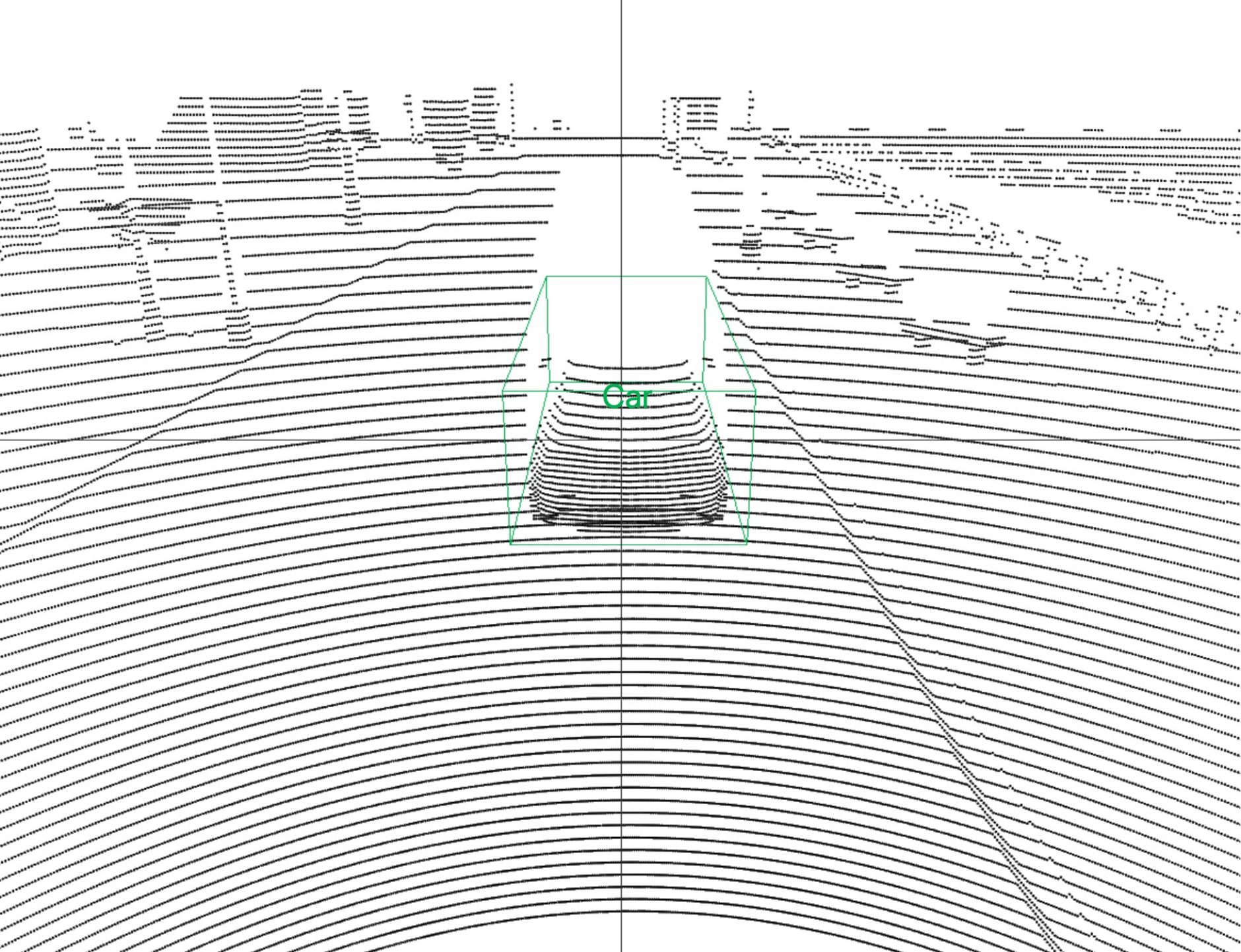}\quad
        \includegraphics[width=.2\linewidth]{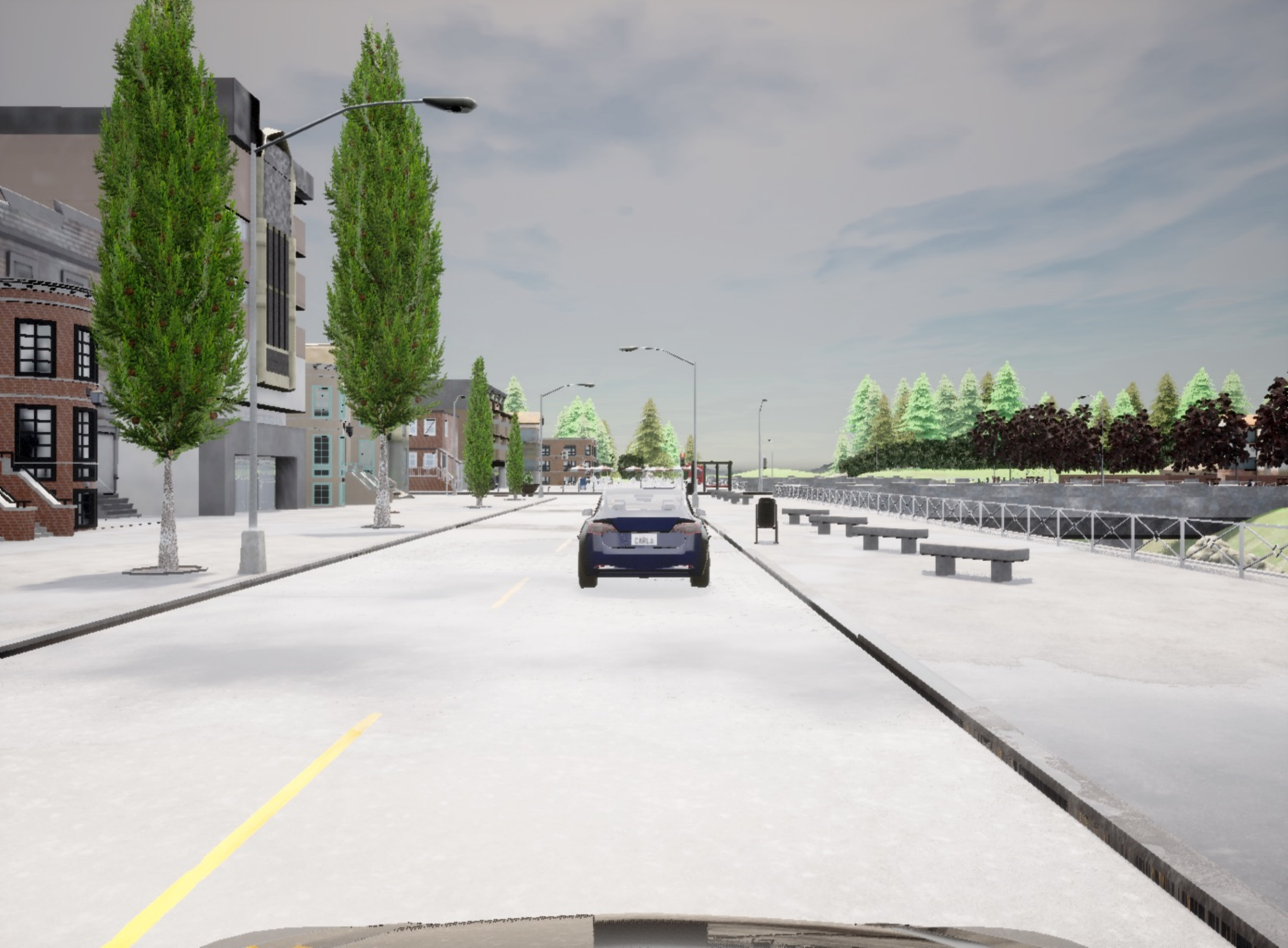}\quad
        \includegraphics[width=.2\linewidth]{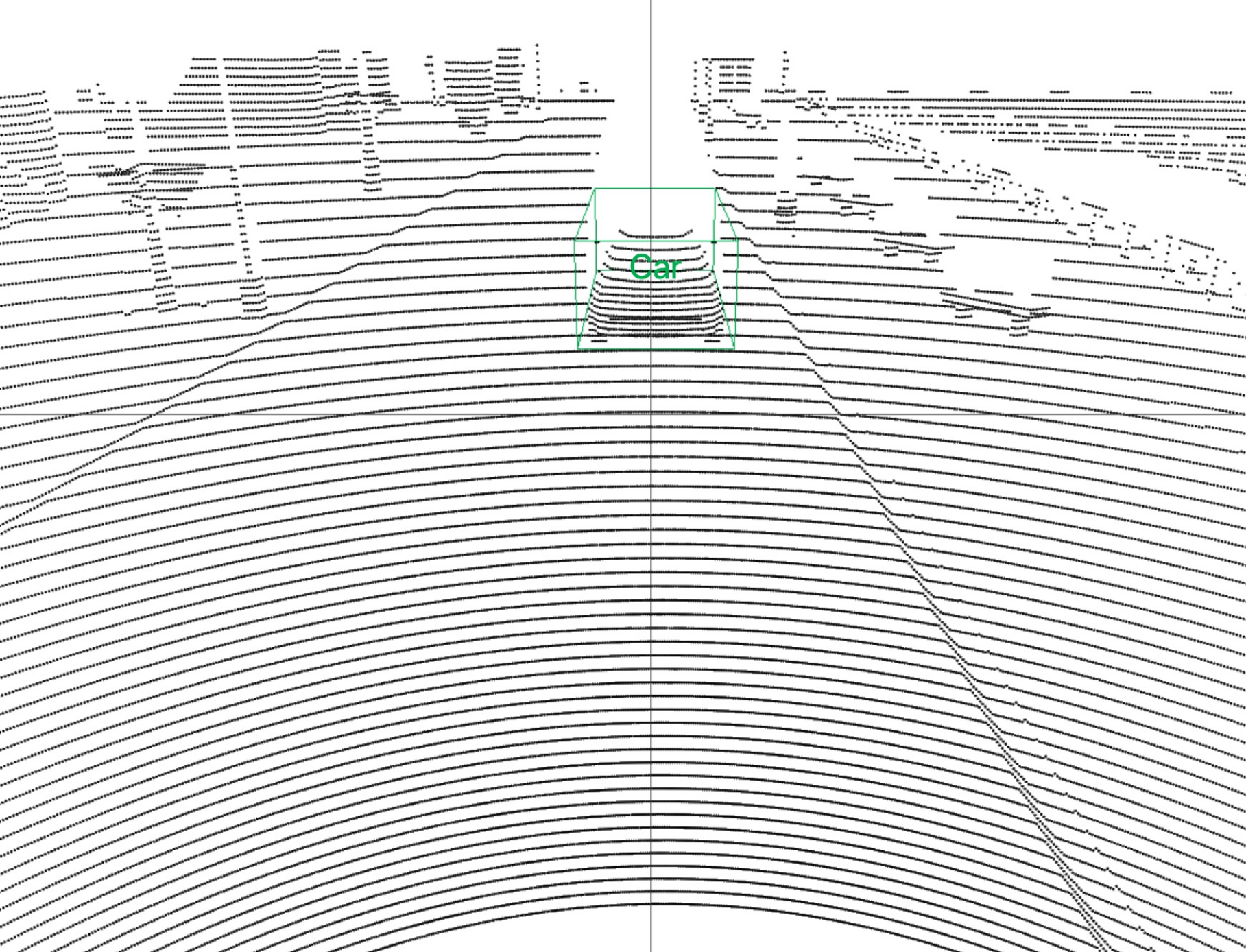}
    }
    \subcaptionbox{Shifting failure cases\label{fig:shift failure 2}}{
        \includegraphics[width=.2\linewidth]{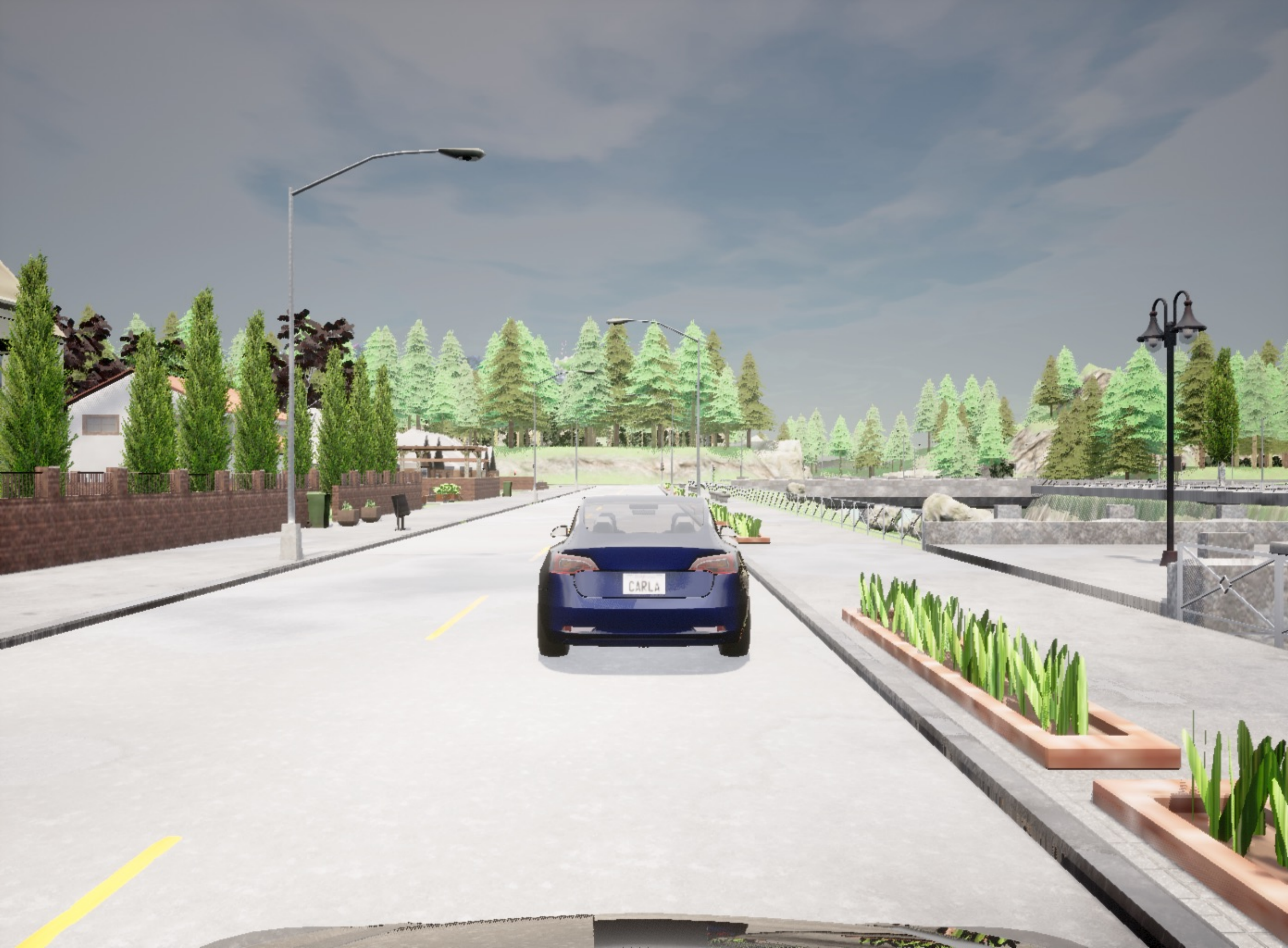}\quad
        \includegraphics[width=.2\linewidth]{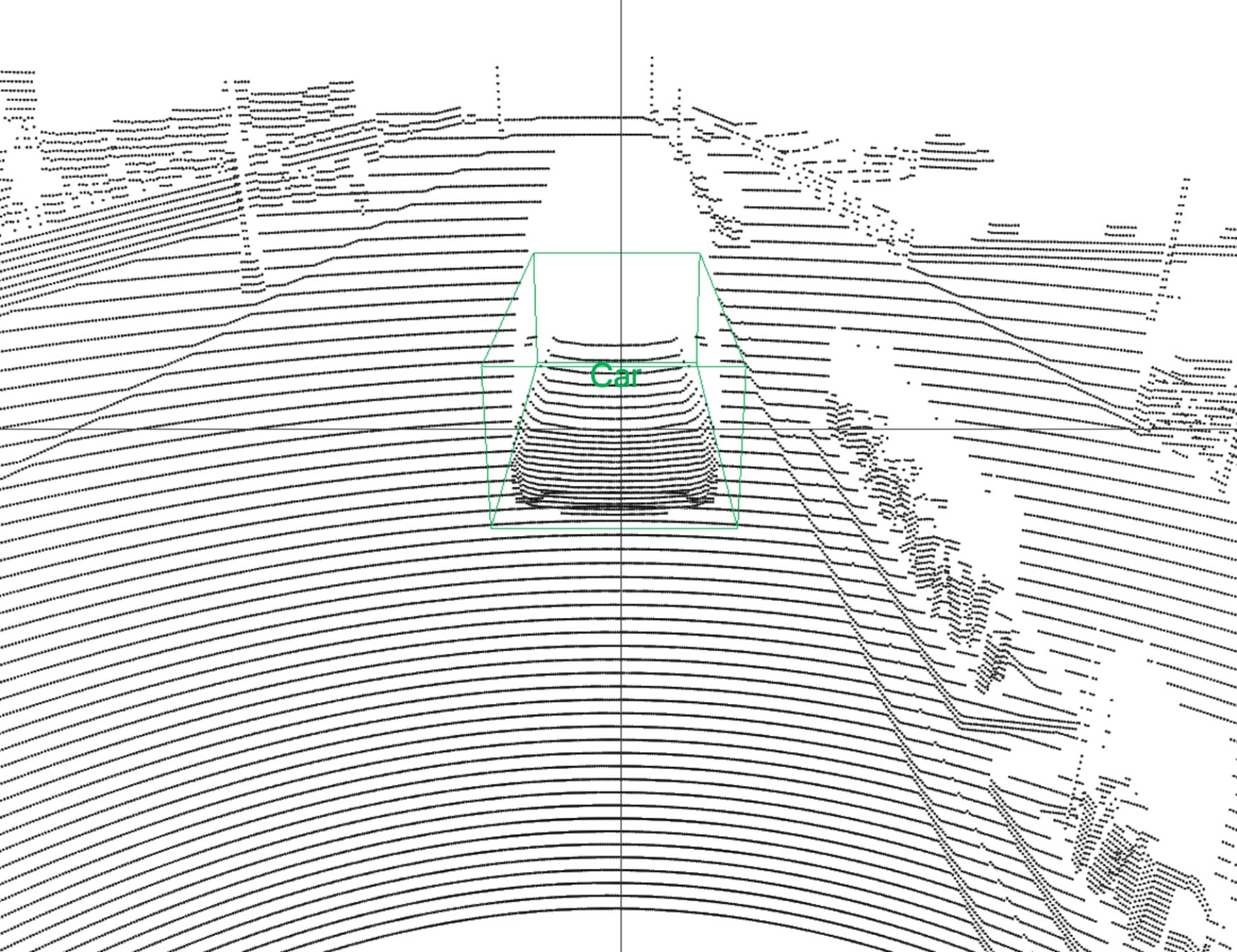}\quad
        \includegraphics[width=.2\linewidth]{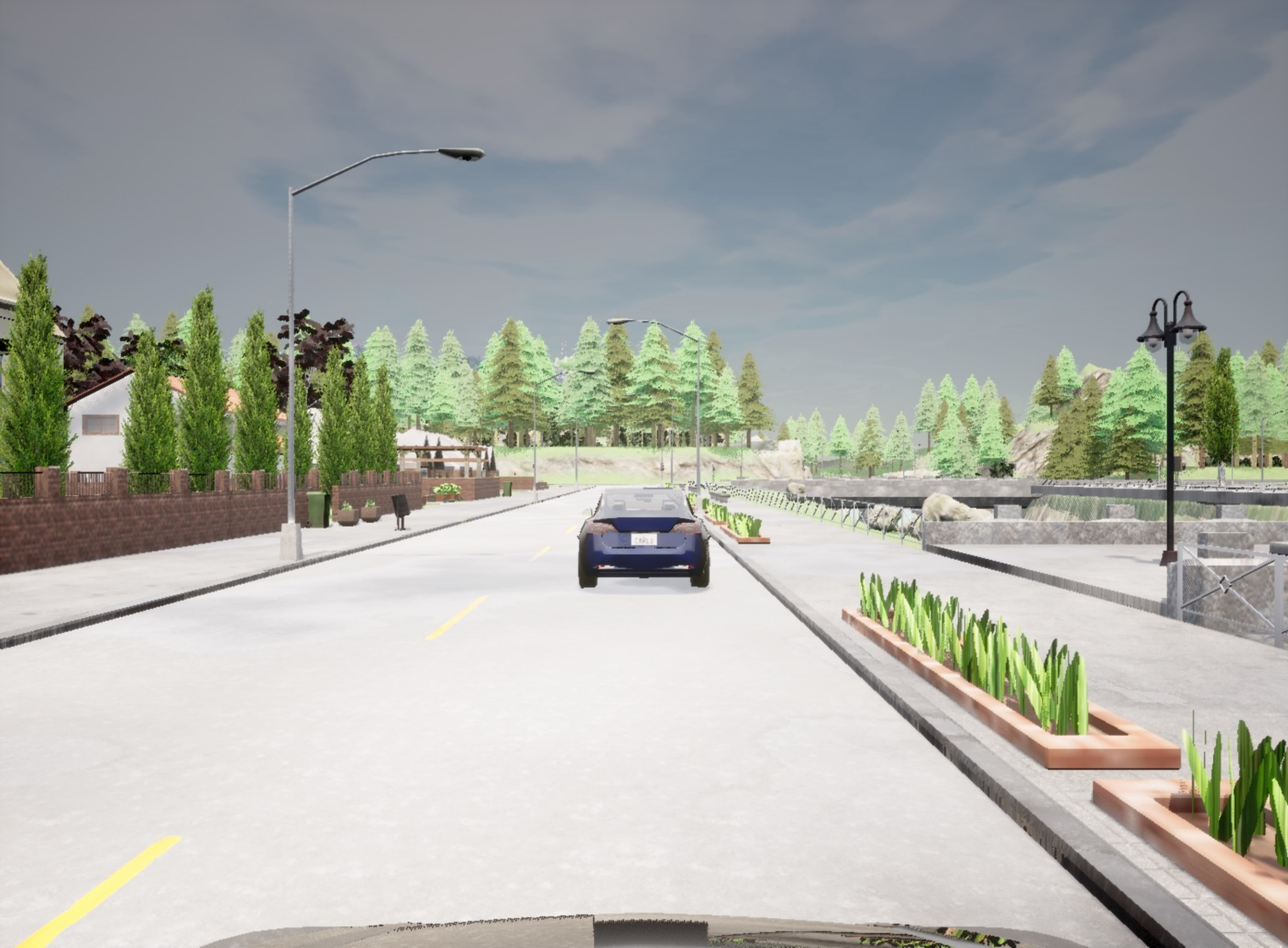}\quad
        \includegraphics[width=.2\linewidth]{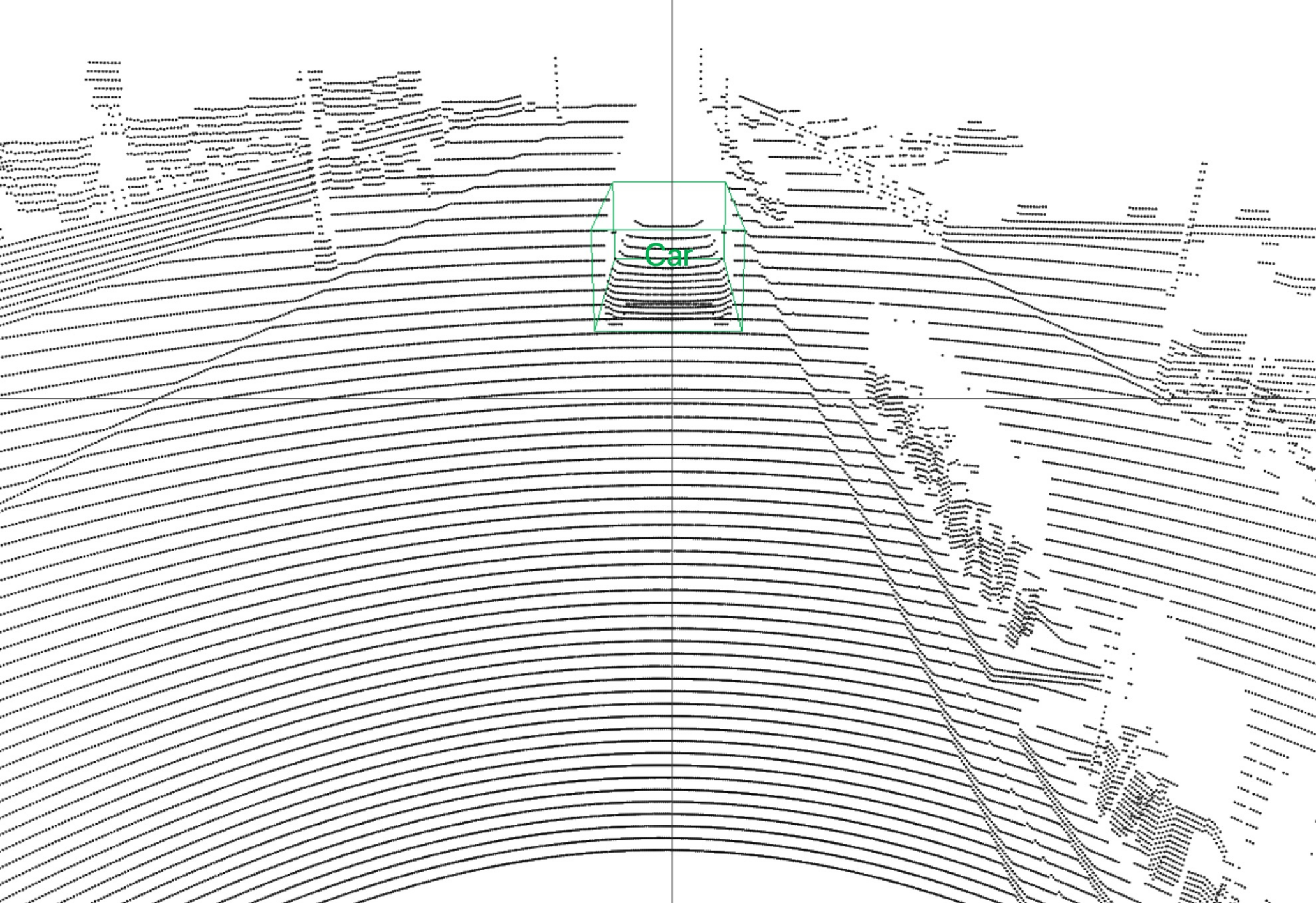}
    }
    \caption{Rotation and shifting failure cases.}
    \label{fig:transformation fail}
% \vspace{-15pt}
\end{figure}

In this subsection, we present some failure cases and possible reasons. 

As shown in \Cref{fig:rotation failure 1}, where MonoCon fails to detect the front vehicle when the rotation angle $r$ is somewhere larger than $20^\circ$ but is able to detect it when $r\leq20^\circ$, where CLOCs can detect cars with all rotation angles between $-30^\circ$ and $30^\circ$. The possible reason for this situation is that there are some objects (e.g. the puddle on the sidewalk and trees far away in this case) with similar color to the vehicle, which impacts the detection ability of camera-based detection modules. This problem can be mitigated by the LiDAR-based detection modules. 

\Cref{fig:shift failure 1} shows another failure case of camera-based models. Although there is no object with a similar color to the vehicle that we want to detect. The car farther is relatively small in the image, which is harder for camera-based models to detect. The application of point cloud data is very helpful in this case because of the perception ability of objects at long distances. 

\Cref{fig:rotation failure 2} and \Cref{fig:shift failure 2} shows two failure cases of SECOND, which is representative of LiDAR-based models where fusion models can detect relatively well. The possible reason for this case is that there are some objects very close to the vehicle (e.g.benches and grass in these cases), which affects the detection ability of point cloud modules, which can be mitigated by the combination with camera-based modules. 

\end{document}